\documentclass{article}

\usepackage{microtype}
\usepackage{graphicx}
\usepackage{subcaption}
\usepackage{booktabs} 

\usepackage{hyperref}

\usepackage[preprint]{icml_arxiv}
\usepackage[numbers]{natbib}

\usepackage{amsmath}
\usepackage{amssymb}
\usepackage{mathtools}
\usepackage{amsthm}

\usepackage{array}

\usepackage[capitalize,noabbrev]{cleveref}

\theoremstyle{plain}
\newtheorem{theorem}{Theorem}[section]

\newtheorem{lemma}[theorem]{Lemma}
\newtheorem{corollary}[theorem]{Corollary}
\theoremstyle{definition}
\newtheorem{definition}[theorem]{Definition}
\newtheorem{assumption}[theorem]{Assumption}
\theoremstyle{remark}
\newtheorem{remark}[theorem]{Remark}

\newcommand{\E}{\mathbb{E}}
\newcommand{\iid}{\overset{\text{iid}}{\sim}}

\newcommand{\R}{\mathbb{R}}

\newcommand{\Zin}{\underline{Z}}           
\newcommand{\Zout}{\bar{Z}}

\newcommand{\BigO}{\mathcal{O}}

\DeclareMathOperator{\sign}{sign}
\newcommand{\Loss}{\mathcal{L}}
\newcommand{\Var}{\mathrm{Var}}
\newcommand{\Normal}{\mathcal{N}}  

\newcommand{\abs}[1]{\left|#1\right|}
\usepackage{xcolor}
\usepackage{enumitem}

\newcommand{\initA}{Init[A]}
\newcommand{\initB}{Init[B]}

\usepackage{colortbl}

\usepackage[textsize=tiny]{todonotes}

\icmltitlerunning{Learning Rate Scaling across LoRA Ranks and Transfer to Full Finetuning}

\begin{document}
\twocolumn[
    \icmltitle{Learning Rate Scaling across LoRA Ranks and Transfer to Full Finetuning}


  \icmlsetsymbol{equal}{*}

  \begin{icmlauthorlist}
        \icmlauthor{Nan Chen}{jhu}
    \icmlauthor{Soledad Villar}{jhu}
    \icmlauthor{Soufiane Hayou}{jhu}
  \end{icmlauthorlist}

  \icmlaffiliation{jhu}{Department of Applied Mathematics and Statistics, Johns Hopkins University, Maryland, United States of America.}

  \icmlcorrespondingauthor{Soledad Villar}{svillar3@jhu.edu}
  \icmlcorrespondingauthor{Soufiane Hayou}{hayou@jhu.edu}

  \icmlkeywords{Machine Learning, ICML}

  \vskip 0.3in
]



\printAffiliationsAndNotice{}  



\begin{abstract}
Low-Rank Adaptation (LoRA) is a standard tool for parameter-efficient finetuning of large models. While it induces a small memory footprint, its training dynamics can be surprisingly complex as they depend on several hyperparameters such as initialization, adapter rank, and learning rate. In particular, it is unclear \emph{how the optimal learning rate scales with adapter rank}, which forces practitioners to re-tune the learning rate whenever the rank is changed. In this paper, we introduce \emph{Maximal-Update Adaptation} ($\mu$A), a theoretical framework that characterizes how the ``optimal'' learning rate should scale with model width and adapter rank to produce stable, non-vanishing feature updates under standard configurations. $\mu$A is inspired from the Maximal-Update Parametrization ($\mu$P) in pretraining.
Our analysis leverages techniques from hyperparameter transfer and reveals that the optimal learning rate exhibits different scaling patterns depending on initialization and LoRA scaling factor.
Specifically, we identify two regimes: one where the optimal learning rate remains roughly invariant across ranks, and another where it scales inversely with rank.
We further identify a configuration that allows learning rate transfer from LoRA to full finetuning, drastically reducing the cost of learning rate tuning for full finetuning.
Experiments across language, vision, vision–language, image generation, and reinforcement learning tasks validate our scaling rules and show that learning rates tuned on LoRA transfer reliably to full finetuning.
\end{abstract}



    

\section{Introduction}

\begin{figure}[ht]
    \centering
    \includegraphics[width=0.95\columnwidth]{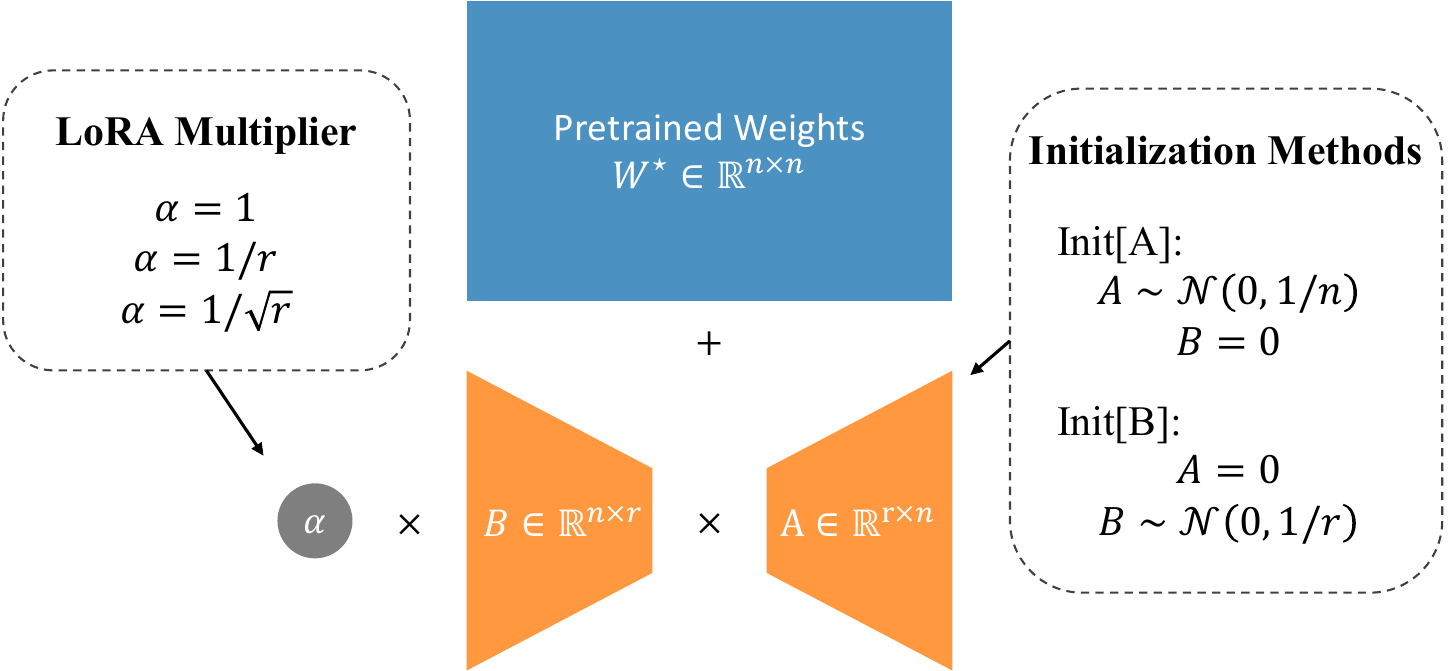}
    \caption{LoRA finetuning with standard design choices.}
    \label{img-design_choice}
\end{figure}

The widespread availability of large pretrained models has made finetuning the standard approach to adapt a general-purpose model to a specific downstream task~\cite{radford2018improving,devlin2019bert}. 
At the same time, model capability typically improves with scale, with state-of-the-art language and vision models now reaching hundreds of billions of parameters~\cite{liu2024deepseek,yang2025qwen3,team2025gemma}.
The improvement in performance with scale makes larger models attractive to practitioners, but it also makes full finetuning (FFT) increasingly expensive. While FFT remains the standard for tasks such as instruction-tuning~\cite{biderman2024lora}, parameter-efficient finetuning (PEFT) has become a default approach in many settings~\cite{han2024parameter}.
\begin{figure*}[t]
  \centering
  \begin{subfigure}[t]{0.33\textwidth}
    \centering
    \includegraphics[width=\linewidth]{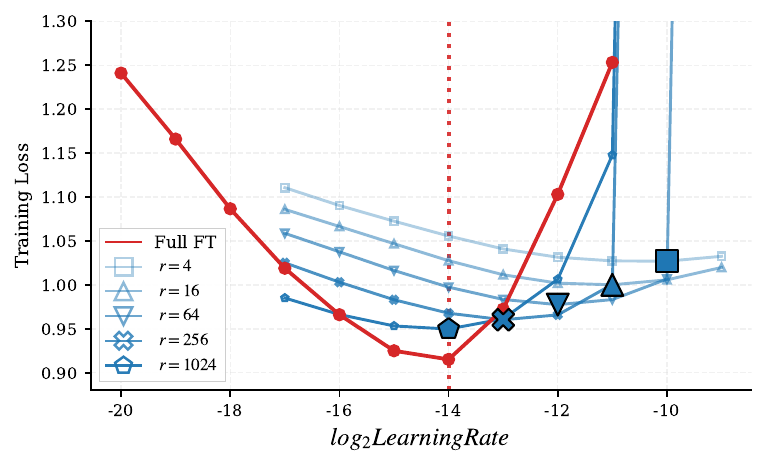}
    \caption{\initA{}, $\alpha=1$}
    \label{fig:llama3-tutu3-demonstration:a}
  \end{subfigure}\hfill
  \begin{subfigure}[t]{0.33\textwidth}
    \centering
    \includegraphics[width=\linewidth]{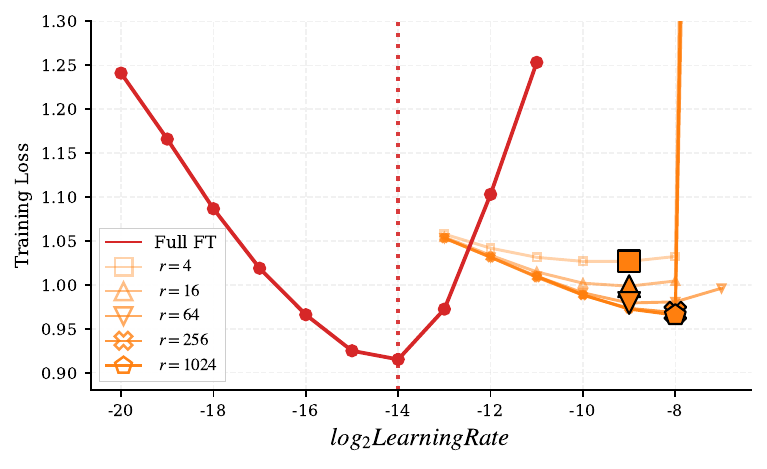}
    \caption{\initA{}, $\alpha=r^{-1}$}
    \label{fig:llama3-tutu3-demonstration:b}
  \end{subfigure}\hfill
  \begin{subfigure}[t]{0.33\textwidth}
    \centering
    \includegraphics[width=\linewidth]{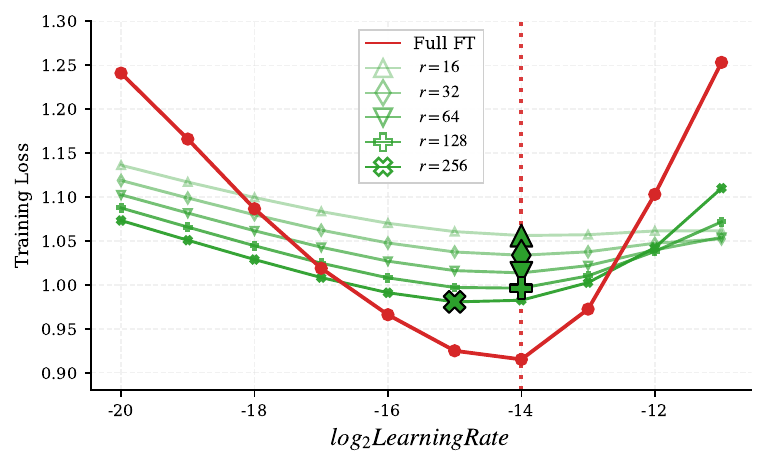}
    \caption{\initB{}, $\alpha=1$}
    \label{fig:llama3-tutu3-demonstration:c}
  \end{subfigure}
\caption{Learning rate sweeps for three standard LoRA configurations on Llama-3.2-1B (Tulu3). Each panel plots final training loss (EMA-smoothed) versus log-scale learning rate $\log_2(\eta)$ for multiple ranks (colored) and FFT (\textcolor[HTML]{d62728}{red}). Large markers indicate the optimal learning rate for each rank; the dashed red line marks the optimal FFT learning rate. The y-axis is clipped for readability.}
  \label{fig:llama3-tutu3-demonstration}
\end{figure*}
Low-Rank Adaptation (LoRA)~\cite{hu2022lora} is arguably the most popular parameter-efficient finetuning method. It adds a low-rank factorization $\alpha BA$ to the (frozen) pretrained weight matrix $W^\star\in\mathbb{R}^{n\times n}$, where $B\in\mathbb{R}^{n\times r}$, $A\in\mathbb{R}^{r\times n}$, $r$ is the LoRA rank, and $\alpha$ is a scaling multiplier. Only $B$ and $A$ are trainable in this setting.

Deploying LoRA in practice requires several design choices beyond the rank $r$.
To start from a no-op adapter, practitioners typically initialize one factor to random values and the other to zero; we denote these two common choices as \initA{} (random $A$, $B=0$) and \initB{} (random $B$, $A=0$)~\cite{hayou2024impact}.
In addition, common implementations use different rank-dependent multipliers, typically $\alpha\in\{1,\,r^{-1},\,r^{-1/2}\}$~\cite{kalajdzievski2023rank,biderman2024lora}; \cref{img-design_choice} summarizes these standard choices.
Because LoRA update is bilinear in $A$ and $B$, the initialization, multiplier $\alpha$, rank $r$, and learning rate $\eta$ are coupled and jointly determine the effective update magnitude. This coupling makes hyperparameter tuning brittle in practice.
Changing the rank or switching between initialization and multiplier choices can render a previously optimal learning rate suboptimal, forcing costly re-tuning.
\Cref{fig:llama3-tutu3-demonstration} provides a representative example on Llama-3.2-1B finetuned on the Tulu3 dataset, showing end-of-training loss as a function of learning rate for multiple ranks under three common configurations.
Under \initA{} with constant $\alpha=1$, the best learning rate decreases with rank $r$ (\cref{fig:llama3-tutu3-demonstration:a}).
Setting $\alpha=r^{-1}$ while keeping \initA{} yields a different pattern: the optimal learning rate becomes approximately rank-invariant, but the stable range is narrow—small increases above the optimum cause divergence (\cref{fig:llama3-tutu3-demonstration:b}).
With \initB{} and constant $\alpha=1$, the optimal learning rate is again rank-invariant and now aligns with the FFT optimum (\cref{fig:llama3-tutu3-demonstration:c}).
Notably, this optimal rate is more than an order of magnitude smaller than under \initA{} with $\alpha=r^{-1}$.
Taken together, these experiments show that the learning rate--rank relationship changes qualitatively across initialization and multiplier choices.
This raises a natural question: \emph{how does the optimal learning rate $\eta$ depend on rank $r$?}

To answer this question, we analyze LoRA finetuning in the joint limit $n, r \to \infty$.
Building on prior works analyzing infinite-width networks~\cite{yang2020feature,yang2021tensor,hayou2024impact,hayou2024lora+}, we study how initialization, multiplier $\alpha$, and learning rate $\eta$ jointly govern the magnitude of LoRA updates under a tractable theoretical setup.
This allows us to characterize when LoRA updates remain bounded yet non-vanishing in the large-$n$, large-$r$ regime—avoiding instability while ensuring learning does not stall. We call this regime the Maximal-Update Adaptation ($\mu$A).
From this analysis, we derive explicit scaling prescriptions for how the learning rate should depend on model width and LoRA rank under each configuration.
These prescriptions explain why some configurations (e.g., \initA{} with constant $\alpha$) require the learning rate to decrease with rank, while others (e.g., \initB{} with constant $\alpha$) admit rank-invariant learning rates.
Our analysis also connects LoRA to FFT.
For \initB{} with constant $\alpha$, the learning rate scaling that ensures stable feature learning matches that of FFT under the same assumptions.
This matching is a necessary condition for hyperparameter transfer between the two regimes.

We verify our theory through extensive experiments spanning diverse tasks, modalities, and training paradigms. 
Across all settings, the observed learning rate–rank patterns consistently match our derived scaling rules, confirming that $\mu$A scaling rules generalize across architectures and domains. Our contributions are summarized as follows:
\begin{itemize}
    \item  We introduce $\mu$A (Maximal-Update Adaptation), a theoretical framework for LoRA finetuning dynamics in the joint limit $n, r \to \infty$ and characterize how the optimal learning rate scales with model width and adapter rank to ensure maximal (stable) adaptation.
    \item Our analysis reveals two distinct regimes: one where the optimal learning rate scales inversely with rank, and one in which it remains rank-invariant. We identify a LoRA configuration where the optimal learning rate matches that of FFT (transfer from LoRA to FFT).
    
    \item Extensive experiments demonstrate that  $\mu$A scaling rules guarantee learning rate transfer across LoRA ranks and from LoRA and to FFT.
\end{itemize}

\section{Related Works}
\paragraph{Parameter-efficient finetuning.}
PEFT techniques fall into three main categories~\citep{han2024parameter, ji2025overview}: (i)~selectively updating a subset of parameters~\citep{collobert2008unified,guo2021parameter, liao2023parameter}; (ii)~adding lightweight adapters~\citep{rebuffi2017learning,houlsby2019parameter, lin2020exploring,he2022towards,hu2022lora}; or (iii)~tuning continuous prompts~\citep{li2021prefix, lester2021power, liu2024gpt}.
LoRA~\citep{hu2022lora} exemplifies the additive method and has inspired numerous variants targeting efficiency~\citep{valipour2023dylora,zhang2023adalora,dettmers2023qlora,kopiczko2024vera,ke2024unveiling,li2024vb,liu2024dora,hayou2025plop} and stability~\citep{hayou2024lora+,hayou2024impact,yen2025lora}.
Recent studies compare LoRA and FFT across gradient alignment~\citep{wang2024lora,wang2025lorapro}, expressive power~\citep{zeng2024the}, and spectral structure of weight updates~\citep{shuttleworth2024lora}.
Complementing these studies, we identify a LoRA configuration where the optimal learning rate matches that of FFT, guaranteeing learning rate transfer between the two methods.

\paragraph{Design choices of LoRA.}
Standard schemes randomly initialize either $A$ or $B$ using established techniques~\citep{glorot2010understanding,he2015delving}, setting the other to zero.
Alternative methods initialize weights from pretrained weights~\citep{buyukakyuz2024olora,meng2024pissa}, deterministic orthogonal bases~\citep{zhang2025the}, downstream data~\citep{yang2024corda,paischer2024one,wang2024lora}, or quantized models~\citep{li2024loftq}.
Prior theoretical analyses of initialization~\citep{hayou2024impact,libeyond} focus on the infinite-width regime and treat the rank as constant; in contrast, we study the joint limit where both width $n$ and rank $r$ grow.
For the LoRA multiplier $\alpha$, the original LoRA~\citep{hu2022lora} uses $\alpha=r^{-1}$ to keep update magnitudes controlled, rsLoRA~\citep{kalajdzievski2023rank} proposes $\alpha=r^{-1/2}$ for stability at larger ranks, and recent works increasingly adopt $\alpha=1$~\citep{gabrielsson2024compress,biderman2024lora,charakorn2025text}. \citet{schulman2025lora} empirically studies LoRA design choices but provides only heuristic rules for learning rate transfer between LoRA and FFT (e.g. dividing LoRA's learning rate by 10 to obtain FFT's learning rate). Another empirical work by \citet{milsom2025functionspacelearningrates} introduced a general rule for learning rate transfer across LoRA ranks but did not provide any explanation to the observed scaling rules. None of these works provide a theoretical analysis of learning rate transfer with LoRA finetuning. In this paper, we propose a principled approach to learning rate transfer across LoRA ranks and from LoRA to FFT.

\paragraph{Asymptotic scaling analysis.}
Our analysis builds on prior works analyzing the asymptotic behavior of neural networks as width approaches infinity~\citep{yang2020feature,yang2021tensor,yang2021tuning,hayou2024impact,hayou2024lora+}.
Related works apply similar frameworks to analyze depth scaling~\citep{hayou2019impact,hayou2021stable,he2023deep,hayou2023width,noci2023shaped,yang2024tensor}.
Most closely related to ours are~\citet{hayou2024impact,hayou2024lora+}, which analyze LoRA finetuning in the limit $n\to\infty$ but treat the rank $r$ as constant.
This analysis does not reveal how training dynamics depend on the rank $r$.
Here, we analyze the joint limit $n,r\to\infty$, yielding scaling prescriptions that depend on both network width and LoRA rank.
Empirically, the optimal learning rate closely match derived scaling rules.

\section{Setup and Definitions}
\label{sec:theory}
We introduce definitions and notation for our theoretical analysis, following \citet{hayou2024impact}.
\subsection{LoRA Layer and Features}
\begin{definition}[Low-Rank Adaptation (LoRA)]
\label{def.low-rank-adaptation}
Let $W^{\star}\in\R^{n\times n}$ be a pretrained weight matrix that remains fixed during finetuning. 
LoRA parameterizes the weights as
\[
W \;=\; W^{\star} + \alpha BA,
\]
where $B\in\R^{n\times r}$ and $A\in\R^{r\times n}$ are trainable factors with rank $r\ll n$ and
$\alpha\in\R$ is a scalar multiplier.
\end{definition}

For a generic LoRA-augmented linear layer, we denote the input by $\Zin\in\R^n$ and the output by $\Zout\in\R^n$,
\[
\Zout \;=\; W^\star \Zin + \alpha BA\Zin.
\]

LoRA factors are usually initialized with $BA=0$ so that finetuning starts from the pretrained model. There are two straightforward schemes that satisfy this property.
\begin{definition}[Initialization schemes for LoRA]
\label{def.init-definition}
To initialize the LoRA update as a no-op, we consider:\footnote{The choice of the Gaussian distribution here is for simplification purpose only. It can be replaced with any centred distribution with finite second moment.}
\begin{itemize}
\item \textbf{\initA{}:} $A_0$ has i.i.d.\ entries $\mathcal{N}(0,1/n)$ and $B_0=0$.
\item \textbf{\initB{}:} $B_0$ has i.i.d.\ entries $\mathcal{N}(0,1/r)$ and $A_0=0$.
\end{itemize}
\end{definition}

We now define the intermediate and output LoRA features.

\begin{definition}[LoRA features]
\label{def.lora-feature}
For a layer input $\Zin$, define the intermediate and output LoRA features
\[
Z_A := A\Zin \in \R^{r},
\qquad
Z_B := \alpha BZ_A = \alpha BA\Zin \in \R^{n}.
\]
\end{definition}

We consider a fixed number $T$ of finetuning steps, where $T$ is independent of $n$ and $r$.
For $t\le T$, we use superscripts to denote LoRA features at step $t$, i.e. $Z_A^t$, $Z_B^t$ and subscripts to denote weights at step $t$, i.e. $A_t$, $B_t$.

To isolate the contribution of a single LoRA module to feature learning, we assume that only one LoRA layer is trainable while all other network components remain frozen.
Under this assumption, the layer input $\Zin$ (corresponding to a given model input $x$) remains constant during finetuning.

At any step $t$, recall that $Z_A^t = A_t \Zin$ and $Z_B^t = \alpha B_t Z_A^t$.
Expanding the increment $\Delta Z_B^t := Z_B^t - Z_B^{t-1}$ yields:
\begin{equation}
\label{eq.z_b_update}
    \Delta Z_B^t = \underbrace{\alpha B_{t-1} \Delta Z_A^t}_{\delta_t^1} + \underbrace{\alpha \Delta B_t Z_A^{t-1}}_{\delta_t^2} + \underbrace{\alpha \Delta B_t \Delta Z_A^t}_{\delta_t^3}
\end{equation}
We refer to the three contributions in~\cref{eq.z_b_update} as $\delta_t^1$, $\delta_t^2$, and $\delta_t^3$. 
As previsouly explained in~\citet{hayou2024lora+}, the terms $\delta_t^1$ and $\delta_t^2$ correspond to \emph{linear} updates in which one LoRA factor is trained while the other is held fixed, whereas $\delta_t^3$ corresponds to the \emph{multiplicative updates} of $A$ and $B$. The next section discusses this decomposition and how it relates to stable feature learning.

\subsection{Stable Feature Learning}

We aim to understand how feature updates behave as a function of model width $n$ and LoRA rank $r$. For this purpose, we consider the setting where both $n$ and $r$ grow, subject to $r \leq n$.
To formalize how features and updates scale with $(n,r)$, we adopt the following asymptotic notation.

\begin{definition}[Asymptotic notation]
\label{def.asymptotic_notion}
For sequences $c_{n,r}\in\R$ and $d_{n,r}\in\R_{+}$, we write $c_{n,r}=\BigO(d_{n,r})$
(resp.\ $c_{n,r}=\Omega(d_{n,r})$) if there exist constants $\kappa>0$ and $N\in\mathbb{N}$ such that for all
$n,r> N$ with $r\leq n$, we have $|c_{n,r}|< \kappa d_{n,r}$ (resp.\ $|c_{n,r}|> \kappa d_{n,r}$).
We write $c_{n,r}=\Theta(d_{n,r})$ if both bounds hold.
This notation extends element-wise to vectors: for $c_{n,r}=(c_{n,r}^i)_{i=1}^k\in\R^k$ and $d_{n,r}=(d_{n,r}^i)_{i=1}^k\in\R_+^k$, we write
$c_{n,r}=\BigO(d_{n,r})$ when $c_{n,r}^i=\BigO(d_{n,r}^i)$ for all $i\in[k]$, and similarly for $\Omega$ and $\Theta$.
For random variables, the notation is understood in the second-moment sense:
$c_{n,r}=\BigO(d_{n,r})$ means $\bigl(\E|c_{n,r}|^2\bigr)^{1/2}=\BigO(d_{n,r})$.
\end{definition}

Our stability objective requires that the LoRA-induced feature contribution $Z_B^t$ remains bounded over $T$ finetuning steps (where $T$ is fixed), so that activations remain bounded.

\begin{definition}[Feature stability]
\label{def.feat_stability}
Fix a finite number of finetuning steps $T$ independent of $(n,r)$. We say LoRA finetuning is \emph{stable} if for all steps $t\le T$, we have $Z_B^{t} = \BigO(1)$ in the joint limit $n,r\to\infty$ with $r\leq n$.
\end{definition}

Since LoRA is typically initialized as a no-op ($Z_B^0=0$), Lemma~\ref{lemma:bounded-features} below implies that bounding the per-step increments $\Delta Z_B^t = \BigO(1)$ is sufficient for stability.

\begin{lemma}[Telescoping bound]
\label{lemma:bounded-features}
Fix $t\le T$. If $Z_B^0 = \BigO(1)$ and $\Delta Z_B^s = \BigO(1)$ for all $s\le t$, then $Z_B^t = \BigO(1)$.
\end{lemma}

However, stability alone permits trivial solutions—for instance, initializing both $A$ and $B$ to zero yields vanishing updates.
We therefore require that the LoRA feature increment remains order-one.

\begin{definition}[Stable feature learning with LoRA]
\label{def.feat_learning}
We say LoRA induces \emph{stable feature learning} if the dynamics are stable (Definition~\ref{def.feat_stability}), and $\Delta Z_B^t = \Theta(1)$ for all $t \le T$.
\end{definition}

Intuitively, having $\Delta Z_B^t = \Theta(1)$ is associated with ``maximal'' adaptation in a similar fashion to $\mu$P for pretraining, see \cite{yang2020feature}. By enforcing this condition, we derived \emph{Maximal-Update Adaptation} scaling rules in the next section.
Beyond stability, it is useful to understand \emph{which} components of the update drive learning.
Ideally, each of $\delta_t^1$, $\delta_t^2$, and $\delta_t^3$ should be $\Theta(1)$, indicating that both factors $A$ and $B$ contribute meaningfully within each step.
Accordingly, in our main results we track the term-wise scalings as a diagnostic: if a term vanishes in the large-$(n,r)$ limit, the corresponding factor becomes effectively inactive.

Recall from~\cref{eq.z_b_update} that the feature update $\Delta Z_B^t$ decomposes into three terms $(\delta_t^i)_{i\in\{1,2,3\}}$.
A sufficient condition for stability is to bound each term individually:
\[
\delta_t^1=\BigO(1),\qquad \delta_t^2=\BigO(1),\qquad \delta_t^3=\BigO(1).
\]
This term-wise control avoids relying on cancellations among the three contributions.
Stable feature learning additionally requires $\Delta Z_B^t=\Theta(1)$, which implies that at least one term must be $\Omega(1)$ under the chosen scaling, while the remaining terms may be lower-order.

\section{Maximal-Update Adaptation ($\mu$A)}

We now state the learning rate scaling rules that guarantee stable feature learning under \initA{} and \initB{}.
Full proofs and derivations (including the scaling of each $\delta_t^i$) appear in Appendix~\ref{app:theory}.
Our analysis relies on two assumptions: (i)~bounded forward and backward signals at each step, and (ii)~a SignSGD update rule that uses the sign of the gradient, serving as a tractable proxy for the Adam optimizer.

\begin{assumption}[Bounded forward/backward signals]
\label{assumption:in_out}
For any fixed step $t\le T$, the layer input and the backpropagated output gradient satisfy
\[
\Zin = \Theta(1),
\qquad
d\Zout := \frac{\partial \mathcal{L}}{\partial \Zout} = \Theta(1).
\]
This ensures that features and gradients neither explode nor vanish at any step as $(n,r)\to\infty$.
\end{assumption}

\begin{assumption}[Optimizer abstraction]
\label{assumption:signsgd}
We use a simplified variant of Adam~\citep{Kingma14Adam} without momentum, reducing it to SignSGD~\citep{Bernstein18SIGNSGD}, which gives element-wise sign updates:
\[
g_{A}^t := \operatorname{sign}\!\Big(\frac{\partial \mathcal{L}_t}{\partial A_{t-1}}\Big),
\qquad
g_{B}^t := \operatorname{sign}\!\Big(\frac{\partial \mathcal{L}_t}{\partial B_{t-1}}\Big),
\]
and parameter updates
\[
A_t = A_{t-1} - \eta\, g_{A}^t,
\qquad
B_t = B_{t-1} - \eta\, g_{B}^t,
\]
where $\operatorname{sign}(\cdot)$ is applied element-wise (with $\operatorname{sign}(0)=0$) and $\eta$ denotes the learning rate.
\end{assumption}

Following \cite{hayou2024impact}, we consider the setting where the loss is computed on a single sample. This simplifies the theoretical analysis and yields closed-form expressions. 
In what follows, we characterize the asymptotic behavior of feature updates as a function of $(n,r)$ and show that it depends on initialization schemes. We present the results for \initA{} and \initB{} and compare the differences.

\subsection{Init[A] Analysis}
We begin with \initA{}, where $A_0$ is randomly initialized and $B_0=0$. 
Theorem~\ref{thm:initA} characterizes the scale of the LoRA feature increment $\Delta Z_B^t$. Corollary~\ref{cor:unified-initA} then yields the learning rate scaling that achieves stable feature learning.

\begin{theorem}[\initA{}]
\label{thm:initA}
Under Assumptions~\ref{assumption:in_out} and~\ref{assumption:signsgd}, for any fixed $t\le T$ with \initA{},
\[
\Delta Z_B^t
\;=\;
\max\{\Theta(\eta \alpha r),\ \Theta(\eta^2 \alpha r n)\}.
\]
\end{theorem}

Imposing the stable feature learning condition $\Delta Z_B^t=\Theta(1)$ yields the following scaling rule for learning rate $\eta$.

\begin{corollary}[$\mu$A scaling rules for \initA{}]
\label{cor:unified-initA}
Let $\alpha = r^{-\gamma}$ for $\gamma\in[0,1]$. Under \initA{}, achieving stable feature learning
(i.e., $\Delta Z_B^t=\Theta(1)$ for fixed $t\le T$) requires
\[
\eta = \Theta(n^{-1/2} r^{-{(1-\gamma)}/{2}}).
\]
With this choice,
\[
\delta_t^1=\Theta(r^{-1/2}),
\
\delta_t^2=\Theta(1),
\
\delta_t^3=\Theta(r^{-1/2}),
\]
and the intermediate feature scales as
\[
Z_A^t=\Theta(n^{1/2}\, r^{-{(1-\gamma)}/{2}}).
\]
\end{corollary}

To make the prescriptions concrete, we summarize the resulting scalings for common multiplier choices in~\cref{table:initA-alpha-result}.

\begin{table}[t]
  \caption{$\mu$A scaling rules from Corollary~\ref{cor:unified-initA} for common choices of $\alpha$. Each cell shows the argument of $\Theta(\cdot)$.}
  \label{table:initA-alpha-result}
  \begin{center}
    \begin{small}
      \begin{sc}
        \begin{tabular}{lccc}
          \toprule
          $\gamma$ & $\alpha$ & $\eta$ & $Z_A^t$ \\
          \midrule
          $0$ & 1 & $n^{-{1}/{2}} r^{-{1}/{2}}$ & $n^{{1}/{2}} r^{-{1}/{2}}$ \\
          $\frac{1}{2}$ & $r^{-{1}/{2}}$ & $n^{-{1}/{2}} r^{-{1}/{4}}$ & $n^{{1}/{2}} r^{-{1}/{4}}$ \\
          $1$ & $r^{-1}$ & $n^{-{1}/{2}}$ & $n^{{1}/{2}}$ \\
          \bottomrule
        \end{tabular}
      \end{sc}
    \end{small}
  \end{center}
  \vskip -0.1in
\end{table}

\paragraph{Interpretation of Corollary~\ref{cor:unified-initA}.} We highlight the following key insights drawn from the corollary.

\emph{Learning rate scaling with rank.}
For a fixed model width $n$, the relationship between the optimal learning rate and rank $r$ depends on the choice of $\alpha$.
With $\alpha=1$, the learning rate decreases as $\eta \propto r^{-1/2}$: increasing the rank by $4\times$ requires halving the learning rate, consistent with~\cref{fig:llama3-tutu3-demonstration:a}.
In contrast, setting $\alpha = r^{-1}$ absorbs the rank dependence, yielding a rank-invariant learning rate that depends only on width, consistent with~\cref{fig:llama3-tutu3-demonstration:b}.

\emph{Intermediate feature scaling.}
The intermediate features $Z_A^t$ grow with width $n$ but shrink with rank when $\alpha \in \{1, r^{-1/2}\}$.
For $\alpha=1$, doubling the rank reduces the intermediate feature scale by $\sqrt{2}$, which can improve internal numerical stability in wide models where large intermediate activations have been identified as a source of training instability~\citep{hayou2024impact}.
However, when $\alpha = r^{-1}$, increasing the rank no longer suppresses the intermediate scale, which may explain why \initA{} with $\alpha = r^{-1}$ exhibits a narrower stable learning rate range in practice.

\emph{Learning efficiency.}
Under the stable learning rate scaling, only $\delta_t^2$ remains order-one while $\delta_t^1$ and $\delta_t^3$ vanish as $r$ grows.
Since $\delta_t^1$ captures the contribution of updating $A$, this means that changes to $A$ contribute negligibly to the overall feature update.
As a result, the adapter increasingly behaves like a random projection $Z_A \approx A_0 \Zin$ with $A$ held fixed at initialization, while learning proceeds primarily through $B$.
This imbalance between the two factors is intrinsic to \initA{} and persists across all standard choices of $\alpha$.

\subsection{\initB{} Analysis}
We next consider \initB{}, where $B_0$ is randomly initialized and $A_0=0$.
\begin{theorem}[\initB{}]
\label{thm:initB}
Under Assumptions~\ref{assumption:in_out} and~\ref{assumption:signsgd}, for any fixed $t\le T$
with \initB{},
\[
\Delta Z_B^t
\;=\;
\max\{\Theta(\eta \alpha n),\ \Theta(\eta^2 \alpha n r)\}.
\]
\end{theorem}

Unlike \initA{}, the resulting learning rate prescription depends more sensitively on how $\alpha$ scales with $r$. In the main text we focus on the most common configuration $\alpha=1$ and defer additional cases to Appendix~\ref{app:theory:initB}.

\begin{corollary}[$\mu$A scaling rules for \initB{} with $\alpha=1$]
\label{cor:initB-constant}
Let $\alpha=1$. Under \initB{}, achieving stable feature learning (i.e., $\Delta Z_B^t=\Theta(1)$ for fixed $t\le T$)
requires
\[
\eta = \Theta(n^{-1}).
\]
With this choice,
\[
\delta_t^1=\Theta(1),
\
\delta_t^2=\Theta(rn^{-1}),
\
\delta_t^3=\Theta(r^{1/2}{n}^{-1}),
\]
and the intermediate feature scales as
\[
Z_A^t=\Theta(1).
\]
\end{corollary}

\paragraph{Interpretation of Corollary~\ref{cor:initB-constant}.}
This corollary reveals a regime qualitatively different from \initA{}.

\emph{Rank-invariant learning rate.}
With $\alpha=1$, the required learning rate scales only with width ($\eta \propto n^{-1}$) and is independent of rank.
This explains the roughly rank-invariant optimal learning rate observed under \initB{} in~\cref{fig:llama3-tutu3-demonstration:c}.

\emph{Connection to FFT.}
The same width scaling $\eta \propto n^{-1}$ arises for full finetuning of a linear layer under the same SignSGD abstraction (see~\cref{thm:app:full_finetuning} in Appendix~\ref{app:theory:full-ft}).
This matching shows a necessary condition for learning rate transfer between LoRA and FFT: a learning rate tuned for \initB{} with $\alpha=1$ can serve as a reasonable starting point for FFT without re-tuning.
We empirically verify this transferability in our experiments. 

\emph{Intermediate feature scaling.}
The intermediate features $Z_A^t$ remain order-one throughout training, keeping internal activations controlled regardless of width.

\emph{Learning efficiency.}
The term-wise contributions exhibit the opposite imbalance from \initA{}: here $\delta_t^1 = \Theta(1)$ while $\delta_t^2$ and $\delta_t^3$ vanish when $r\ll n$.
When the rank is much smaller than the width, updates to $B$ contribute negligibly, and the adapter behaves as though $B$ were fixed at its random initialization, with learning occurring primarily through $A$.
Nonetheless, increasing the rank can strengthen $\delta_t^2$; in the extreme case where $r$ is on the order of $n$, both factors contribute meaningfully to feature updates.

\begin{figure*}[ht]
  \centering
  \begin{subfigure}[t]{0.24\textwidth}
    \centering
    \includegraphics[width=\linewidth]{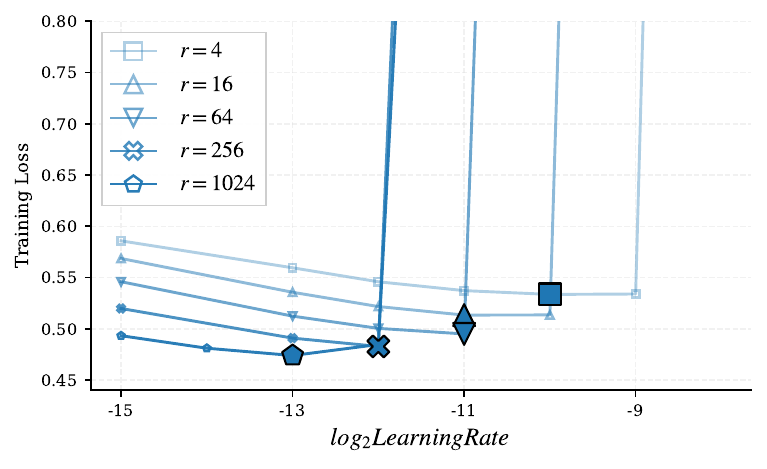}
    \caption{Qwen2.5-3B-Instruct}
    \label{fig:initA-constant1-qwen}
  \end{subfigure}\hfill
  \begin{subfigure}[t]{0.24\textwidth}
    \centering
    \includegraphics[width=\linewidth]{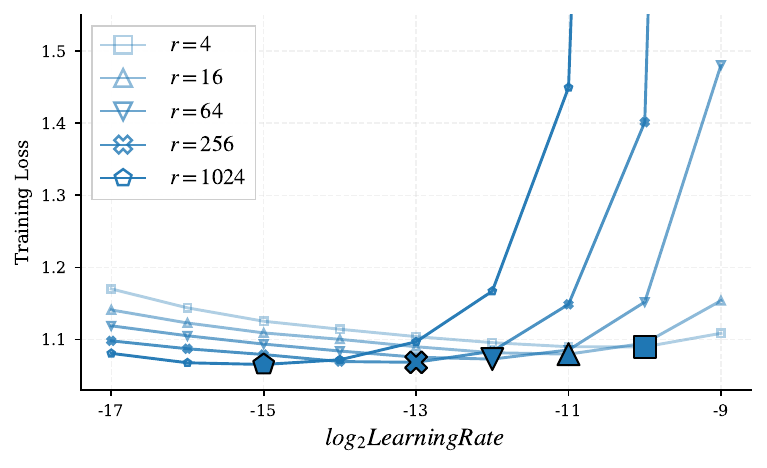}
    \caption{Qwen3-VL-2B-Instruct}
    \label{fig:initA-constant1-qwen-vl}
  \end{subfigure}\hfill
  \begin{subfigure}[t]{0.24\textwidth}
    \centering
    \includegraphics[width=\linewidth]{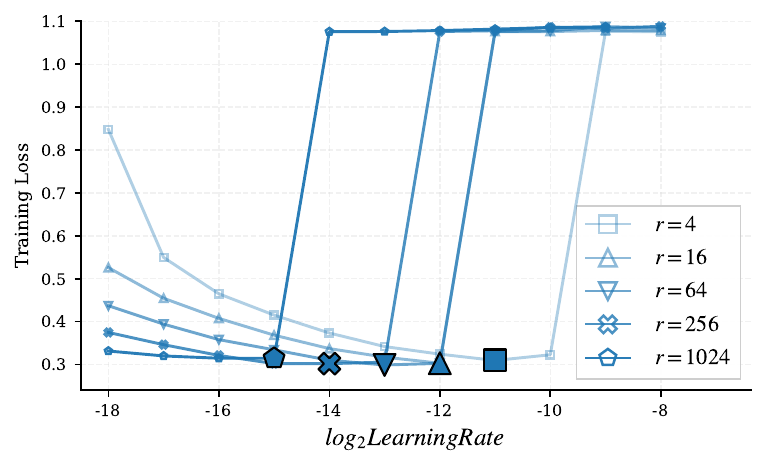}
    \caption{RoBERTa-large}
    \label{fig:initA-constant1-roberta}
  \end{subfigure}
  \begin{subfigure}[t]{0.24\textwidth}
    \centering
    \includegraphics[width=\linewidth]{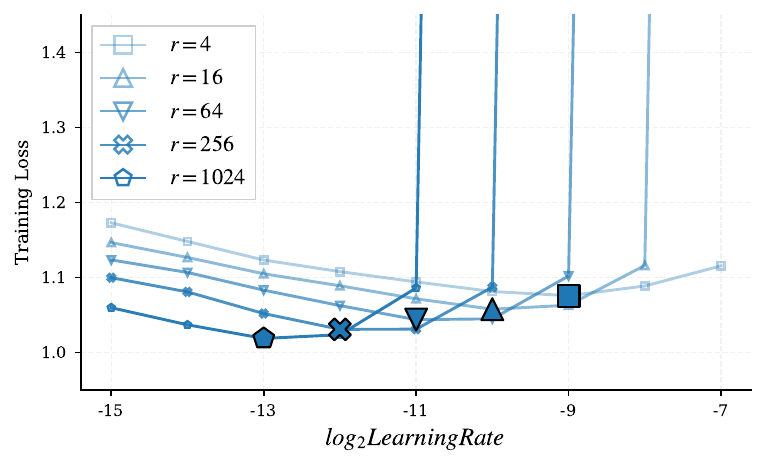}
    \caption{ViT-Huge/14}
    \label{fig:initA-constant1-vit-huge}
  \end{subfigure}
  \caption{Learning rate sweeps for \initA{} with $\alpha=1$ across four models. Curves show final training loss (EMA-smoothed) versus log-scale learning rate ($\log_2(\eta)$). Large markers denote per-rank optima. The Y-axis is clipped for readability.}

  \label{fig:initA-constant1-main}
\end{figure*}

\begin{figure*}[t]
  \centering
  \begin{subfigure}[t]{0.24\textwidth}
    \centering
    \includegraphics[width=\linewidth]{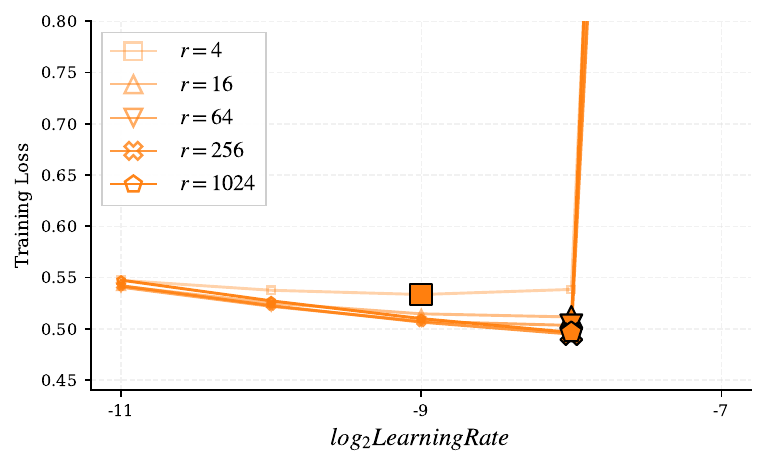}
    \caption{Qwen2.5-3B-Instruct}
    \label{fig:initA-alpha1-qwen}
  \end{subfigure}\hfill
  \begin{subfigure}[t]{0.24\textwidth}
    \centering
    \includegraphics[width=\linewidth]{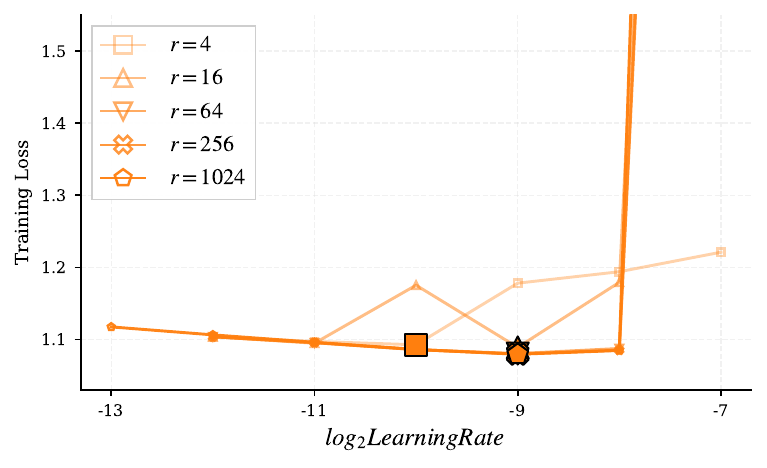}
    \caption{Qwen3-VL-2B-Instruct}
    \label{fig:initA-alpha1-qwen-vl}
  \end{subfigure}\hfill
  \begin{subfigure}[t]{0.24\textwidth}
    \centering
    \includegraphics[width=\linewidth]{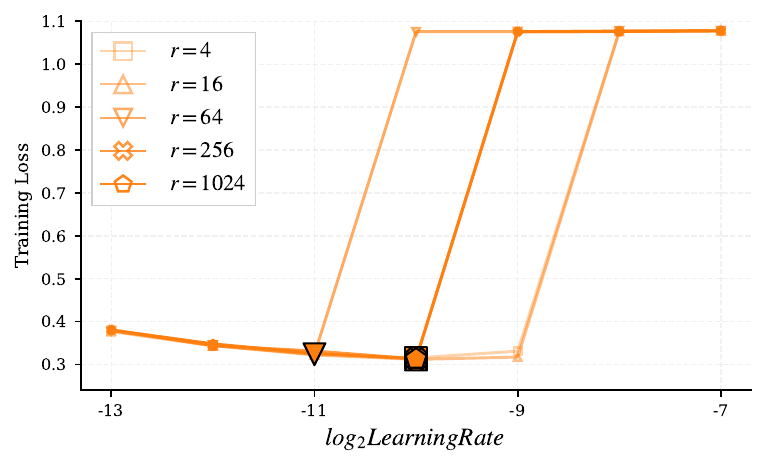}
    \caption{RoBERTa-large}
    \label{fig:initA-alpha1-roberta}
  \end{subfigure}
  \begin{subfigure}[t]{0.24\textwidth}
    \centering
    \includegraphics[width=\linewidth]{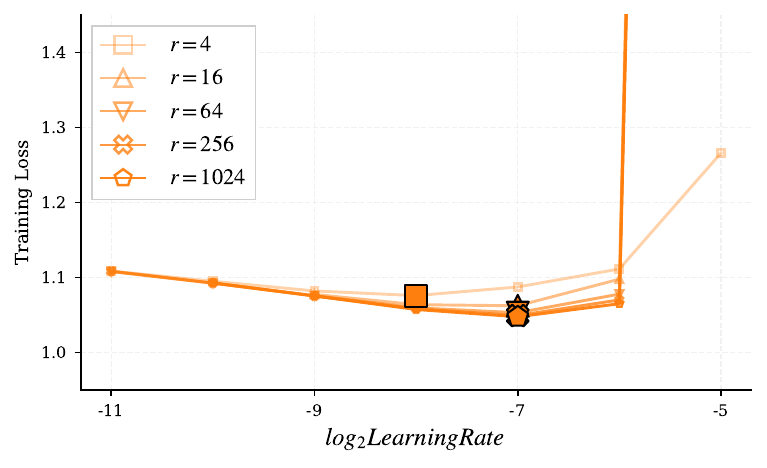}
    \caption{ViT-Huge/14}
    \label{fig:initA-alpha1-vit-huge}
  \end{subfigure}
  \caption{Learning rate sweeps for \initA{} with rank-dependent $\alpha=r^{-1}$ across four models. Curves show final training loss (EMA-smoothed) versus log-scale learning rate ($\log_2(\eta)$). Large markers denote per-rank optima. The Y-axis is clipped for readability.}
  \label{fig:initA-alpha1-main}
\end{figure*}

\section{Experiments}

Our experiments verify $\mu$A scaling rules empirically and learning rate transfer from LoRA to FFT.
We examine three regimes: supervised finetuning (SFT) across diverse architectures, reinforcement learning with verifiable rewards (RLVR), and text-to-image diffusion model finetuning.

\subsection{Experimental Setup}
\label{sec:setup}

\paragraph{Unified finetuning recipe.}
We use a unified finetuning recipe across all experiments.
Specifically, we use AdamW with weight decay of $0.01$ and global-norm gradient clipping of $1.0$.
The learning rate warms up linearly over the first $5\%$ of training steps to its peak ($\eta$), then follows a cosine decay to $0.1\times$ the peak over the remaining steps.

\paragraph{SFT setting.}
We experiment with five public base models:


\emph{Decoder-only language models.}
For instruction tuning, we tune Llama-3.2-1B~\citep{grattafiori2024llama} on the Tulu-3 supervised finetuning mixture~\citep{lambert2024tulu}.
For long-context reasoning, we train Qwen2.5-3B-Instruct~\citep{qwen2025qwen25technicalreport} on the OpenThoughts-114k~\citep{guha2025openthoughts}.


\emph{Encoder-only models.}
We finetune RoBERTa-large~\citep{liu2019roberta} on the ANLI benchmark~\citep{nie2019adversarial} and adapt ViT-Huge/14~\citep{dosovitskiy2021an} to ImageNet-1K~\citep{deng2009imagenet}.
Following prior work~\citep{guo2024crossmae}, we fix the learning rate of the classification head at $10^{-3}$ and sweep only the LoRA learning rate.

\emph{Vision--language model.}
For multimodal instruction following, we finetune Qwen3-VL-2B-Instruct~\citep{bai2025qwen3vltechnicalreport} on LLaVA-Instruct-Mix, containing image--instruction pairs. 

\paragraph{RLVR setting.}
We adapt Llama-3.1-8B~\citep{grattafiori2024llama} to GSM8k~\citep{cobbe2021gsm8k} with GRPO~\citep{shao2024deepseekmath}.
The reward is based on answer correctness and format compliance.

\paragraph{Text-to-image diffusion model finetuning.}
We finetune Stable-Diffusion-v1.5~\citep{Rombach_2022_CVPR} on the Naruto-BLIP-Captions~\citep{cervenka2022naruto2} dataset.

\paragraph{Evaluation protocol.}
We compare three LoRA configurations: \initA{} with $\alpha=1$, \initA{} with $\alpha=r^{-1}$, and \initB{} with $\alpha=1$.
For each configuration, we sweep learning rate $\eta$ (peak learning rate)  on a $\log_2$-scale grid with multiplicative steps of $2\times$ and select the best value independently for (i) final training loss (EMA-smoothed) and (ii) a task-specific validation metric.
For RLVR, we select learning rates by final training reward; for diffusion models, by validation FID~\citep{heusel2017gans}.
More details are in Appendix~\ref{app:exp-details}.
We report training-loss-based sweeps in the main text and defer the other results to Appendix~\ref{app:exp-sft-results},~\ref{app:sec:revr-results},~\ref{app:sec:stable-diffusion-results}.

\begin{figure*}[t]
  \centering
  \begin{subfigure}[t]{0.24\textwidth}
    \centering
    \includegraphics[width=\linewidth]{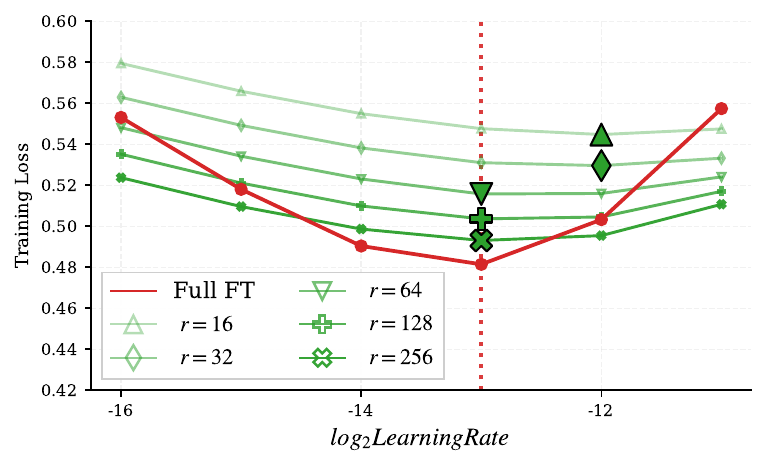}
    \caption{Qwen2.5-3B-Instruct}
    \label{fig:initB-constant1-qwen}
  \end{subfigure}\hfill
  \begin{subfigure}[t]{0.24\textwidth}
    \centering
    \includegraphics[width=\linewidth]{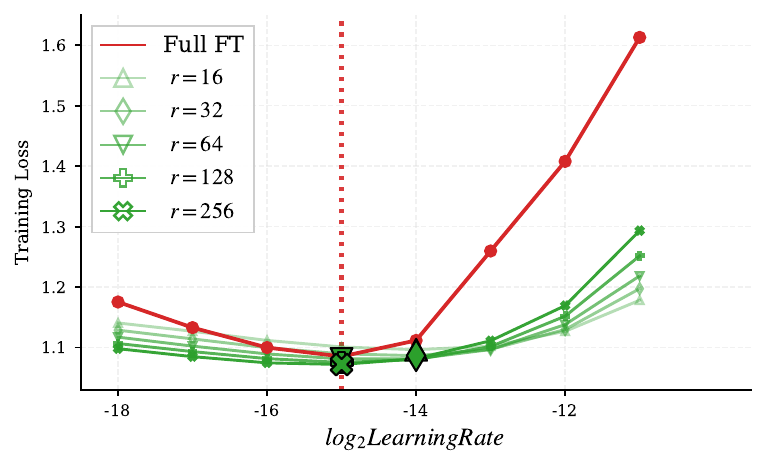}
    \caption{Qwen3-VL-2B-Instruct}
    \label{fig:initB-constant1-qwen-vl}
  \end{subfigure}\hfill
  \begin{subfigure}[t]{0.24\textwidth}
    \centering
    \includegraphics[width=\linewidth]{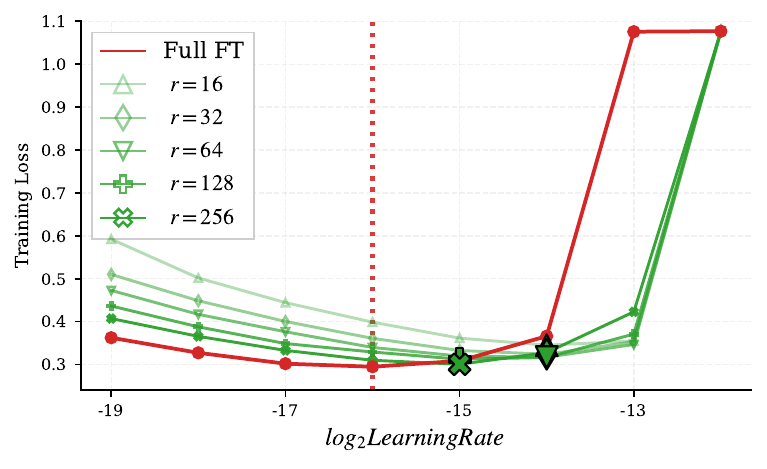}
    \caption{RoBERTa-large}
    \label{fig:initB-constant1-roberta}
  \end{subfigure}
  \begin{subfigure}[t]{0.24\textwidth}
    \centering
    \includegraphics[width=\linewidth]{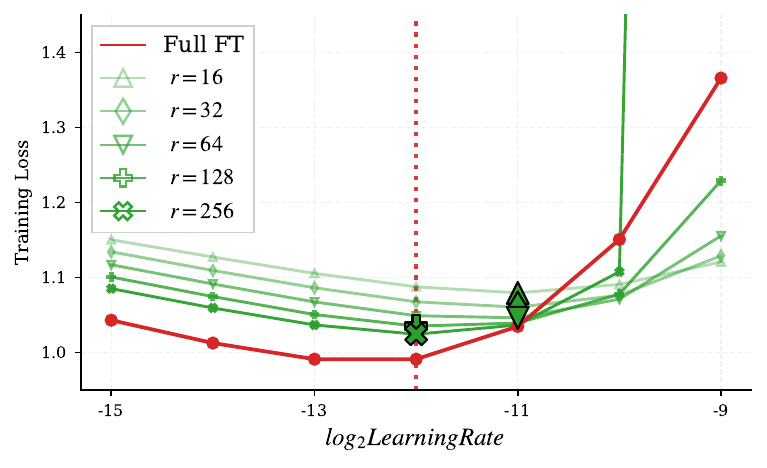}
    \caption{ViT-Huge/14}
    \label{fig:initB-constant1-vit-huge}
  \end{subfigure}
  \caption{Learning rate sweeps for \initB{} with $\alpha=1$ across four models.
  Curves show final training loss (EMA-smoothed) versus log-scale learning rate ($\log_2(\eta)$) for multiple ranks (\textcolor[HTML]{2ca02c}{green}) and FFT (\textcolor[HTML]{d62728}{red}). Large markers denote per-rank optima. The Y-axis is clipped for readability. The red vertical dashed line marks the optimal FFT learning rate.}
  \label{fig:initB-constant1-main}
\end{figure*}

\begin{figure*}[t]
  \centering
  \begin{subfigure}[t]{0.24\textwidth}
    \centering
    \includegraphics[width=\linewidth]{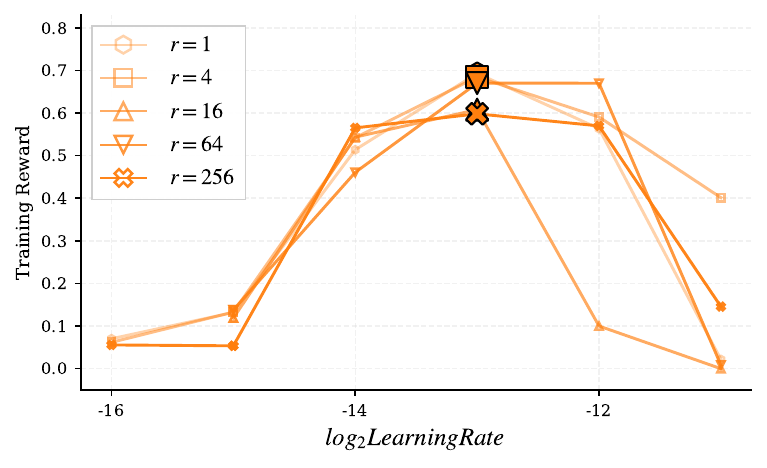}
    \caption{\initA{}, $\alpha=r^{-1}$}
    \label{fig:reinforcement-learning-initA-alpha1}
  \end{subfigure}\hfill
  \begin{subfigure}[t]{0.24\textwidth}
    \centering
    \includegraphics[width=\linewidth]{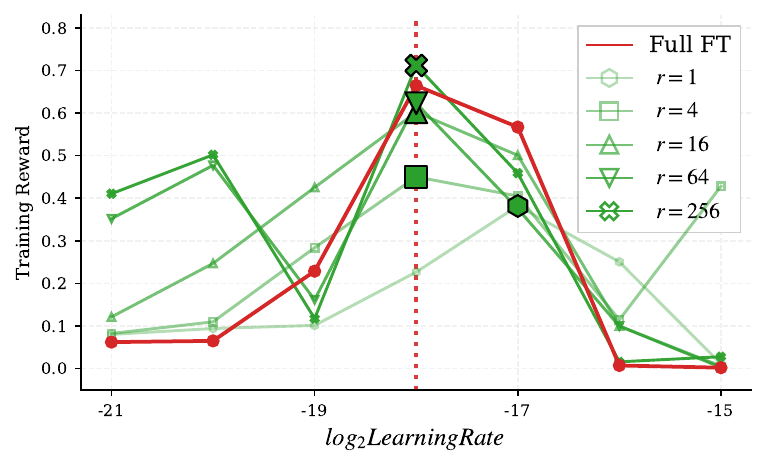}
    \caption{\initB{}, $\alpha=1$}
    \label{fig:reinforcement-learning-initB-constant1}
  \end{subfigure}\hfill
  \begin{subfigure}[t]{0.24\textwidth}
    \centering
    \includegraphics[width=\linewidth]{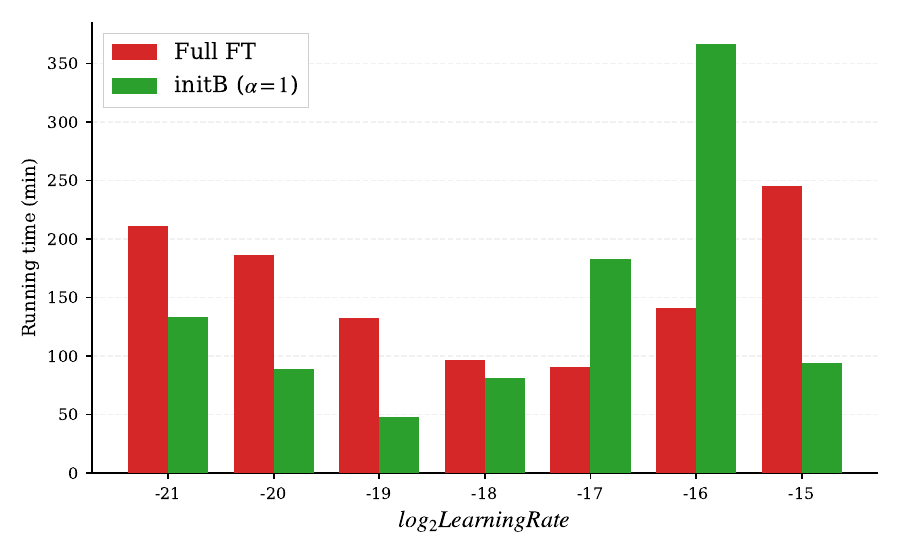}
    \caption{Training time}
    \label{fig:reinforcement-learning-training-time}
  \end{subfigure}
  \begin{subfigure}[t]{0.24\textwidth}
    \centering
    \includegraphics[width=\linewidth]{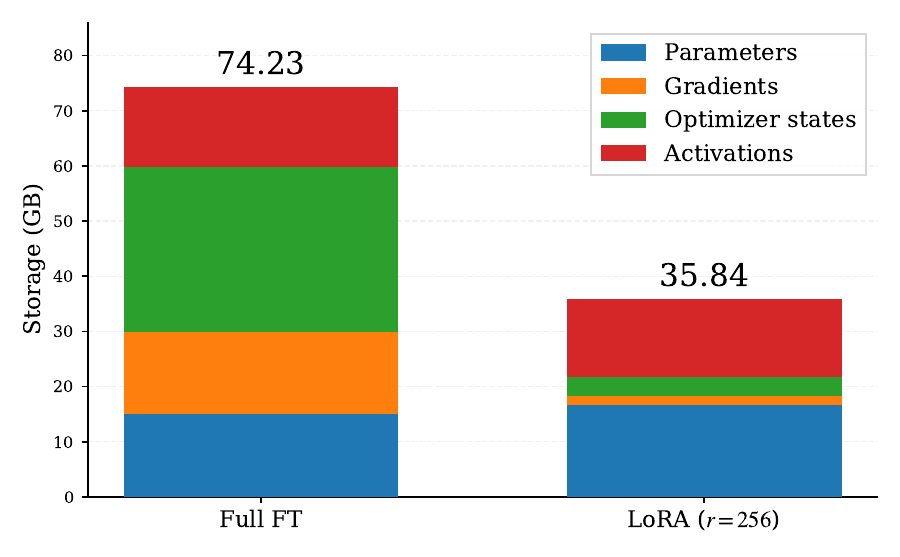}
    \caption{Peak memory profile}
    \label{fig::reinforcement-learning-storage-profile}
  \end{subfigure}
  \caption{RLVR results on Llama-3.1-8B (GSM8k) with GRPO. 
  (a)--(b)~Learning rate sweeps plotting final training reward versus $\log_2(\eta)$. Large markers denote the best learning rate per rank; the red dashed line marks the FFT optimum. 
  (c)~Training time versus learning rate for FFT (\textcolor[HTML]{d62728}{red}) and \initB{} with $\alpha{=}1$ and $r{=}256$ (\textcolor[HTML]{2ca02c}{green}). 
  (d)~Peak GPU memory breakdown for FFT (left) and LoRA with $r{=}256$ (right).}
  \label{fig:reinforcement-learning}
\end{figure*}


\paragraph{Theory–experiment alignment.}
Our scaling rules characterize the leading-order dependence of the stable learning rate on width $n$ and rank $r$ in the joint limit $(n, r) \to \infty$.
In practice, however, finetuning at finite width and rank introduces two deviations from the asymptotic prediction.

\emph{Finite-size effects.}
When $(n, r)$ are finite, constant factors and lower-order terms can shift the optimal learning rate. These shifts are more pronounced for smaller ranks, where the large-$r$ asymptotics are less accurate.

\emph{Discrete learning rate grid.}
We sweep learning rates on a $\log_2$-spaced grid. 
The reported optimum is the best value among discrete points.
When the loss landscape is flat near the true optimum, adjacent grid points perform similarly, and training noise may determine which one appears best.

Accordingly, we evaluate the scaling rules by qualitative trends rather than exact numerical agreement.
For configurations predicted to be rank-invariant, we expect optimal learning rates to cluster within one grid step ($2\times$).
For rank-dependent configurations, we expect the optimum to change monotonically with rank and to span multiple grid points.

\subsection{SFT Results}


\paragraph{\initA{} with $\alpha=1$ (Figures~\ref{fig:initA-constant1-main},~\ref{fig:llama3-tutu3-demonstration:a})}

For a fixed base model (i.e., fixed width $n$), our theory in~\cref{table:initA-alpha-result} predicts that the optimal learning rate scales as $\eta \propto {r}^{-1/2}$ with rank. 
This prediction is broadly supported by the experiments: as rank increases, the best-performing learning rate monotonically shifts toward smaller values. 
Since each figure plots $\log_2(\eta)$ on the X-axis, a $4\times$ increase in rank should shift the optimum leftward by one unit (i.e., a $2\times$ reduction in learning rate)—a pattern clearly visible across most models.
We observe minor deviations on Qwen2.5-3B-Instruct, where both $r=16$ and $r=64$ attain their minimum loss at the same learning rate.
We attribute this to the learning rate grid discretization: as shown in~\cref{fig:initA-constant1-qwen}, the loss curves for $r=4$ and $r=16$ flatten near the optimum.
Moreover, models finetuned with \initA{} and $\alpha=1$ are sensitive to large learning rates: increasing the learning rate by $2\times$ beyond the optimum can substantially degrade training loss or cause divergence, particularly for Qwen2.5-3B-Instruct and RoBERTa-large.

While we report results for $\alpha=1$ here, our theory holds for any rank-independent constant $\alpha$. Since $\alpha=2$ is a common default in practice, we verify in Appendix~\ref{app:exp-sft-results-initA-constant1} that the same scaling behavior persists for $\alpha \in \{2, 4\}$.

\paragraph{\initA{} with $\alpha=r^{-1}$ (Figures~\ref{fig:initA-alpha1-main},~\ref{fig:llama3-tutu3-demonstration:b})}

According to our theory, for a fixed base model, the optimal learning rate should be independent of LoRA rank under this setting. 
The experimental results support this prediction: for all models, the optimal learning rates across ranks fall within a narrow range (typically within $2\times$), consistent with rank-invariance.
As with \initA{} and $\alpha=1$, this small variation is attributable to the discrete learning rate grid.
Compared to the $\alpha=1$ setting, the optimal learning rates under $\alpha=r^{-1}$ are substantially larger.
We also observe that smaller ranks (e.g., $r=4$) usually favor slightly smaller learning rates than larger ranks—opposite to the trend observed with $\alpha=1$.
This likely reflects the finite-size effects: at smaller ranks, lower-order terms can exert a non-negligible influence on feature learning and suppress the optimal learning rate.

\paragraph{\initB{} with $\alpha=1$ (Figures~\ref{fig:initB-constant1-main},~\ref{fig:llama3-tutu3-demonstration:c})}

In this case, the optimal learning rate should be independent of LoRA rank. 
The empirical results confirm this prediction: across all models, the best learning rates for different ranks cluster tightly.
The small variation likely stems from learning rate grid resolution and finite-size effects.
Although both \initB{} with $\alpha=1$ and \initA{} with $\alpha=r^{-1}$ exhibit rank-invariant optimal learning rates, our theory shows they differ in width scaling: $\eta \propto n^{-1}$ for the former versus $\eta \propto n^{-1/2}$ for the latter.
For large $n$, \initB{} thus requires substantially smaller learning rates, as our experiments confirm.
We also observe that \initB{} with $\alpha=1$ is more tolerant of larger learning rates.
This stability is consistent with the theoretical finding that intermediate features remain order-one under \initB{}, preventing internal activations from growing with width.

Crucially, our theory reveals that FFT shares the same learning rate scaling $\eta \propto n^{-1}$ with \initB{} with $\alpha=1$, suggesting that their optimal learning rates may align and leading to  learning rate transfer between the two. 
Experiments support this conjecture: for most models, the optimal learning rates for \initB{} ($\alpha=1$) align closely with those for FFT.
The exception is RoBERTa-large, where FFT favors a slightly smaller learning rate than \initB{}. 
We attribute this gap to the additional classification head: although its learning rate is fixed, jointly training it with LoRA adapters may still perturb the optimal learning rate. 
Nevertheless, our experiments show that learning rates tuned for \initB{} with $\alpha=1$ transfers effectively to FFT in most model architectures, with minimal performance degradation.

\subsection{RLVR Results}

RLVR results are reported in~\cref{fig:reinforcement-learning}.
The learning rate sweeps reveal patterns consistent with our findings under SFT.
Under \initA{} with $\alpha = r^{-1}$ (\cref{fig:reinforcement-learning-initA-alpha1}), the optimal learning rate is approximately rank-invariant: curves for different ranks peak at nearly the same learning rate.
Similarly, under \initB{} with $\alpha = 1$ (\cref{fig:reinforcement-learning-initB-constant1}), the optimal learning rate is rank-invariant and aligns with the FFT optimum, suggesting that learning rates tuned with LoRA can transfer seamlessly to FFT.
For results of \initA{} with $\alpha=1$, please kindly refer to Appendix~\ref{app:sec:revr-results-sweep}.

Compared with SFT, RL finetuning is more sensitive to learning rate selection. In SFT, a learning rate that is too small simply slows convergence; in RLVR, however, overly large or small learning rates substantially degrade training reward.
Importantly, this sensitivity also manifests in training time.
As shown in~\cref{fig:reinforcement-learning-training-time}, the optimal learning rate incurs the shortest training time, while suboptimal rates can increase wall-clock time by $2$--$4\times$.
The $\mu$A scaling rules thus provide practical value beyond final performance: by identifying well-behaved learning rate regions in advance, practitioners can avoid time-consuming hyperparameter searches over suboptimal configurations. 
We provide a detailed investigation on this phenomenon in Appendix~\ref{app:subsection:rl-ft-case-study}.

Although LoRA and FFT have comparable per-step wall-clock times in our experiment, their memory footprints differ substantially.
\Cref{fig::reinforcement-learning-storage-profile} shows that FFT consumes over $2\times$ the GPU memory of LoRA ($r{=}256$), with most overhead attributable to gradient buffers and optimizer states.
This gap widens further at larger model sizes.
Our theory thus enables a practical workflow that exploits this difference: one can sweep learning rates using \initB{} with $\alpha{=}1$ on mid-tier GPUs (e.g., A6000), which are more available and cost-effective, then transfer the optimal rate directly to FFT on higher-capacity hardware with minimal re-tuning.

\begin{figure}[t]
  \centering
  \begin{subfigure}[t]{0.24\textwidth}
    \centering
    \includegraphics[width=\linewidth]{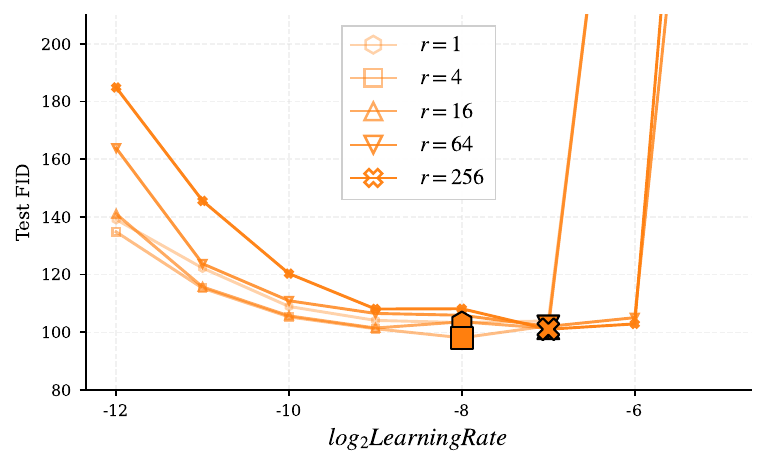}
    \caption{\initA{}, $\alpha=r^{-1}$}
    \label{fig:stable-diffusion-initA-alpha1}
  \end{subfigure}\hfill
  \begin{subfigure}[t]{0.24\textwidth}
    \centering
    \includegraphics[width=\linewidth]{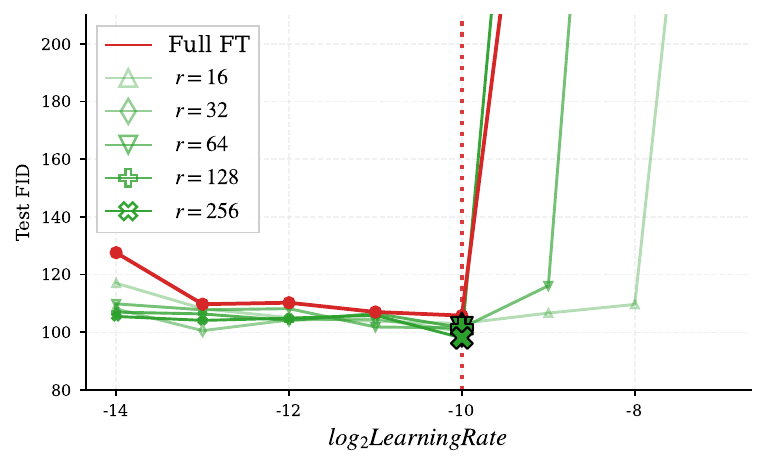}
    \caption{\initB{}, $\alpha=1$}
    \label{fig:stable-diffusion-initB-constant1}
  \end{subfigure}
  \caption{Results on Stable-Diffusion-1.5 (naruto-blip-captions). Subfigures plot validation FID score versus log-scale learning rate ($\log_2(\eta)$). Large markers denote the best learning rate for each rank; the red vertical dashed line marks the FFT optimum.}
  \label{fig:stable-diffusion}
\end{figure}

\subsection{Text-to-Image Diffusion Finetuning Results}
\Cref{fig:stable-diffusion} summarizes the results for finetuning a stable diffusion model. 
The learning rate sweeps show patterns consistent with our theory: under \initA{} with $\alpha=r^{-1}$, the optimal learning rate is rank-invariant, while under \initB{} with $\alpha=1$, the optimal learning rate remains independent of rank and aligns with that of FFT. Notably, FFT is more sensitive to learning rate selection on diffusion models, exhibiting a sharp increase in validation FID at higher learning rates. 
This may be due to the UNet architecture's blocks of varying width, which amplify sensitivity under FFT.
Additional results for \initA{} with $\alpha=1$ are in Appendix~\ref{app:sec:stable-diffusion-results-lr-sweep}.

\section{Conclusion and Limitations}
This work proposes $\mu$A (Maximal-Update Adaptation), a theoretical framework for LoRA finetuning that characterizes how the optimal learning rate scales with model width and adapter rank to achieve maximal feature updates. 
Using this framework, we identify a LoRA configuration that enables learning rate transfer between LoRA and FFT. 
Comprehensive experiments validate the $\mu$A scaling rules and demonstrate effective learning rate transfer in practice.

Our work has two main limitations. First, our theory is asymptotic and can be sensitive to noise for finite $(n, r)$. 
Second, our analysis establishes a necessary condition for learning rate transfer; future work is needed to identify and prove sufficient conditions for a more complete theory. 
Nevertheless, our experiments show strong agreement between theoretical predictions and empirical results.

\newpage
\section*{Acknowledgment}
This work used GPU servers at Delta AI (NCSA) through allocation \#CIS250883 from the Advanced Cyberinfrastructure Coordination Ecosystem: Services \& Support (ACCESS) program \cite{access}, which is supported by U.S. National Science Foundation grants \#2138259, \#2138286, \#2138307, \#2137603, and \#2138296.

\bibliography{reference}
\bibliographystyle{icml2026}

\newpage
\appendix
\onecolumn

\section{Theoretical Proofs and Derivations for $\mu$A}
\label{app:theory}

\paragraph{Purpose and scope.}
This appendix provides the detailed derivations supporting the theories in the main text.
We analyze the dynamics of LoRA finetuning in the limit where both the network width $n$ and the LoRA rank $r$ grow to infinity. 
Our first goal is to ensure that the output features remain bounded throughout finetuning---that is, they do not explode as $n$ and $r$ increase. 
Achieving this requires scaling the learning rate appropriately with $n$ and $r$. 
Our second goal is to ensure that the feature updates remain non-vanishing, so that learning proceeds efficiently and does not stall as the network grows.

\subsection{Preliminaries and Notation Recap}
\label{app:theory:prelims}

We begin by briefly restating the notation used in the main text to make the appendix self-contained.

\subsubsection{Asymptotic notation and scaling conventions}
\label{app:theory:prelims:asymptotic}

\begin{definition}[Asymptotic notation]
\label{def:app:asymptotic_notion}
For sequences $c_{n,r}\in\R$ and $d_{n,r}\in\R_{+}$, we write $c_{n,r}=\BigO(d_{n,r})$
(resp.\ $c_{n,r}=\Omega(d_{n,r})$) if there exist constants $\kappa>0$ and $N\in\mathbb{N}$ such that for all
$n,r> N$ with $r\leq n$, we have $|c_{n,r}|< \kappa d_{n,r}$ (resp.\ $|c_{n,r}|> \kappa d_{n,r}$).
We write $c_{n,r}=\Theta(d_{n,r})$ if both bounds hold.
This notation extends element-wise to vectors: for $c_{n,r}=(c_{n,r}^i)_{i=1}^k\in\R^k$ and $d_{n,r}=(d_{n,r}^i)_{i=1}^k\in\R_+^k$, we write
$c_{n,r}=\BigO(d_{n,r})$ when $c_{n,r}^i=\BigO(d_{n,r}^i)$ for all $i\in[k]$, and similarly for $\Omega$ and $\Theta$.
For random variables, the notation is understood in the second-moment sense:
$c_{n,r}=\BigO(d_{n,r})$ means $\bigl(\E|c_{n,r}|^2\bigr)^{1/2}=\BigO(d_{n,r})$.
\end{definition}

\paragraph{Fixed training regime.}
All results are proved for a fixed training horizon $T$ that does not scale with $(n,r)$.

\subsubsection{LoRA layer, features, and the per-step decomposition}
\label{app:theory:prelims:lora-features}

\begin{definition}[Low-Rank Adaptation (LoRA)]
\label{def:app:lora-features}
Let $W^{\star}\in\R^{n\times n}$ be a pretrained weight matrix that remains fixed during finetuning. 
LoRA parameterizes the finetuned weights as
\[
W \;=\; W^{\star} + \alpha BA,
\]
where $B\in\R^{n\times r}$ and $A\in\R^{r\times n}$ are trainable factors with rank $r\ll n$ and
$\alpha\in\R$ is a scalar multiplier.
\end{definition}

For a given layer augmented with a LoRA adaptor, we use $\Zin$ and $\Zout$ to denote the input and output of the layer, respectively. Specifically, we can write the layer operation as $\Zout = W^* \Zin + \alpha BA\Zin$. In addition, we use $d\Zout$ to denote the backpropagated gradient of the loss with respect to $\Zout$, i.e., $d\Zout = {\partial\Loss}/{\partial\Zout}$. Next, we define the following intermediate and output features.

\begin{definition}[LoRA features]
\label{def:app:lora-feature}
For a layer input $\Zin$, define the intermediate and output LoRA features
\[
Z_A := A\Zin \in \R^{r},
\qquad
Z_B := \alpha BZ_A = \alpha BA\Zin \in \R^{n}.
\]
\end{definition}

For $t\le T$, we use superscripts to denote LoRA features at step $t$, i.e. $Z_A^t$, $Z_B^t$ and subscripts to denote weights at step $t$, i.e. $A_t$, $B_t$.
To isolate the contribution of individual LoRA layers to feature learning, we assume that only a single LoRA layer is trainable while all other layers remain fixed. Under this assumption, the layer input $\Zin$ 
remains unchanged across finetuning steps. After step $t$, the update to $Z_B$ decomposes as:

\begin{equation}
\label{eq:app:delta-decomposition}
    \Delta Z_B^t = \underbrace{\alpha B_{t-1} \Delta Z_A^t}_{\delta_t^1} + \underbrace{\alpha \Delta B_t Z_A^{t-1}}_{\delta_t^2} + \underbrace{\alpha \Delta B_t \Delta Z_A^t}_{\delta_t^3}
\end{equation}

Next, we give the formal definition of the initialization schemes considered in this work.
\begin{definition}[Initialization schemes for LoRA]
\label{def:app:init-schemes}
To initialize the LoRA update as a no-op, we consider:
\begin{itemize}
\item \textbf{\initA{}:} $A_0$ has i.i.d.\ entries $\mathcal{N}(0,1/n)$ and $B_0=0$.
\item \textbf{\initB{}:} $B_0$ has i.i.d.\ entries $\mathcal{N}(0,1/r)$ and $A_0=0$.
\end{itemize}
\end{definition}

\subsubsection{Stability and non-vanishing learning criteria}
\label{app:theory:prelims:stability}

\begin{definition}[Feature stability]
\label{def:app:feature-stability}
Fix a finite number of finetuning steps $T$ independent of $(n,r)$. We say LoRA finetuning is \emph{stable} if for all steps $t\le T$, we have $Z_B^{t} = \BigO(1)$ in the joint limit $n,r\to\infty$ with $r\leq n$.
\end{definition}

\begin{lemma}[Telescoping bound]
\label{lem:app:telescoping}
Fix $t\le T$. If $Z_B^0 = \BigO(1)$ and $\Delta Z_B^s = \BigO(1)$ for all $s\le t$, then $Z_B^t = \BigO(1)$.
\end{lemma}
\begin{proof}
By telescoping, $Z_B^t = Z_B^0 + \sum_{s=1}^{t} \Delta Z_B^s$. For a fixed $t$, this is a sum of $Z_B^0 = \BigO(1)$ and finitely many $\BigO(1)$ terms, hence $\BigO(1)$.
\end{proof}

Since the LoRA adapter is typically initialized as a no-op (i.e., $Z_B^0 = 0$), Lemma~\ref{lem:app:telescoping} implies that bounding the feature updates $\Delta Z_B^s = \BigO(1)$ is sufficient to ensure bounded features $Z_B^t = \BigO(1)$. Thus, verifying stability (Definition~\ref{def:app:feature-stability}) reduces to verifying that the updates remain bounded. However, to avoid trivial updates, we also require $\Delta Z_B^s$ to be of non-vanishing magnitude. This leads to the following definition:

\begin{definition}[Stable feature learning with LoRA]
\label{def:app:stable-feature-learning}
We say LoRA induces \emph{stable feature learning} if the dynamics are stable (Definition~\ref{def:app:feature-stability}), and $\Delta Z_B^t = \Theta(1)$ for all $t \le T$.
\end{definition}

As shown in ~\cref{eq:app:delta-decomposition}, the feature update $\Delta Z_B^t$ decomposes into three terms $(\delta_t^i)_{i \in \{1,2,3\}}$. 
A \emph{sufficient} condition for feature stability is the term-wise bound $\delta_t^i=\BigO(1)$ for all $i\in\{1,2,3\}$, which avoids relying on cancellations between terms.
Stable feature learning additionally requires $\Delta Z_B^t=\Theta(1)$; under the term-wise control above, this in particular implies that at least one term must satisfy $\delta_t^i=\Omega(1)$.
Efficiency imposes a stronger requirement: each term must be exactly $\Theta(1)$, ensuring that both weight matrices $A$ and $B$ contribute meaningfully to the update. We formalize this as follows:

\begin{definition}[Efficient Feature Learning with LoRA] 
\label{def:app:efficient-feature-learning}
We say that LoRA finetuning is efficient if it induces stable feature learning (Definition~\ref{def:app:stable-feature-learning}), and for all finetuning steps $t\le T$, we have:
\begin{equation*}
\delta_t^i = \Theta(1), \quad i\in\{1,2,3\}.
\end{equation*}
\end{definition}

\subsection{Technical Setup}

\subsubsection{Assumptions on Training Dynamics}
\label{app:theory:assumption}

\begin{assumption}[Bounded forward/backward signals]
\label{assump:app:bounded-signals}
    For any fixed step $t\le T$, the layer input and the backpropagated output gradient satisfy
    \[
    \Zin = \Theta(1),
    \qquad
    d\Zout := \frac{\partial \mathcal{L}}{\partial \Zout} = \Theta(1).
    \]
    This ensures that features and gradients neither explode nor vanish at any step as $(n,r)\to\infty$.
\end{assumption}

\begin{assumption}[Optimizer abstraction]
\label{assump:app:optimizer}
    We use a simplified variant of Adam~\citep{Kingma14Adam} without momentum, reducing it to SignSGD~\citep{Bernstein18SIGNSGD}, which gives element-wise sign updates:
    \[
    g_{A}^t := \operatorname{sign}\!\Big(\frac{\partial \mathcal{L}_t}{\partial A_{t-1}}\Big),
    \qquad
    g_{B}^t := \operatorname{sign}\!\Big(\frac{\partial \mathcal{L}_t}{\partial B_{t-1}}\Big),
    \]
    and parameter updates
    \[
    A_t = A_{t-1} - \eta\, g_{A}^t,
    \qquad
    B_t = B_{t-1} - \eta\, g_{B}^t,
    \]
    where $\operatorname{sign}(\cdot)$ is applied element-wise (with $\operatorname{sign}(0)=0$) and $\eta$ denotes the learning rate.
\end{assumption}

At a specific training step $t$, we would have:
\begin{align}
    \frac{\partial\Loss_t}{\partial A_{t-1}} &= \alpha (B_{t-1}^\top d\Zout^{t-1}) \otimes \Zin \label{eq.derivative_A}\\
    \frac{\partial\Loss_t}{\partial B_{t-1}} &= \alpha d\Zout^{t-1} \otimes Z_A^{t-1}
\end{align}
where $\otimes$ denotes outer product.

Following \cite{hayou2024impact}, we also consider the single-sample loss setting. This simplifies the theoretical analysis and leads to closed-form expressions.

\subsubsection{Auxiliary lemmas}
\label{app:theory:aux-lemmas}

This section contains auxiliary lemmas that are used as intermediate steps in the following analysis.

\begin{lemma}[Recursive Formulas] 
\label{lem:app:recursive}
For $t$ fixed, the asymptotic dynamics of $Z_A^t$ and $B_t$ follow the recursive formula:
\begin{align}
    Z_A^t &= \max(\Theta(Z_A^{t-1}), \Theta(\eta n )) \label{eq.update_A} \\
    B_t &= \max(\Theta(B_{t-1}), \Theta(\eta)) \label{eq.update_B}
\end{align}
\end{lemma}
\begin{proof}
    Let $S^t=\alpha B_{t-1}^\top d\Zout^{t-1}$. We, from Eq.~\ref{eq.derivative_A}, have:
    \begin{align*}
        g_A^t &= \sign(S^t \otimes \Zin) \\
        &= \sign(S^t)\otimes \sign(\Zin)
    \end{align*}
    where in the second equation we use the fact that both $S^t$ and $\Zin$ are both vectors. Hence, we obtain:
    \begin{align*}
        g_A^t \Zin &= (\sign(S^t)\otimes \sign(\Zin)) \Zin \\
        &= (\sign(\Zin)^\top \Zin)\sign(S^t) \\
        &= \Theta(n)
    \end{align*}
    where we used the fact that $\sign(\Zin)^\top \Zin=\|\Zin\|_1=\Theta(n)$ and $\sign(\cdot)=\Theta(1)$ due to the $\sign$ operator. Combining, we obtain Eq.~\ref{eq.update_A}:
    \begin{align*}
        Z_A^t &= \Theta( (A_{t-1}- \eta g_A^{t})\Zin ) \\
        &= \Theta(Z_A^{t-1} - \eta g_A^{t}\Zin) \\
        &= \max(\Theta(Z_A^{t-1}), \Theta(\eta n))
    \end{align*}
    Similarly, we have the following result for $B_t$:
    \begin{align*}
        B_t &= \Theta(B_{t-1} - \eta g_B^{t}) \\
        &= \Theta(B_{t-1} - \eta \sign(\frac{\partial \Loss_t}{\partial B_{t-1}})) \\
        &= \max(\Theta(B_{t-1}), \Theta(\eta))
    \end{align*}
\end{proof}

We take the next lemma from~\cite{hayou2025optimal}, and we restate it and show its proof to make this proof self-contained.

\begin{lemma}[Stein’s lemma for a sign–Gaussian product] Let $(Z, G)$ be a bi-variate centered Gaussian vector with $\Var(Z) = \Var(G) = 1$ and correlation $\rho = \E[ZG]$. Then
\begin{equation*}
    \E[\sign(Z)G] = \rho\sqrt{\frac{2}{\pi}} 
\end{equation*}
\label{lem:app:stein-sign}
\end{lemma}
\begin{proof}
    For a pair of standard Gaussians with correlation $\rho$, we have:
    \begin{equation*}
        G = \rho Z + \sqrt{1-\rho^2}G'
    \end{equation*}
    where $G'\sim \Normal(0,1)$ is another standard Gaussian independent of $Z$. Putting it into the expectation:
    \begin{align*}
        \E[\sign(Z)G] &= \E[\sign(Z)(\rho Z+\sqrt{1-\rho^2}G')] \\
        &= \E[\rho\sign(Z)Z] + \sqrt{1-\rho^2}\cdot\E[\sign(Z)G'] \\
        &= \rho\E[\sign(Z)Z] \\
        &= \rho \E\abs{Z}
    \end{align*}
    Since $Z\sim\Normal(0,1)$, $\E\abs{Z}=\sqrt{\frac{2}{\pi}}$. Plugging this in, we obtain:
    \begin{equation*}
        \E[\sign(Z)G] = \rho \sqrt{\frac{2}{\pi}} 
    \end{equation*}
\end{proof}

\begin{corollary}[General sign–Gaussian product]
\label{cor:app:stein-sign-general}
Let $(Z, G)$ be a bi-variate centered Gaussian vector with $\Var(Z) = \sigma_Z^2$ and $\Var(G) = \sigma_G^2$ and correlation $\rho = \frac{\E[ZG]}{\sigma_Z\sigma_G}$. Then
\begin{equation*}
    \E[\sign(Z)G] = \rho\sigma_G \sqrt{\frac{2}{\pi}} 
\end{equation*}
\begin{proof}
    Define the standardized variables as:
    \begin{equation*}
        U := \frac{Z}{\sigma_Z},\quad V:=\frac{G}{\sigma_G}
    \end{equation*}
    Then $(U,V)$ is a centered bi-variate Gaussian with
    \begin{equation*}
        \Var(U)=\Var(V)=1, \quad \rho_{UV}=Corr(U,V)=\E[UV]=\frac{\E[ZG]}{\sigma_Z\sigma_G}=\rho
    \end{equation*}
    Thus, $(U,V)$ satisfies the assumptions of Lemma~\ref{lem:app:stein-sign}. Applying Lemma~\ref{lem:app:stein-sign} gives:
    \begin{equation*}
        \E[\sign(U)V] = \rho\sqrt{\frac{2}{\pi}}
    \end{equation*}
    Note that $\sign(U)=\sign(Z)$ since $\sigma_Z>0$ and $V=\frac{G}{\sigma_G}$. Thus,
    \begin{align*}
        \E[\sign(Z)G] &= \E[\sign(U)\sigma_G V] \\
        &= \sigma_G \E[\sign(U)V] \\
        &= \rho\sigma_G\sqrt{\frac{2}{\pi}}
    \end{align*}
\end{proof}

\end{corollary}

With these tools in place, we can now analyze each contribution $\delta_t^i,\ i\in\{1,2,3\}$ to the feature increment $\Delta Z_B^t$ in turn.

\subsection{Scaling of Per-Step Update Terms}
\label{app:theory:delta-terms}

Recall from Eq.~\eqref{eq:app:delta-decomposition} that the per-step LoRA feature increment decomposes into three contributions.
In this part we derive the asymptotic scaling of each term for a fixed step $t$, expressed in terms of the learning rate $\eta$, LoRA multiplier $\alpha$, network width $n$, rank $r$, and the current scales of $A_{t-1}$ and $B_{t-1}$.

\begin{lemma}[Asymptotic Dynamics of $\delta_t^1$]
\label{lem:app:delta1}
For $t$ fixed, the asymptotic dynamics of $\delta_t^1$ follows the following formula:
\begin{equation}
    \delta_t^1 = \Theta(\eta\alpha n)\cdot \Theta(\sqrt{r +\frac{2}{\pi} \cdot \frac{r(r-1)}{n}}) \cdot \Theta(B_{t-1})
\end{equation}
\end{lemma}
\begin{proof}
    Substituting $\Delta Z_A^t=- \eta g_A^{t}\Zin$ into $\delta_t^1$, we have:
    \begin{align*}
        \delta_t^1 &= \alpha B_{t-1} \Delta Z_A^t \\
        &= \alpha B_{t-1} (-\eta \cdot \sign(\frac{\partial \Loss_t}{\partial A_{t-1}})\Zin) \\
        &= (-\eta\alpha) \cdot (\sign(\Zin)^\top \Zin)\cdot B_{t-1} \cdot \sign(\alpha B_{t-1}^\top d\Zout^{t-1})
    \end{align*}
    {Here, we analyze $B_{t-1} \cdot \sign(\alpha B_{t-1}^\top d\Zout^{t-1})$. For the ease of analysis, we drop the subscript in $B_{t-1}$ and instead use $B$ with $B=\Theta(\beta)$. In addition, we use $h=\alpha d\Zout^{t-1}\in\R^{n}$ with $\Theta(h)=\Theta(\alpha)$. Let $X=B\cdot \sign(B^Th)$, we aim to analyze the asymptotic dynamic of $X$ by taking an arbitrary element $X_k$ and analyze $\E[X_k^2]$. To simplify the analysis, we treat $B\in\mathbb{R}^{n\times r}$ as an i.i.d. Gaussian matrix with entries $B_{ij}\stackrel{\text{i.i.d.}}{\sim}\mathcal{N}(0,\beta^{2})$, and we assume $h\in\mathbb{R}^{n}$ is independent of $B$ with entries $h_i\stackrel{\text{i.i.d.}}{\sim}\mathcal{N}(0,\alpha^{2})$.}

    \noindent Fixing an index $k$, we have:
    \begin{equation*}
        X_k = \sum_{a=1}^r B_{ka}\cdot\sign(u_a)
    \end{equation*}
    where $u_{a}=\sum_{j=1}^{n} B_{ja}h_j$. Then we have:
    \begin{align*}
        \E[X_k^2] &= \E[\sum_{a=1}^r B_{ka}^2\sign(u_a)^2 + 2\sum_{a<a'} B_{ka}\sign(u_a)B_{ka'}\sign(u_{a'})] \\
        &= r\beta^2 + 2\sum_{a<a'}\E[B_{ka}\sign(u_a)B_{ka'}\sign(u_{a'})]
    \end{align*}
    Given $h$, the pair $(u_a, B_{ka})$ is Gaussian with $\E[u_a \mid h]=0$ and $\E[B_{ka} \mid h]=0$, and the correlation is
    \begin{align*}
        \rho_k &= Corr(B_{ka}, u_a \mid h) \\
        &= \frac{\E[B_{ka}u_a \mid h]}{\sqrt{\Var[B_{ka}\mid h]\Var[u_a \mid h]}} \\
        &= \frac{h_k\beta^2}{\sqrt{\beta^2\cdot (\sum_{j=1}^n h_j^2\beta^2)}} \\
        &= \frac{h_k}{\sqrt{\sum_{j=1}^{n} h_j^2}}
    \end{align*}
    Using Corollary~\ref{cor:app:stein-sign-general}, we have 
    \begin{align*}
        \mu &:= \E[B_{ka}\sign(u_a)\mid h] \\
        &= \rho_k \beta \sqrt{\frac{2}{\pi}}
    \end{align*}
    By the same argument, $\E[B_{ka'}\sign(u_{a'}) \mid h] = \rho_k \beta \sqrt{\frac{2}{\pi}}$. Moreover, conditionally on $h$, the random variables $B_{ka}\sign(u_a)$ and $B_{ka'}\sign(u_{a'})$ are independent for $a \neq a'$, since they depend on distinct columns of $B$. This yields:
    \begin{align*}
        \E[B_{ka}\sign(u_a)B_{ka'}\sign(u_{a'})|h] = \rho_k^2 \beta^2 \frac{2}{\pi} 
    \end{align*}
    Therefore, 
    \begin{align*}
        \E[B_{ka}\sign(u_a)B_{ka'}\sign(u_{a'})] = \beta^2\frac{2}{\pi} \E[\rho_k^2]
    \end{align*}
    Denote $h_k = \alpha q_k$ for all $k$ where $q_k\iid\Normal(0,1)$, then we have:
    \begin{align*}
        \E[\rho_k^2]= \E[\frac{q_k^2}{{\sum_{j=1}^{n} q_j^2}}]
    \end{align*}
    Let $S=\sum_{j=1}^{n} q_j^2$. Given that $q_j$ are independent, standard normal random variables, we get $S\sim\chi^2(n)$. Furthermore, by symmetry, we have:
    \begin{equation*}
        \sum_{k=1}^n \E[\frac{q_k^2}{S}] = \E[\frac{\sum_{k=1}^n q_k^2}{S}] = 1
    \end{equation*}
    Therefore, $\E[\rho_k^2]=\frac{1}{n}$, and we obtain:
    \begin{equation*}
        \E[B_{ka}\sign(u_a)B_{ka'}\sign(u_{a'})] = \beta^2\cdot\frac{2}{\pi}\cdot \frac{1}{n}
    \end{equation*}
    Plugging this term:
    \begin{align*}
        \E[X_k^2] &= r\beta^2 + 2\sum_{a<a'}\E[B_{ka}\sign(u_a)B_{ka'}\sign(u_{a'})] \\
        &= r\beta^2 + 2 \cdot \frac{r(r-1)}{2} \cdot \beta^2\cdot\frac{2}{\pi}\cdot\frac{1}{n} \\
        &= \beta^2(r + \frac{2}{\pi} \cdot \frac{r(r-1)}{n})
    \end{align*}
    Hence, $X_k=\sqrt{\E[X_k^2]}=\Theta(\beta)\cdot\Theta(\sqrt{r +\frac{2}{\pi} \cdot \frac{r(r-1)}{n})}$. Combining things together, we get 
    \begin{align*}
        \delta_t^1 &= \Theta(-\eta\alpha) \cdot \Theta(n) \cdot \Theta(B_{t-1})\cdot \Theta(\sqrt{r +\frac{2}{\pi} \cdot \frac{r(r-1)}{n}}) \\
        &= \Theta(\eta\alpha n)\cdot \Theta(B_{t-1})\cdot \Theta(\sqrt{r +\frac{2}{\pi} \cdot \frac{r(r-1)}{n}})
    \end{align*}
\end{proof}

\begin{remark}[Simplified Dynamics for $\delta_t^1$]
\label{rem:app:delta1-simplified}
When $r \leq n$, Lemma~\ref{lem:app:delta1} reduces to
\begin{equation}
    \delta_t^1 = \Theta(\eta \alpha n \sqrt{r}) \cdot \Theta(B_{t-1}).
\end{equation}
\end{remark}

\begin{lemma}[Asymptotic Dynamics of $\delta_t^2$]
\label{lem:app:delta2}
For $t$ fixed, the asymptotic dynamics of $\delta_t^2$ follows the following formula:
\begin{equation}
    \delta_t^2 = \Theta(\eta \alpha r) \cdot \Theta(Z_A^{t-1})
\end{equation}
\end{lemma}
\begin{proof}
    Substituting $\Delta B_t = -\eta g_B^{t}$ and $g_B^{t}=\sign(\frac{\partial \Loss_t}{\partial B_{t-1}})$ into $\delta_t^2$, we have:
    \begin{align*}
        \delta_t^2 &= \alpha \Delta B_t Z_A^{t-1} \\
        &= \alpha (-\eta\cdot \sign(\frac{\partial \Loss_t}{\partial B_{t-1}}))Z_A^{t-1} \\
        &= (-\eta \alpha) \sign( \alpha d\Zout^{t-1} \otimes Z_A^{t-1}) Z_A^{t-1}\\
        &= (-\eta \alpha) (\sign(\alpha d\Zout^{t-1}) \otimes \sign(Z_A^{t-1})) Z_A^{t-1} \\
        &= (-\eta \alpha) (\sign(Z_A^{t-1})^\top Z_A^{t-1}) \sign(\alpha d\Zout^{t-1})
    \end{align*}
    Therefore, we can get the asymptotic behavior of $\delta_t^2$ as:
    \begin{align*}
        \delta_t^2 &= \Theta(-\eta\alpha) \cdot \Theta(r) \cdot \Theta(Z_A^{t-1}) \cdot \Theta(1) \\
        &= \Theta(\eta \alpha r) \cdot \Theta(Z_A^{t-1})
    \end{align*}
    where the extra $r$ comes from the fact that we are summing over $r$ positive terms in $\sign(Z_A^{t-1})^\top Z_A^{t-1}$.
\end{proof}

\begin{lemma}[Asymptotic Dynamics of $\delta_t^3$]
\label{lem:app:delta3}
For $t$ fixed, the asymptotic dynamics of $\delta_t^3$ follows the following formula:
\begin{equation}
    \delta_t^3 = \Theta(\alpha \eta^2 n r^{1/2})
\end{equation}
\end{lemma}
\begin{proof}
    Substituting $\Delta B_t = -\eta g_B^{t}$  and $\Delta Z_A^t=- \eta g_A^{t}\Zin$ into $\delta_t^3$, we have:
    \begin{align*}
        \delta_t^3 &= \alpha \Delta B_t \Delta Z_A^t \\
        &= \alpha \cdot (-\eta g_B^{t}) \cdot (-\eta g_A^{t}\Zin) \\
        &= (\alpha \eta^2) \cdot \sign(\alpha d\Zout^{t-1} \otimes Z_A^{t-1}) \cdot \sign(\alpha (B_{t-1}^\top d\Zout^{t-1}) \otimes \Zin) \cdot \Zin \\
        &= (\alpha \eta^2) \cdot (\sign(\Zin)^\top \Zin)\cdot \sign(\alpha d\Zout^{t-1} \otimes Z_A^{t-1}) \cdot \sign(\alpha B_{t-1}^\top d\Zout^{t-1}) \\
        &= (\alpha \eta^2) \cdot (\sign(\Zin)^\top \Zin) \cdot (\sign(Z_A^{t-1})^\top\sign(\alpha B_{t-1}^\top d\Zout^{t-1})) \cdot \sign(\alpha d\Zout^{t-1})
    \end{align*}
    {Here, we assume that elements in $\sign(Z_A^{t-1})$ and $\sign(\alpha B_{t-1}^\top d\Zout^{t-1})$ are independent. Under this assumption, we have: 
    \begin{equation*}
        \sign(Z_A^{t-1})^\top\sign(\alpha B_{t-1}^\top d\Zout^{t-1})=\Theta(r^{1/2})
    \end{equation*}} 
    Therefore, we have:
        \begin{align*}
        \delta_t^3 &= \Theta(\alpha \eta^2)\cdot \Theta(n) \cdot \Theta(r^{1/2})\cdot \Theta(1) \\
        &=\Theta(\alpha \eta^2 n r^{1/2})
    \end{align*}
\end{proof}

With the foregoing lemmas established, the proof of the main theorems becomes straightforward:
for each initialization scheme, we (i) identify the induced initial scales of $(Z_A^0, B_0)$,
(ii) propagate these scales through Lemma~\ref{lem:app:recursive},
(iii) plug into Lemma~\ref{lem:app:delta1}, Lemma~\ref{lem:app:delta2}, and Lemma~\ref{lem:app:delta3},
and (iv) derive the learning rate scaling that yields stable, non-vanishing feature updates $\Delta Z_B^t$, i.e., $\Delta Z_B^t=\Theta(1)$.

\subsection{\initA{}: Theorems and Corollaries}
\label{app:theory:initA}

This section provides the derivations for the \initA{} regime in which $A_0$ is randomly initialized and $B_0=0$.
We follow the same structure as the main text: first characterize the scale of $\Delta Z_B^t$, then impose the stable feature learning condition to obtain the learning rate scaling.

\begin{theorem}[\initA{}]
\label{thm:app:initA}
Under Assumptions~\ref{assump:app:bounded-signals} and~\ref{assump:app:optimizer}, for any fixed $t\le T$ with \initA{},
\begin{equation}
\label{app:eq:initA-main}
\Delta Z_B^t
\;=\;
\max\{\Theta(\eta \alpha r),\ \Theta(\eta^2 \alpha r n)\}.
\end{equation}
\end{theorem}
\begin{proof}
    With \initA{}, we have $B_0=0$ and $Z_A^0=A_0\Zin=\Theta(1)$. Thus, for all $t$, we have:
    \begin{align*}
        Z_A^t &= \max(\Theta(1), \Theta(\eta n)) \\
        B_t &= \Theta(\eta)
    \end{align*}
    Thus, we have:
    \begin{align*}
        \delta_t^1 &= \Theta(\eta^2\alpha n)\cdot \Theta(\sqrt{r +\frac{2}{\pi} \cdot \frac{r(r-1)}{n}}) \\
        &= \Theta(\eta^2 \alpha n r^{1/2})\\
        \delta_t^2 &= \Theta(\eta \alpha r) \cdot \max(\Theta(1), \Theta(\eta n)) \\
        &= \max(\Theta(\eta \alpha r), \Theta(\eta^2 \alpha r n)) \\
        \delta_t^3 &= \Theta(\alpha \eta^2 n r^{1/2})
    \end{align*}
    Therefore,
    \begin{align*}
        \Delta Z_B^t &= \max(\Theta(\delta_t^1), \Theta(\delta_t^2), \Theta(\delta_t^3)) \\
        &= \max(\Theta(\eta^2 \alpha r^{1/2}n), \Theta(\eta \alpha r), \Theta(\eta^2 \alpha r n)) \\
        &= \max(\Theta(\eta \alpha r), \Theta(\eta^2 \alpha r n))
    \end{align*}
\end{proof}

Imposing $\Delta Z_B^t=\Theta(1)$ and solving for the learning rate scaling in terms of $(n,r)$ and $\alpha=r^{-\gamma}$ with $\gamma\in[0,1]$ yields the following unified scaling laws for \initA{}:

\begin{corollary}[Unified scaling rules for \initA{}]
\label{cor:app:initA-scaling}
Let $\alpha = r^{-\gamma}$ for $\gamma\in[0,1]$. Under \initA{}, achieving stable feature learning
(i.e., $\Delta Z_B^t=\Theta(1)$ for fixed $t\le T$) requires
\[
\eta = \Theta(n^{-1/2} r^{-{(1-\gamma)}/{2}}).
\]
With this choice,
\begin{align*}
\delta_t^1&=\Theta(r^{-1/2}),
\\
\delta_t^2&=\Theta(1),
\\
\delta_t^3&=\Theta(r^{-1/2}),
\end{align*}

and the intermediate feature scales as
\[
\Theta(Z_A^t)=\Theta(n^{1/2}\, r^{-{(1-\gamma)}/{2}}).
\]
In particular:
\begin{center}
\renewcommand{\arraystretch}{1.3}
\begin{tabular}{c|c|c|c}
    $\gamma$ & $\alpha$ & $\Theta(\eta)$ & $\Theta(Z_A^t)$ \\
    \hline
    $0$ & $1$ & $n^{-{1}/{2}} r^{-{1}/{2}}$ & $n^{{1}/{2}} r^{-{1}/{2}}$ \\
    $\frac{1}{2}$ & $r^{-{1}/{2}}$ & $n^{-{1}/{2}} r^{-{1}/{4}}$ & $n^{{1}/{2}} r^{-{1}/{4}}$ \\
    $1$ & $r^{-1}$ & $n^{-{1}/{2}}$ & $n^{{1}/{2}}$
\end{tabular}
\end{center}
\end{corollary}

\begin{proof}
From ~\cref{app:eq:initA-main}, under \initA{} we have
\begin{equation*}
    \Delta Z_B^t = \max\left(\Theta(\eta \alpha r),\, \Theta(\eta^2 \alpha r n)\right).
\end{equation*}
Substituting $\alpha = r^{-\gamma}$ yields
\begin{equation*}
    \Delta Z_B^t = \max\left(\Theta(\eta r^{1-\gamma}),\, \Theta(\eta^2 n r^{1-\gamma})\right).
\end{equation*}

\noindent To achieve $\Delta Z_B^t = \Theta(1)$, we require
\begin{equation*}
    \eta = \min\left(\Theta(r^{-(1-\gamma)}),\, \Theta(n^{-{1}/{2}} r^{-{(1-\gamma)}/{2}})\right).
\end{equation*}

\noindent Since $r \leq n$, we have $r^{-(1-\gamma)} \geq n^{-{1}/{2}} r^{-{(1-\gamma)}/{2}}$ for all $\gamma \in [0,1]$. Thus,
\begin{equation*}
    \eta = \Theta(n^{-{1}/{2}} r^{-{(1-\gamma)}/{2}}).
\end{equation*}

\textit{Feature update terms.} Substituting into the expressions from Lemmas~\ref{lem:app:delta1},~\ref{lem:app:delta2},~\ref{lem:app:delta3}:
\begin{itemize}
    \item From Lemma~\ref{lem:app:delta1} with $B_{t-1} = \Theta(\eta) = \Theta(n^{-{1}/{2}} r^{-{(1-\gamma)}/{2}})$:
    \begin{align*}
        \delta_t^1 &= \Theta(\eta \alpha n \sqrt{r}) \cdot \Theta(B_{t-1}) \\
        &= \Theta(n^{-{1}/{2}} r^{-{(1-\gamma)}/{2}} \cdot r^{-\gamma} \cdot n \cdot r^{{1}/{2}} \cdot n^{-{1}/{2}} r^{-{(1-\gamma)}/{2}}) \\
        &= \Theta(r^{-{1}/{2}}).
    \end{align*}
    
    \item From Lemma~\ref{lem:app:delta2}, using $Z_A^{t-1} = \max(\Theta(1), \Theta(\eta n)) = \Theta(\eta n)$:
    \begin{align*}
        \delta_t^2 &= \Theta(\eta \alpha r) \cdot \Theta(Z_A^{t-1}) \\
        &= \Theta(n^{-{1}/{2}} r^{-{(1-\gamma)}/{2}} \cdot r^{-\gamma} \cdot r \cdot n^{-{1}/{2}} r^{-{(1-\gamma)}/{2}} \cdot n) \\
        &= \Theta(1).
    \end{align*}
    
    \item From Lemma~\ref{lem:app:delta3}, we have:
    \begin{align*}
        \delta_t^3 &= \Theta(\alpha \eta^2 n r^{{1}/{2}}) \\
        &= \Theta(r^{-\gamma} \cdot n^{-1} r^{-(1-\gamma)} \cdot n \cdot r^{{1}/{2}}) \\
        &= \Theta(r^{-{1}/{2}}).
    \end{align*}
\end{itemize}
Thus, under \initA{}, the dominant contribution comes from $\delta_t^2$, while $\delta_t^1$ and $\delta_t^3$ vanish as rank $r\to\infty$.

\textit{Intermediate features.} From the recursive formula (Lemma~\ref{lem:app:recursive}):
\begin{equation*}
    Z_A^t = \max\left(\Theta(1), \Theta(\eta n)\right) = \Theta(\eta n) = \Theta(n^{{1}/{2}} r^{-{(1-\gamma)}/{2}}).
\end{equation*}
\end{proof}

\subsection{\initB{}: Theorems and Corollaries}
\label{app:theory:initB}

This part deals with \initB{}, in which $B_0$ is randomly initialized and $A_0=0$.
As in the main text, we first establish the scale of $\Delta Z_B^t$ and then specialize to common multiplier configurations, namely $\alpha=1$ and $\alpha=r^{-1}$

\begin{theorem}[\initB{}]
\label{thm:app:initB}
Under Assumptions~\ref{assump:app:bounded-signals} and~\ref{assump:app:optimizer}, for any fixed $t\le T$
with \initB{},
\begin{equation}
\label{eq:initB-main}
\Delta Z_B^t
\;=\;
\max\{\Theta(\eta \alpha n),\ \Theta(\eta^2 \alpha n r)\}.
\end{equation}
\end{theorem}

\begin{proof}
    With \initB{}, we have $B_0=\Theta(r^{-{1}/{2}})$ and $Z_A^0=0$. As a result, we have for all $t\leq T$:
    \begin{align*}
        Z_A^t &= \Theta(\eta n) \\
        B_t &= \max\left(\Theta(r^{-{1}/{2}}), \Theta(\eta)\right)
    \end{align*}
    Thus, we have:
    \begin{align*}
        \delta_t^1 &= \Theta(\eta\alpha n r^{{1}/{2}}) \cdot \Theta(B_{t-1}) \\
        &= \max\left(\Theta(\eta \alpha n), \Theta(\eta^2 \alpha n r^{{1}/{2}})\right) \\
        \delta_t^2 &= \Theta(\eta^2 \alpha n r) \\
        \delta_t^3 &= \Theta(\alpha \eta^2 n r^{{1}/{2}})
    \end{align*}
    Therefore,
    \begin{align*}
        \Delta Z_B^t &= \max\left(\Theta(\delta_t^1), \Theta(\delta_t^2), \Theta(\delta_t^3)\right) \\
        &= \max\left(\Theta(\eta \alpha n), \Theta(\eta^2 \alpha n r)\right)
    \end{align*}
\end{proof}

We proceed by considering the LoRA multiplier $\alpha$ case by case. We first consider the choice $\alpha=1$ (Corollary~\ref{cor:app:initB-constant}), and then the rank-dependent choice $\alpha=r^{-1}$ (Corollary~\ref{cor:app:initB-invr}).

\begin{corollary}[\initB{} with $\alpha=1$]
\label{cor:app:initB-constant}
Let $\alpha=1$. Under \initB{}, achieving stable feature learning (i.e., $\Delta Z_B^t=\Theta(1)$ for fixed $t\le T$)
requires
\[
\eta = \Theta(n^{-1}).
\]
With this choice,
\begin{align*}
\delta_t^1&=\Theta(1),
\\
\delta_t^2&=\Theta(rn^{-1}),
\\
\delta_t^3&=\Theta(r^{1/2}n^{-1}),
\end{align*}
and the intermediate feature scales as
\[
\Theta(Z_A^t)=\Theta(1).
\]
\end{corollary}
\begin{proof}
From ~\cref{eq:initB-main} with constant $\alpha$:
\begin{equation*}
    \Delta Z_B^t = \max\left(\Theta(\eta n),\, \Theta(\eta^2 n r)\right).
\end{equation*}
Obtaining $\Delta Z_B^t = \Theta(1)$ requires:
\begin{equation*}
    \eta = \min\left(\Theta(n^{-1}),\, \Theta(n^{-{1}/{2}} r^{-{1}/{2}})\right).
\end{equation*}
Since $r \leq n$ implies $n^{-1} \leq n^{-{1}/{2}} r^{-{1}/{2}}$, we have $\eta = \Theta(n^{-1})$.

\medskip
\textit{Feature update terms.}
Under \initB{}, we have $B_{t-1} =\max\left(\Theta(r^{-{1}/{2}}), \Theta(\eta)\right) = \Theta(r^{-{1}/{2}})$ and $Z_A^{t-1} = \Theta(\eta n) = \Theta(1)$. Substituting into Lemmas~\ref{lem:app:delta1},~\ref{lem:app:delta2},~\ref{lem:app:delta3}:
\begin{align*}
    \delta_t^1 
        &= \Theta(\eta \alpha n \sqrt{r}) \cdot \Theta(B_{t-1})  \\
        &= \Theta(n^{-1} \cdot 1 \cdot n \cdot r^{{1}/{2}} \cdot r^{-{1}/{2}}) \\
        &= \Theta(1), \\[4pt]
    \delta_t^2
        &= \Theta(\eta \alpha r) \cdot \Theta(Z_A^{t-1}) \\
        &= \Theta(n^{-1} \cdot r \cdot 1), \\
        &= \Theta(rn^{-1}) \\
    \delta_t^3
        &= \Theta(\alpha \eta^2 n r^{{1}/{2}}) \\
        &= \Theta(1 \cdot n^{-2} \cdot n \cdot r^{{1}/{2}}) \\
        &= \Theta({r^{1/2}}{n^{-1}}).
\end{align*}
\textit{Intermediate features.} From the recursive formula (Lemma~\ref{lem:app:recursive}):
\begin{align*}
    Z_A^t &= \Theta(\eta n) \\
    &= \Theta(1)
\end{align*}
\end{proof}

\begin{corollary}[\initB{} with $\alpha = r^{-1}$]
\label{cor:app:initB-invr}
Let $\alpha = r^{-1}$. Under \initB{}, achieving stable feature learning (i.e., $\Delta Z_B^t=\Theta(1)$ for fixed $t\le T$) requires
\begin{equation}
\eta = \min\left(\Theta(rn^{-1}),\,\Theta(n^{-1/2})\right).
\end{equation}
With this choice, the feature update terms satisfy:
\begin{align*}
\delta_t^1&=\Theta(\eta nr^{-1}) \\
\delta_t^2&=\Theta(\eta^2 n) \\
\delta_t^3&=\Theta(\eta^2 n r^{-1/2})
\end{align*}
and the intermediate feature scales as 
\begin{equation*}
    Z_A^t=\Theta(\eta n)
\end{equation*} In particular:
\begin{itemize}
\item If $r\leq\sqrt n$, we have $\eta=\Theta(rn^{-1})$ and thus obtain:
\begin{align*}
\delta_t^1&=\Theta(1) \\
\delta_t^2&=\Theta(r^2n^{-1}) \\
\delta_t^3&=\Theta(r^{{3}/{2}}n^{-1}) \\
Z_A^t&=\Theta(r)
\end{align*}
\item If $r>\sqrt{n}$, we have $\eta=\Theta(n^{-1/2})$ and thus obtain:
\begin{align*}
\delta_t^1&=\Theta(n^{1/2}r^{-1}) \\
\delta_t^2&=\Theta(1) \\
\delta_t^3&=\Theta(r^{-{1}/{2}}) \\
Z_A^t&=\Theta(n^{1/2}).
\end{align*}
\end{itemize}
\end{corollary}

\begin{proof}
Substituting $\alpha=r^{-1}$ into Theorem~\ref{thm:app:initB} gives
\begin{align*}
\Delta Z_B^t
&= \max\left(\Theta(\eta\alpha n),\,\Theta(\eta^2\alpha n r)\right) \\
&= \max\left(\Theta(\eta nr^{-1}),\,\Theta(\eta^2 n)\right).
\end{align*}
Imposing $\Delta Z_B^t=\Theta(1)$ yields
\begin{equation*}
\eta = \min\left(\Theta(rn^{-1}),\,\Theta(n^{-1/2})\right).
\end{equation*}

\textit{Feature update terms.}
Under \initB{}, we have $B_{t-1}=\max(\Theta(r^{-1/2}),\Theta(\eta))=\Theta(r^{-1/2})$ for either choice of $\eta$ and $Z_A^{t-1}=\Theta(\eta n)$.
Substituting into Lemmas~\ref{lem:app:delta1},~\ref{lem:app:delta2},~\ref{lem:app:delta3}:
\begin{align*}
\delta_t^1&=\Theta(\eta \alpha n \sqrt{r}) \cdot \Theta(B_{t-1}) \\
&= \Theta(\eta nr^{-1}), \\
\delta_t^2&=\Theta(\eta \alpha r) \cdot \Theta(Z_A^{t-1}) \\
&= \Theta(\eta^2 n), \\
\delta_t^3&=\Theta(\alpha \eta^2 n r^{{1}/{2}}) \\
&= \Theta(\eta^2 n r^{-{1}/{2}}).
\end{align*}

Next, we need to consider the relationship between the network width $n$ and the LoRA rank $r$ to determine scaling rule of $\eta$. 

When $r\leq \sqrt n$, we have $rn^{-1}\leq n^{-1/2}$ and thus $\eta=\Theta(rn^{-1})$, which incurs:
\begin{align*}
    \delta_t^1 &= \Theta(1) \\
    \delta_t^2 &= \Theta(r^2n^{-1}) \\
    \delta_t^3 &= \Theta(r^{{3}/{2}}n^{-1})
\end{align*}

When $r>\sqrt{n}$, we have $rn^{-1}> n^{-1/2}$ and thus $\eta=\Theta(n^{-1/2})$, which incurs:
\begin{align*}
    \delta_t^1 &= \Theta(n^{1/2} r^{-1}) \\
    \delta_t^2 &= \Theta(1) \\
    \delta_t^3 &= \Theta(r^{-{1}/{2}})
\end{align*}

\textit{Intermediate features.} From the recursive formula (Lemma~\ref{lem:app:recursive}):
\begin{align*}
    Z_A^t &= \Theta(\eta n)
\end{align*}
Similarly, when $r\leq \sqrt n$, we have:
\begin{equation*}
    Z_A^t = \Theta(r)
\end{equation*}
When $r> \sqrt n$, we obtain:
\begin{equation*}
    Z_A^t = \Theta(n^{1/2})
\end{equation*}

\end{proof}

\subsection{Learning Rate Scaling for Full Finetuning}
\label{app:theory:full-ft}
In this section, we analyze how the learning rate should scale for full finetuning with regard to the network width $n$. 
We use a framework similar to our previous derivations.

Consider a single linear layer of width $n$ and pretrained weights $W^\star\in\R^{n\times n}$. Under full finetuning, we update $W$ directly. Here we denote the per-step weight update by
\begin{equation*}
\Delta W_t := W_t - W_{t-1}, \qquad W_0 := W^\star.
\end{equation*}
Then, for any fixed $t\leq T$, the layer output can be written as
\begin{equation}
\Zout^t \;=\; W_t\Zin \;=\; W^\star \Zin + \sum_{s=1}^{t}\Delta W_s\,\Zin.
\end{equation}
Following Assumption~\ref{assump:app:bounded-signals}, we assume that pretraining yields bounded features $\Zin=\BigO(1)$ and $W^\star\Zin=\BigO(1)$. 
By Lemma~\ref{lem:app:telescoping}, it suffices to control each per-step increment $\Delta W_s\Zin$ for feature stability: requiring $\Delta W_s\Zin=\BigO(1)$ for all $s\le t$ implies $\sum_{s=1}^{t}\Delta W_s\Zin=\BigO(1)$. 
To avoid vanishing updates, we also impose $\Delta W_s\Zin=\Theta(1)$. 
Assuming that we are using SignSGD (Assumption~\ref{assump:app:optimizer}) and single-sample loss, we can obtain the following theorem:

\begin{theorem}[Scaling rules for full finetuning]
\label{thm:app:full_finetuning}
Under Assumptions~\ref{assump:app:bounded-signals} and~\ref{assump:app:optimizer}, for any fixed $t\le T$, with full finetuning, the learning rate should scale as
\begin{equation}
    \eta \;=\; \Theta(n^{-1}).
\end{equation}
\end{theorem}

\begin{proof}
Under SignSGD, the per-step update is $\Delta W_t = -\eta g_t$, where
\begin{equation*}
g_t := \sign\!\Big(\frac{\partial\Loss_t}{\partial W_{t-1}}\Big).
\end{equation*}
For the linear map $\Zout = W\Zin$, the gradient factorizes as:
\begin{equation}
\frac{\partial \Loss_t}{\partial W_{t-1}} \;=\; d\Zout^{t-1}\otimes \Zin,
\end{equation}
where $d\Zout^{t-1}:=\partial\Loss_t/\partial\Zout^{t-1}$ and $\otimes$ denotes the outer product. Therefore,
\begin{align*}
\Delta W_t
&= -\eta\,\sign(d\Zout^{t-1}\otimes \Zin) \\
&= -\eta\bigl(\sign(d\Zout^{t-1})\otimes \sign(\Zin)\bigr).
\label{eq:app:fullft-step-update}
\end{align*}
Multiplying by $\Zin$ yields
\begin{align}
\Delta W_t\,\Zin
&= \Theta(\eta\cdot(\sign(\Zin)^\top \Zin)\cdot\sign(d\Zout^{t-1})) \\
&= \Theta(\eta n),
\end{align}
since $\sign(\Zin)^\top\Zin=\|\Zin\|_1=\Theta(n)$ and $\sign(\cdot)=\Theta(1)$. 
Imposing $\Delta W_t\Zin=\Theta(1)$ yields:
\begin{equation*}
    \eta=\Theta(n^{-1}).
\end{equation*}
\end{proof}

From above, we show that under the same assumptions, full finetuning admits the \emph{same} width-dependent learning rate scaling $\eta=\Theta(n^{-1})$. 
At the same time, our analysis for LoRA with \initB{} and constant $\alpha=1$ shows that stable (non-exploding) and non-vanishing feature updates require a learning rate $\eta=\Theta(n^{-1})$ (Corollary~\ref{cor:app:initB-constant}).
This matching of learning rate scaling might indicate that the hyperparameter landscapes of full finetuning and LoRA with \initB{} and constant $\alpha=1$ are very close, which opens up an opportunity to do learning rate transfer between these two regimes.


\newpage
\section{Implementation and Experimental Details}
\label{app:exp-details}

In this section, we provide more details on our experimental settings.

\subsection{Model Configuration}
\label{app:shared-recipe}

All pretrained models are obtained from public model hubs and used without architectural modifications beyond (i) adding LoRA adapters and (ii) adding a linear classification head for discriminative tasks. Concretely, we use the following models:
\begin{itemize}[leftmargin=5mm]
    \item \textbf{Llama-3.2-1B}~\citep{grattafiori2024llama} (\texttt{meta-llama/Llama-3.2-1B})\footnote{\url{https://huggingface.co/meta-llama/Llama-3.2-1B}}.
    \item \textbf{Qwen2.5-3B-Instruct}~\citep{qwen2025qwen25technicalreport} (\texttt{Qwen/Qwen2.5-3B-Instruct})\footnote{\url{https://huggingface.co/Qwen/Qwen2.5-3B-Instruct}}.
    \item \textbf{RoBERTa-large}~\citep{liu2019roberta} (\texttt{FacebookAI/roberta-large})\footnote{\url{https://huggingface.co/FacebookAI/roberta-large}}.
    \item \textbf{ViT-Huge/14}~\citep{dosovitskiy2021an} (\texttt{google/vit-huge-patch14-224-in21k})\footnote{\url{https://huggingface.co/google/vit-huge-patch14-224-in21k}}.
    \item \textbf{Qwen3-VL-2B-Instruct}~\citep{bai2025qwen3vltechnicalreport} (\texttt{Qwen/Qwen3-VL-2B-Instruct})\footnote{\url{https://huggingface.co/Qwen/Qwen3-VL-2B-Instruct}}.
    \item \textbf{Stable Diffusion v1.5}~\citep{Rombach_2022_CVPR}. We use the standard \emph{v1.5} weights in diffusers format; a widely used public mirror is \texttt{stable-diffusion-v1-5/stable-diffusion-v1-5}\footnote{\url{https://huggingface.co/stable-diffusion-v1-5/stable-diffusion-v1-5}}.
    \item \textbf{Llama-3.1-8B}~\citep{grattafiori2024llama} (\texttt{meta-llama/Llama-3.1-8B})\footnote{\url{https://huggingface.co/meta-llama/Llama-3.1-8B}}.
\end{itemize}

For discriminative tasks, we attach a task-specific classification head to the final pooled representation: the pooled text representation for RoBERTa on ANLI, and the \texttt{[CLS]} token embedding for ViT on ImageNet-1K. During training, this classification head is trained jointly with either the full backbone (full finetuning) or the adapters (LoRA). Across all hyperparameter sweeps, we fix the classifier learning rate at $10^{-3}$ and only vary the backbone (or adapter) learning rate.

Unless otherwise mentioned, we use the same optimizer and learning rate schedule across all architectures. The random seed is set to $42$ across all experiments. \cref{tab:shared-optim} summarizes these settings.

\begin{table}[t]
    \centering
    \caption{Shared optimization and training settings.}
    \label{tab:shared-optim}
    \begin{tabular}{ll}
        \toprule
        \textbf{Setting} & \textbf{Value} \\
        \midrule
        Optimizer & AdamW~\citep{Loshchilov19Decoupled} \\
        AdamW momentum & $\beta_1=0.9$, $\beta_2=0.999$ \\
        AdamW epsilon & $\epsilon=10^{-8}$ \\
        Weight decay & $0.01$ \\
        Gradient clipping & global norm clipped to $1.0$ \\
        LR schedule & linear warmup for first $5\%$ of steps,  then cosine decay to $0.1\times$ peak LR \\
        Precision & bfloat16 mixed precision; TF32 \texttt{matmul} enabled \\
        LoRA dropout & 0.0 \\
        LoRA Bias & \texttt{None} \\
        \bottomrule
    \end{tabular}
\end{table}



\subsection{LoRA Configurations and PEFT Implementation}
\label{app:lora-impl}

Throughout the paper, we analyze the LoRA-augmented weight matrix formulated as follows:
\[
\mathbf{W} \leftarrow \mathbf{W} + \alpha\,\mathbf{B}\mathbf{A},
\]
where $\mathbf{A}\in\mathbb{R}^{r\times n}$ and $\mathbf{B}\in\mathbb{R}^{n\times r}$ are rank-$r$ factors, and $\alpha$ is the LoRA multiplier.

Our implementation uses the PEFT library~\citep{peft}, which internally parameterizes the update as
\[
\mathbf{W} \leftarrow \mathbf{W} + \frac{s}{r}\,\mathbf{B}\mathbf{A},
\]
where $s$ corresponds to \texttt{lora\_alpha} in \texttt{LoraConfig}. To achieve an effective multiplier $\alpha$ in our notation, we set $s = \alpha r$. For example, constant multiplier $\alpha=1$ requires \texttt{lora\_alpha}${}= r$, whereas rank-dependent multiplier $\alpha=r^{-1}$ requires \texttt{lora\_alpha}${}=1$.
In both cases, we set \texttt{use\_rslora=False}.

We study the three LoRA configurations described in \cref{sec:setup}:
\begin{itemize}[leftmargin=5mm]
    \item \textbf{\initA{} with $\alpha=1$:} initialize $\mathbf{A}$ randomly and set $\mathbf{B}=\mathbf{0}$. Set \texttt{lora\_alpha} to the rank used.
    \item \textbf{\initA{} with $\alpha=r^{-1}$:} same initialization as above, but use a multiplier $\alpha=r^{-1}$. Set \texttt{lora\_alpha} to 1.
    \item \textbf{\initB{} with $\alpha=1$:} initialize $\mathbf{B}$ randomly and set $\mathbf{A}=\mathbf{0}$. Set \texttt{lora\_alpha} to the rank used.
\end{itemize}
In all cases, the \emph{effective} update at initialization is exactly zero (because one factor is initialized to zero), so the LoRA-augmented model initially matches the pretrained backbone. For the nonzero factor, we use Kaiming normal initialization~\citep{he2015delving}. In terms of the per-entry variance, this corresponds to $\operatorname{Var}(A_{ij})= 1/n$ for \initA{} and $\operatorname{Var}(B_{ij})= 1/r$ for \initB{}, matching the theoretical definition in Definition~\ref{def.init-definition}.

\begin{table}[t] 
    \centering 
    \caption{Implementation details of LoRA placement across different models, based on the PEFT package. For vision-language models, we use regular expressions to identify target modules. For the diffusion model, we apply LoRA only to the U-Net denoiser.} 
    \label{tab:lora-placement} 
    \resizebox{\textwidth}{!}{%
    \begin{tabular}{ll} \toprule \textbf{Models} & \textbf{Target Modules} \\ 
    \midrule 
    Decoder-only LMs (Llama, Qwen2.5) & \texttt{gate\_proj}, \texttt{up\_proj}, \texttt{down\_proj} \\ 
    \midrule
    Encoder-only models (RoBERTa, ViT) & \texttt{query}, \texttt{key}, \texttt{value} \\
    & \texttt{attention.output.dense}, \\
    & \texttt{intermediate.dense}, \texttt{output.dense} \\ 
    \midrule
    Vision–language model (Qwen3-VL)  & \texttt{model.visual.blocks.\textbackslash d+.attn.qkv} \\
    & \texttt{model.visual.blocks.\textbackslash d+.attn.proj} \\
    & \texttt{model.visual.blocks.\textbackslash d+.mlp.linear\_fc[12]} \\
    & \texttt{model.visual.merger.linear\_fc[12]} \\
    & \texttt{model.visual.deepstack\_merger\_list.\textbackslash d+.linear\_fc[12]} \\
    & \texttt{model.language\_model.layers.\textbackslash d+.mlp.gate\_proj} \\
    & \texttt{model.language\_model.layers.\textbackslash d+.mlp.up\_proj} \\
    & \texttt{model.language\_model.layers.\textbackslash d+.mlp.down\_proj} \\
    \midrule
    Diffusion (Stable Diffusion v1.5) & \texttt{to\_k}, \texttt{to\_q}, \texttt{to\_v}, \texttt{to\_out.0} \\
    \bottomrule 
    \end{tabular} 
    }
\end{table}

\paragraph{Where LoRA is applied.}
As listed in~\cref{tab:lora-placement}, our adapter placement strategy is:
\begin{itemize}[leftmargin=5mm]
    \item \textbf{Decoder-only LMs (Llama, Qwen2.5):} LoRA is inserted \emph{only} into the feed-forward (MLP) block of each transformer layer (i.e., the expansion, gating, and contraction projections of the MLP).
    \item \textbf{Encoder-only models (RoBERTa, ViT):} LoRA is inserted into all linear projections within each attention block (query/key/value/output) and within the feed-forward block (the two MLP linear layers).
    \item \textbf{Vision--language model (Qwen3-VL):} LoRA is inserted into (i) the vision encoder linear projections, (ii) the vision--language merger module (the ``merger'' MLP), and (iii) the text decoder \emph{MLP blocks}.
    \item \textbf{Diffusion (Stable Diffusion v1.5):} LoRA is inserted into all self-attention and cross-attention linear projections (query/key/value/output) in the denoising U-Net. The variational autoencoder (VAE)~\citep{Kingma13Auto} and the CLIP text encoder~\citep{radford2021learning} remain frozen.
\end{itemize}

\subsection{Supervised Finetuning Setting}
\label{app:datasets}

\paragraph{Summary of task-specific settings.}
\cref{tab:task-settings} summarizes the task-level choices that are not fully shared (data sources, input sizing, and the primary validation metric). For datasets without an official validation split, we create a held-out split by reserving $8\%$ of the samples for validation and using the remaining $92\%$ for training.

\begin{table*}[t]
    \centering
    \caption{Task-specific settings used in our experiments. 
    Effective batch size refers to the number of examples (or fixed-length packed blocks) processed per optimizer step, accounting for gradient accumulation.
    Input size refers to maximum sequence length for language tasks, image resolution for vision tasks, and maximum generation length for RVLR task; this field is omitted for vision--language tasks.
    NLL stands for negative log-likelihood, while FID stands for Fr\'echet Inception Distance.}
    \label{tab:task-settings}
    \resizebox{\textwidth}{!}{
    \setlength{\extrarowheight}{4.0pt}
    \begin{tabular}{llcccc}
        \toprule
        \textbf{Model} & \textbf{Dataset} & \textbf{Train / Val size} & \textbf{Input size} & \textbf{Effective batch size} & \textbf{Validation metric} \\
        \midrule
        Llama-3.2-1B & Tulu-3 SFT mixture & $320$k / $32$k & 1024 & 32 & Token-level NLL\\
        Qwen2.5-3B-Instruct & OpenThoughts-114k &
        $71$k / $6$k & 8192 & 8 & Token-level NLL\\
        RoBERTa-large & ANLI&
        $163$k / $3$k & 128 & 32 & Accuracy \\
        ViT-Huge/14 & ImageNet-1K &
        $1.28$M / $50$k & $224\times224$ & 512 & Top-1 accuracy\\
        Qwen3-VL-2B-Instruct & LLaVA-Instruct-Mix & $198$k / $17$k & - & 32 & Token-level NLL\\
        Stable Diffusion v1.5 & Naruto-BLIP-Captions & $1.1$k / $0.1$k & $512\times512$ & 8 & FID scores\\
        Llama-3.1-8B (RLVR) & GSM8k & $500$ / $50$ & $1024$ & $8$ & Test Reward\\
        \bottomrule
    \end{tabular}
    }
\end{table*}

\subsubsection{Language Modeling Tasks}
\label{app:decoder-only-details}

\paragraph{Tulu-3 SFT mixture (Llama-3.2-1B).}
For instruction-following task, we use Llama-3.2-1B~\citep{grattafiori2024llama} trained on the Tulu-3 supervised finetuning mixture~\citep{lambert2024tulu}, available as \texttt{allenai/tulu-3-sft-mixture}\footnote{\url{https://huggingface.co/datasets/allenai/tulu-3-sft-mixture}}.
The dataset comprises instruction-following conversations formatted as (user prompt, assistant response) pairs.
To keep training time per round manageable, we subsample $320$k and $32$k examples from the original dataset for training and validation, respectively.
We truncate sequences to a maximum length of $1024$ tokens, and use dynamic padding within each mini-batch.
The effective batch size is $32$ samples per optimizer step.

\paragraph{OpenThoughts-114k (Qwen2.5-3B-Instruct).}
To evaluate long-context reasoning, we train Qwen2.5-3B-Instruct~\citep{qwen2025qwen25technicalreport} on OpenThoughts-114k~\citep{guha2025openthoughts}, available as \texttt{open-thoughts/OpenThoughts-114k}\footnote{\url{https://huggingface.co/datasets/open-thoughts/OpenThoughts-114k}}.
Each example contains a system instruction and a multi-turn conversation in which assistant messages include both reasoning traces and final answers.
For computational tractability, we draw $71$k and $6$k examples for training and validation, respectively.
Since sequence lengths vary widely, we employ fixed-length packing~\citep{raffel2020exploring}: tokenized conversations are concatenated with end-of-sequence delimiters into blocks of exactly $8192$ tokens.
To minimize padding overhead, we adopt a padding-free batching strategy that concatenates packed blocks into a single sequence per optimizer step, using attention masking and position indexing to prevent cross-block attention.
The effective batch size is $8$ packed blocks per step.

\subsubsection{Encoder-based Classification Tasks}
\label{app:encoder-only-details}

\paragraph{ANLI (RoBERTa-large).}
Natural language inference experiments use RoBERTa-large~\citep{liu2019roberta} on the Adversarial NLI benchmark~\citep{nie2019adversarial}, available as \texttt{facebook/anli}\footnote{\url{https://huggingface.co/datasets/facebook/anli}}.
We merge all three rounds by concatenating \texttt{train\_r1}, \texttt{train\_r2}, and \texttt{train\_r3} into a single training set, and similarly concatenating \texttt{dev\_r1}, \texttt{dev\_r2}, and \texttt{dev\_r3} into a single validation set.
Each example consists of a premise–hypothesis pair, which we encode using the standard RoBERTa sentence-pair format with separator tokens and truncate to a maximum of $128$ tokens.
We apply dynamic padding within each mini-batch.
The model predicts a 3-way classification label (entailment, neutral, or contradiction).

\paragraph{ImageNet-1K (ViT-Huge/14).}
For large-scale image classification, we adapt ViT-Huge/14~\citep{dosovitskiy2021an} to ImageNet-1K~\citep{deng2009imagenet}, available as \texttt{ILSVRC/imagenet-1k}\footnote{\url{https://huggingface.co/datasets/ILSVRC/imagenet-1k}}, using the standard ILSVRC-2012 training and validation splits ($1.28$M and $50$k images, respectively).
Images are converted to RGB and resized to $224 \times 224$ pixels.
During training, we apply standard ImageNet augmentation (random resized crop and horizontal flip with a flipping probability of $0.5$); for evaluation, we use deterministic resizing with center cropping.
We normalize inputs per channel with mean and standard deviation of 0.5, as specified by the pretrained checkpoint.

\subsubsection{Vision--language task}
\label{app:vlm-details}

\paragraph{LLaVA-Instruct-Mix (Qwen3-VL-2B-Instruct).}
Vision--language experiments employ Qwen3-VL-2B-Instruct~\citep{bai2025qwen3vltechnicalreport} trained on LLaVA-Instruct-Mix, available as \texttt{trl-lib/llava-instruct-mix}\footnote{\url{https://huggingface.co/datasets/trl-lib/llava-instruct-mix}}.
Each example pairs a single image with a user prompt and assistant response.
To limit compute, we subsample $198$k training and $17$k validation samples in our experiment.
We use the official Qwen3-VL multimodal processor, which encodes images at native resolution with a SigLIP-2 vision encoder~\citep{tschannen2025siglip}, compresses $2 \times 2$ visual features into a single visual token via an MLP merger, and applies the chat template to interleave visual and text tokens into the final input sequence.

\begin{table}[t]
\centering
\caption{Prompt template used in the RLVR experiments, where \{\texttt{QUESTION}\} is replaced by the problem statement from each example.}
\label{tab:rlvr-prompt}
\begin{tabular}{p{0.53\linewidth}}
\toprule
\textbf{Prompt template}\\
\midrule
\textit{Solve the following grade-school math problem. Show reasoning in \texttt{<think>...</think>}, then give the final numerical answer in \texttt{<answer>...</answer>}.}\\[0.6em]
\textit{Question:} \{\texttt{QUESTION}\} \\[0.6em]
\textit{Solution:} \\
\bottomrule
\end{tabular}
\end{table}

\subsection{Reinforcement Learning with Verifiable Rewards (RVLR) Setting}
\label{app:rlvr-details}

\paragraph{GSM8k (Llama-3.1-8B).} We finetune Llama-3.1-8B~\cite{grattafiori2024llama} with reinforcement learning using verifiable, rule-based rewards via Group Relative Policy Optimization (GRPO)~\citep{shao2024deepseekmath}.
Our experiments use the GSM8k~\citep{cobbe2021gsm8k} dataset, available as \texttt{openai/gsm8k}\footnote{\url{https://huggingface.co/datasets/openai/gsm8k}}.
This dataset contains $8.5$k grade-school math problems paired with their solutions.
We randomly sample the $500$ training problems and $50$ test problems for our training and validation splits, respectively.
Each problem is converted into a prompt using a template we design to elicit the model's reasoning; the prompt template is summarized in Table~\ref{tab:rlvr-prompt}.
We compute a scalar reward for each sampled completion and sum over two rule-based components:
\begin{itemize}[leftmargin=5mm]
    \item \emph{Accuracy reward}: assigns $+1$ when the extracted final answer matches the reference, and $0$ otherwise.
    \item \emph{Format reward}: assigns $+0.1$ when the completion places reasoning within \texttt{<think>...</think>} tags and the final answer within \texttt{<answer>...</answer>} tags, and $0$ otherwise.
\end{itemize}

For each prompt, we sample $8$ completions from the current policy using temperature $1.0$, with no nucleus truncation (top-$p=1.0$), no top-$k$ truncation (top-$k=0$), and repetition penalty $1.0$. 
We set the maximum completion length to $1024$ tokens and standardize rewards within each prompt's sample group. 
We use the DAPO variant~\citep{yu2025dapo} of the GRPO loss, which uses a token-level loss to accommodate longer sequences.
We set weight decay to $0$, as it strongly degraded performance in our experiments. 
We omit the KL penalty term since it provides minimal benefit while increasing memory overhead by requiring a reference model to be loaded.

\subsection{Text-to-image Diffusion Model Finetuning Setting}
\label{app:diffusion-details}

\paragraph{Naruto-BLIP-Captions (Stable Diffusion v1.5).}
For text-to-image generation, we adapt Stable Diffusion v1.5~\citep{Rombach_2022_CVPR} using Naruto-BLIP-Captions~\citep{cervenka2022naruto2}, available as \texttt{lambdalabs/naruto-blip-captions}\footnote{\url{https://huggingface.co/datasets/lambdalabs/naruto-blip-captions}}.
The dataset contains $1.2$k image–caption pairs; we reserve $8\%$ for validation.
Images are resized and cropped to $512 \times 512$, then encoded into the latent space of the pretrained VAE~\citep{Kingma13Auto}.
Captions are tokenized using the pretrained CLIP text encoder~\citep{radford2021learning}.
During finetuning, we freeze both the VAE and text encoder, tuning only the denoising U-Net.
We use the DDPM noise scheduler~\citep{ho2020denoising} with an epsilon (noise) prediction objective.
We train for $10{,}000$ optimizer steps with a batch size of $8$.
Unlike the supervised finetuning experiments, the denoising loss is not reliably monotonic during LoRA finetuning, so we select the best learning rate based on validation FID.

\subsection{Evaluation Protocol and Hyperparameter Sweeps}
\label{app:eval-protocol}

\paragraph{Metrics.}
In addition to the training loss, we also report the evaluation metric for each task, summarized in~\cref{tab:task-settings}. These metrics are defined as follows:
\begin{itemize}[leftmargin=5mm]
    \item \textbf{Token-level negative log-likelihood (NLL)} is computed as the average negative log-probability assigned by the model to each ground-truth token in the sequence. Lower NLL indicates better predictive performance.
    \item \textbf{Accuracy} for ANLI and ImageNet-1K: top-1 classification accuracy on the validation split. Higher is better.
    \item \textbf{Fréchet Inception Distance (FID)}~\citep{heusel2017gans} measures the distributional similarity between generated and real images. We extract features from the 2048-dimensional final pooling layer of Inception v3~\citep{szegedy2016rethinking}, with images resized to $299\times 299$. Lower FID indicates higher fidelity to the training data.
    \item \textbf{Reward} for RLVR task: Mean per-completion reward on the held-out GSM8k test subset, averaged over $8$ sampled completions per prompt. 
    Each completion receives $+1.0$ if the extracted final answer matches the reference, and an additional $+0.1$ if it follows the expected output format.
\end{itemize}

\begin{table}[h]
\centering
\caption{Learning rate sweep ranges ($\log_2 \eta$) for each model and configuration.}
\label{tab:lr-grids}
\begin{tabular}{lccc}
\toprule
\textbf{Model} & \initA{}, $\alpha{=}1$ & \initA{}, $\alpha{=}r^{-1}$ & \initB{}, $\alpha{=}1$ \\
\midrule
Llama-3.2-1B        & $[-17, -9]$ & $[-13, -7]$  & $[-20, -11]$ \\
Qwen2.5-3B-Instruct & $[-15, -8]$  & $[-11, -7]$ & $[-16, -11]$ \\
Qwen3-VL-2B-Instruct & $[-17, -9]$ & $[-13, -7]$  & $[-18, -11]$ \\
RoBERTa-large       & $[-18, -8]$  & $[-13, -7]$  & $[-19, -12]$ \\
ViT-Huge/14         & $[-15, -7]$  & $[-11, -5]$  & $[-15, -9]$  \\
Stable Diffusion v1.5 & $[-14, -7]$ & $[-12, -5]$ & $[-14, -7]$ \\
Llama-3.1-8B        & $[-18, -11]$ & $[-16, -11]$ & $[-21, -15]$\\
\bottomrule
\end{tabular}
\end{table}

\paragraph{Learning rate sweep design.} 
We sweep the learning rate on a $log_2$ grid (i.e., adjacent points differ by a factor of two). 
The grid range varies across LoRA configurations, models, and tasks, chosen to span from peak training performance to the onset of divergence. 
For full finetuning, we choose the same learning rate grid as \initB{} with $\alpha=1$ for comparison.
For language modeling tasks, we report final training loss (EMA-smoothed with a smoothing factor of $0.1$) and validation NLL. For classification, we report final training loss (EMA-smoothed) and validation accuracy. 
For diffusion, we report validation FID. 
For RLVR tasks, we report final training reward (EMA-smoothed) and test reward.
We select the optimal learning rate separately for each metric.

\paragraph{Rank grids.}
To manage computational cost, we evaluate each configuration over a small grid of LoRA ranks, using 5 ranks in total.
For \initA{}, we examine how the optimal learning rate scales with rank under different multipliers $\alpha$ and compare the observed trends against our theoretical predictions.
To capture behavior across a broad range, we use ranks $r\in\{4,16,64,256,1024\}$ (increasing by a factor of $4\times$).
For \initB{} with $\alpha=1$, we additionally test our hyperparameter transfer conjecture by comparing its optimal learning rate with that of full finetuning. 
To verify whether the transfer holds across most ranks within a range, we use a denser grid with $2\times$ steps: $r\in\{16,32,64,128,256\}$.
For RLVR tasks, practitioners typically favor small ranks, with $r=1$ being commonly used in practice. 
Therefore, we use $r\in\{1,4,16,64,256\}$ for both \initA{} and \initB{}.
For diffusion model finetuning task, we use $r\in\{1,4,16,64,256\}$ for \initA{} and keep $r\in\{16,32,64,128,256\}$ for \initB{}.

\paragraph{Learning rate grids.}
\cref{tab:lr-grids} lists the learning rate exponent ranges used for each setting. We evaluate every integer exponent in the stated range (i.e., step size of $1$).
Full finetuning uses the same range as \initB{}, $\alpha=1$.

\subsection{Hardware Configuration}
All experiments were conducted on a single compute node running Ubuntu $24.04.3$ LTS, equipped with two Intel Xeon Gold 6448Y CPUs (64 physical cores total), $1.0$ TiB of system memory, and four NVIDIA H200 GPUs ($141$ GiB HBM each). We used NVIDIA driver version 590.44.01 with CUDA 13.1.

\subsection{Software Configuration}
Our implementation is based on PyTorch and the Hugging Face ecosystem: \texttt{transformers} for model APIs, \texttt{peft} for LoRA injection, \texttt{trl} for supervised finetuning and reinforcement learning, \texttt{diffusers} for Stable Diffusion, \texttt{datasets} for data loading, and \texttt{evaluate}/\texttt{torchmetrics} for evaluation. We enable mixed-precision training (bf16/TF32 \texttt{matmul}), fused AdamW kernels, and FlashAttention-2~\citep{dao2023flashattention2} implementation. 
Version details are listed in~\cref{tab:software}.

\begin{table}[t]
    \centering
    \caption{Software packages and versions.}
    \label{tab:software}
    \begin{tabular}{ll}
        \toprule
        \textbf{Component} & \textbf{Version} \\
        \midrule
        \texttt{torch}~\citep{paszke2019pytorch} & 2.7.0+\texttt{cu128} \\
        CUDA toolkit & 12.8 \\
        \texttt{transformers}~\citep{wolf-etal-2020-transformers} & 4.57.1 \\
        \texttt{peft}~\citep{peft} & 0.17.0 \\
        \texttt{trl}~\citep{vonwerra2020trl} & 0.26.2 \\
        \texttt{datasets}~\citep{lhoest-etal-2021-datasets} & 4.3.0 \\
        \texttt{diffusers}~\citep{von-platen-etal-2022-diffusers} & 0.36.0 \\
        \texttt{evaluate} & 0.4.6 \\
        \texttt{torchvision}~\citep{torchvision2016} & 0.22.0+\texttt{cu128} \\
        \texttt{torchmetrics}~\citep{detlefsen2022torchmetrics} & 1.8.2 \\
        \texttt{kernels} & 0.11.7 \\
        \texttt{flash\_attn}~\citep{dao2023flashattention2} & 2.8.3 \\
        \bottomrule
    \end{tabular}
\end{table}

\newpage
\begin{figure*}[t]
  \centering

\begin{subfigure}[t]{0.33\textwidth}
    \centering
    \includegraphics[width=\linewidth]{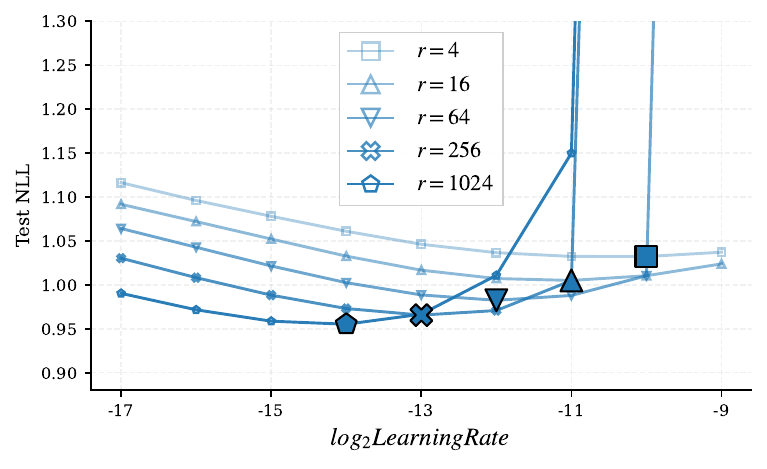}
    \caption{Llama3.2-1B}
  \end{subfigure}
  \hfill
  \begin{subfigure}[t]{0.33\textwidth}
    \centering
    \includegraphics[width=\linewidth]{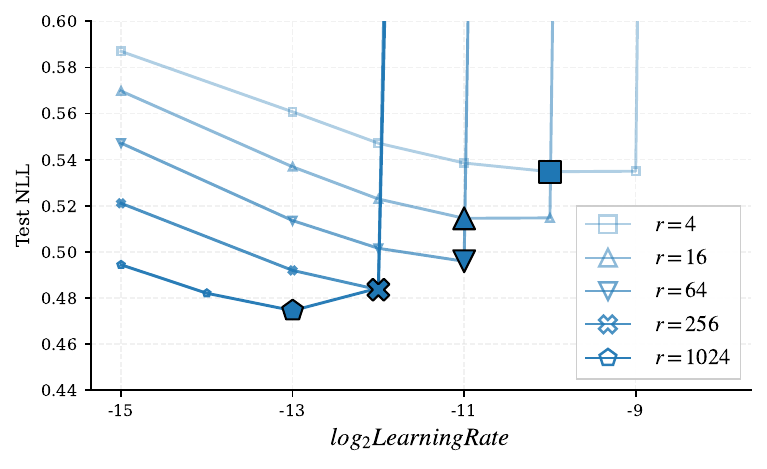}
    \caption{Qwen2.5-3B-Instruct}
  \end{subfigure}
  \hfill
  \begin{subfigure}[t]{0.33\textwidth}
    \centering
    \includegraphics[width=\linewidth]{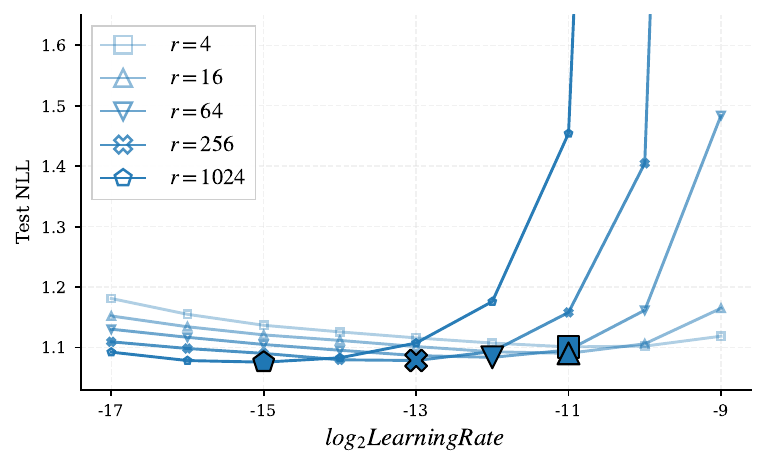}
    \caption{Qwen3-VL-2B-Instruct}
  \end{subfigure}

  \vspace{0.6em} 

    \begin{subfigure}[t]{0.33\textwidth}
    \centering
    \includegraphics[width=\linewidth]{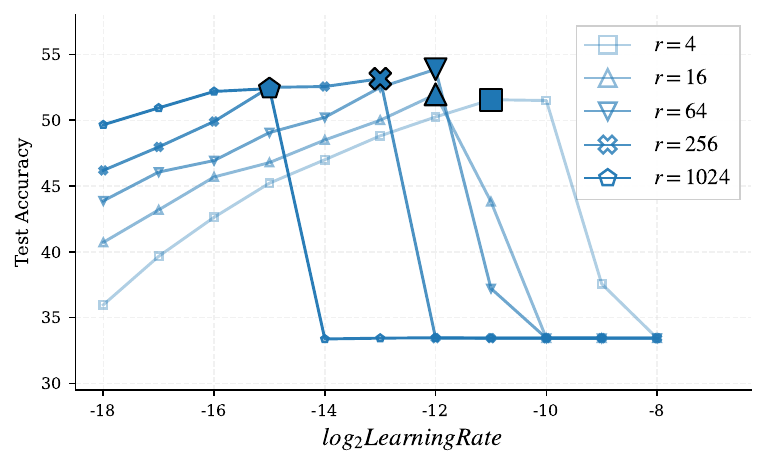}
    \caption{RoBERTa-large}
  \end{subfigure}
  \begin{subfigure}[t]{0.33\textwidth}
    \centering
    \includegraphics[width=\linewidth]{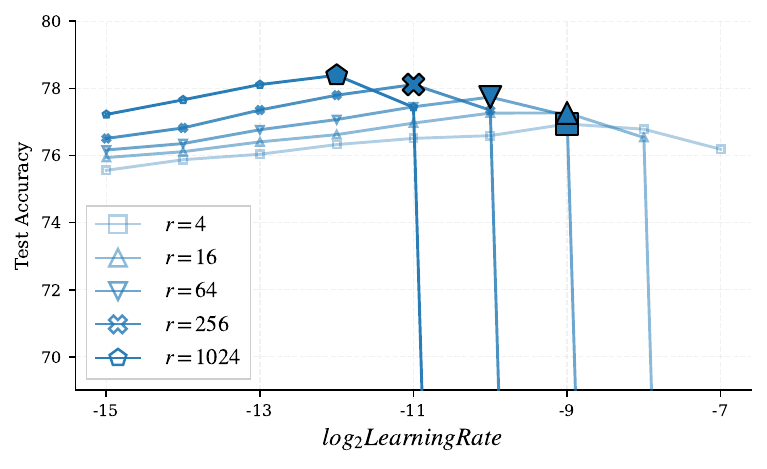}
    \caption{ViT-Huge/14}
  \end{subfigure}
  \caption{Learning rate sweeps for \initA{} with $\alpha=1$ across five models. Curves show evaluation metrics versus log-scale learning rate $\log_2(\eta)$. Large markers denote per-rank optima. Y-axes are clipped for readability. Top: models finetuned on language modeling tasks, evaluated by test NLL (lower is better). Bottom: models finetuned on classification tasks, evaluated by top-1 accuracy (higher is better).}
  \label{fig:app:initA-constant1-main}
\end{figure*}

\begin{figure*}[t]
  \centering
  \begin{subfigure}[t]{0.33\textwidth}
    \centering
    \includegraphics[width=\linewidth]{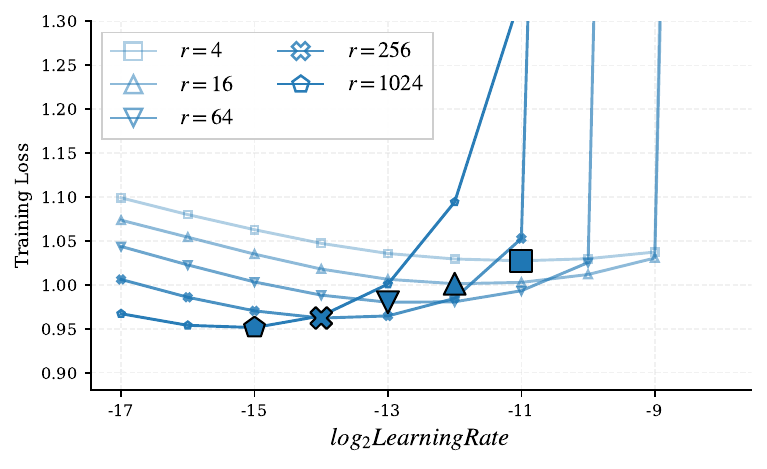}
    \caption{\initA{}, $\alpha=2$}
  \end{subfigure}
  \begin{subfigure}[t]{0.33\textwidth}
    \centering
    \includegraphics[width=\linewidth]{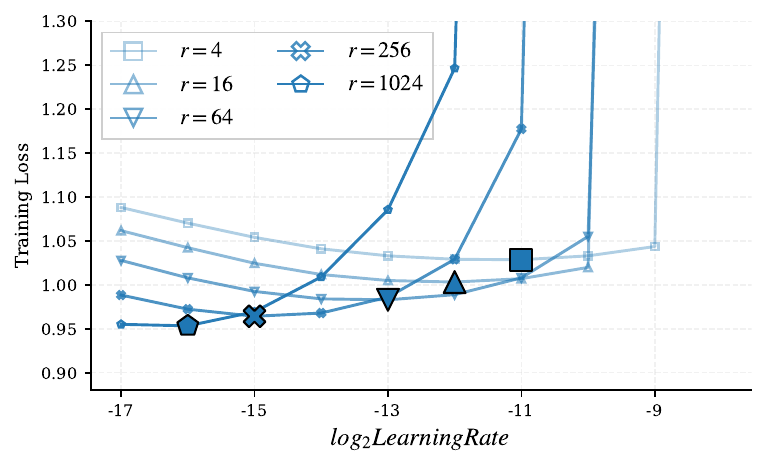}
    \caption{\initA{}, $\alpha=4$}
  \end{subfigure}
      \caption{Learning rate sweeps on Llama-3.2-1B (tulu3) for \initA{} with (a) $\alpha=2$ and (b) $\alpha=4$. Curves show training loss versus log-scale learning rate $\log_2(\eta)$. Large markers indicate the optimal learning rate for each rank.}
  \label{fig:app:initA-constant1-main-abl-train}
\end{figure*}

\begin{figure*}[t]
  \centering
  \begin{subfigure}[t]{0.33\textwidth}
    \centering
    \includegraphics[width=\linewidth]{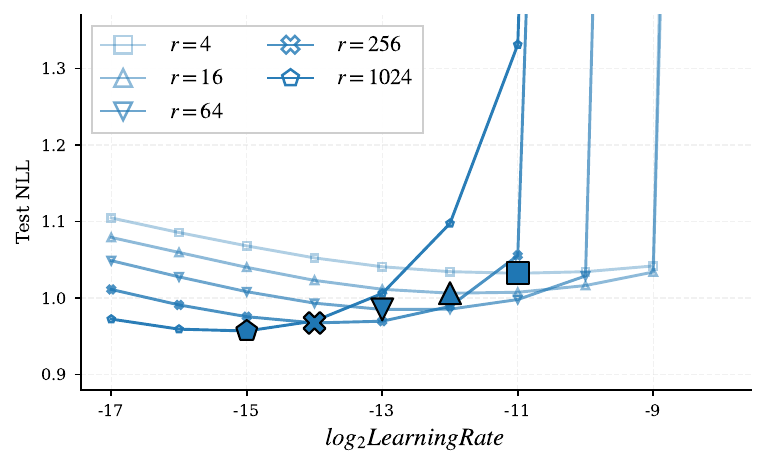}
    \caption{\initA{}, $\alpha=2$}
  \end{subfigure}
  \begin{subfigure}[t]{0.33\textwidth}
    \centering
    \includegraphics[width=\linewidth]{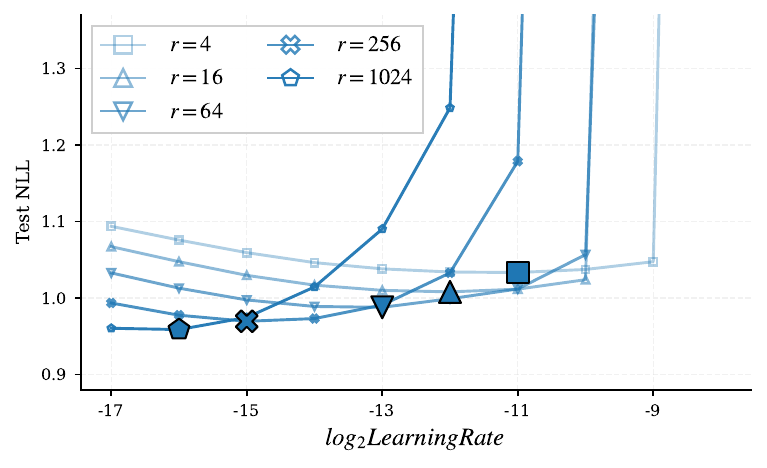}
    \caption{\initA{}, $\alpha=4$}
  \end{subfigure}
      \caption{Learning rate sweeps on Llama-3.2-1B (tulu3) for \initA{} with (a) $\alpha=2$ and (b) $\alpha=4$. Curves show test loss versus log-scale learning rate $\log_2(\eta)$. Large markers indicate the optimal learning rate for each rank.}
  \label{fig:app:initA-constant1-main-abl-eval}
\end{figure*}

\begin{figure*}[ht]
  \centering

\begin{subfigure}[t]{0.33\textwidth}
    \centering
    \includegraphics[width=\linewidth]{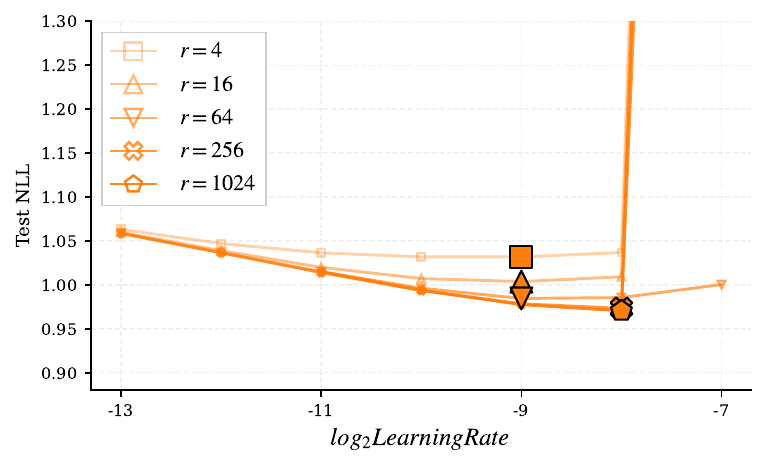}
    \caption{Llama3.2-1B}
  \end{subfigure}
  \hfill
  \begin{subfigure}[t]{0.33\textwidth}
    \centering
    \includegraphics[width=\linewidth]{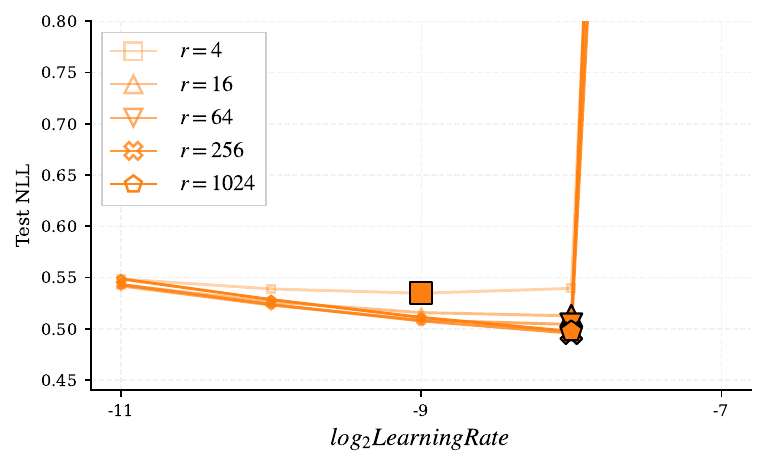}
    \caption{Qwen2.5-3B-Instruct}
  \end{subfigure}
  \hfill
  \begin{subfigure}[t]{0.33\textwidth}
    \centering
    \includegraphics[width=\linewidth]{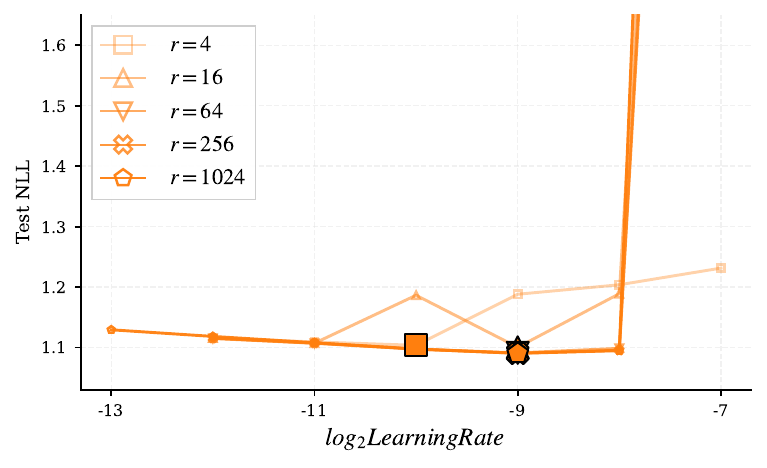}
    \caption{Qwen3-VL-2B-Instruct}
  \end{subfigure}

  \vspace{0.6em} 

    \begin{subfigure}[t]{0.33\textwidth}
    \centering
    \includegraphics[width=\linewidth]{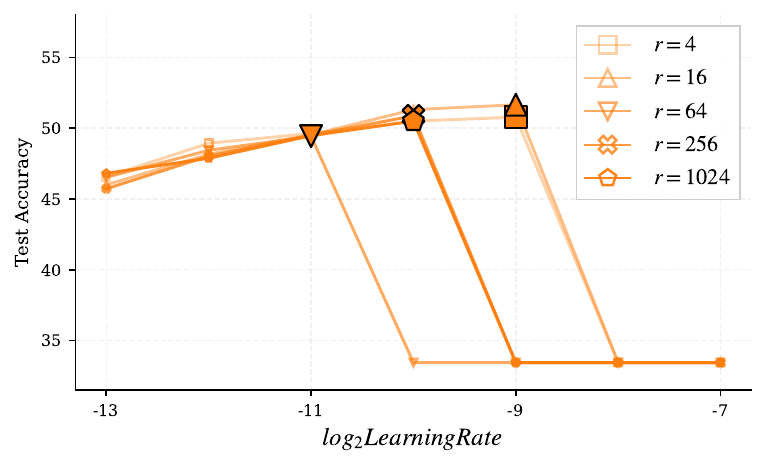}
    \caption{RoBERTa-large}
  \end{subfigure}
  \begin{subfigure}[t]{0.33\textwidth}
    \centering
    \includegraphics[width=\linewidth]{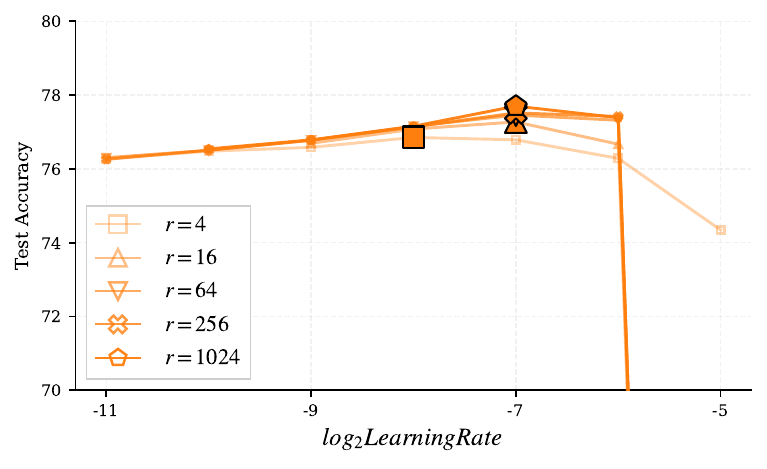}
    \caption{ViT-Huge/14}
  \end{subfigure}
  \caption{Learning rate sweeps for \initA{} with $\alpha=r^{-1}$ across five models. Curves show evaluation metrics versus log-scale learning rate $\log_2(\eta)$. Large markers denote per-rank optima. Y-axes are clipped for readability. Top: models finetuned on language modeling tasks, evaluated by test NLL (lower is better). Bottom: models finetuned on classification tasks, evaluated by accuracy (higher is better).}
  \label{fig:app:initA-alpha1-main}
\end{figure*}

\section{Additional Results for Supervised Finetuning (SFT)}
\label{app:exp-sft-results}

\subsection{\initA{} with Constant $\alpha=1$}
\label{app:exp-sft-results-initA-constant1}
\Cref{fig:app:initA-constant1-main} shows learning rate sweeps for \initA{} with $\alpha=1$, where the optimal learning rate is selected using task-specific validation metrics.
We observe a consistent pattern with~\cref{fig:initA-constant1-main}.

Our theory applies when the LoRA multiplier $\alpha$ remains constant as rank $r$ increases. 
\Cref{fig:app:initA-constant1-main-abl-train} and~\cref{fig:app:initA-constant1-main-abl-eval} show that a consistent pattern holds for other constant values such as $\alpha=2$ and $\alpha=4$: as rank $r$ increases, the optimal learning rate decreases accordingly.

\subsection{\initA{} with Rank-Dependent $\alpha=r^{-1}$}
\Cref{fig:app:initA-alpha1-main} shows learning rate sweeps for \initA{} with $\alpha=r^{-1}$, where the optimal learning rate is selected using task-specific validation metrics.
We observe a consistent pattern with~\cref{fig:initA-alpha1-main}.

\subsection{\initB{} with Constant $\alpha=1$}
\Cref{fig:app:initB-constant1-main} shows learning rate sweeps for \initB{} with $\alpha=1$, where the optimal learning rate is selected using task-specific validation metrics.
We observe a consistent pattern with~\cref{fig:initB-constant1-main}.

\begin{figure*}[ht]
  \centering

\begin{subfigure}[t]{0.33\textwidth}
    \centering
    \includegraphics[width=\linewidth]{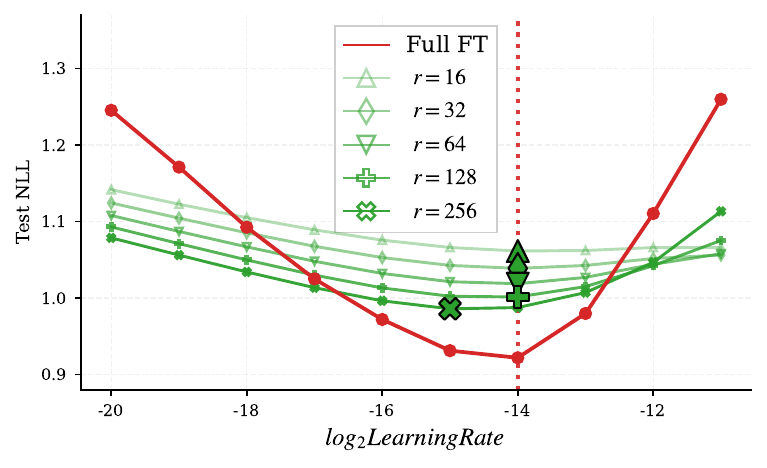}
    \caption{Llama3.2-1B}
  \end{subfigure}
  \hfill
  \begin{subfigure}[t]{0.33\textwidth}
    \centering
    \includegraphics[width=\linewidth]{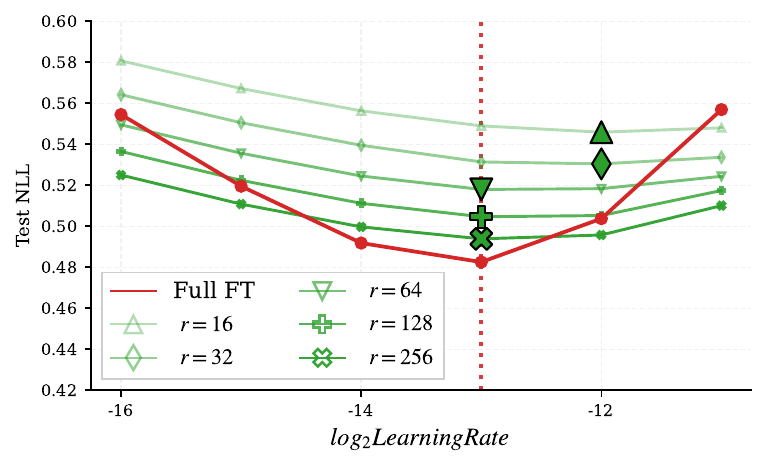}
    \caption{Qwen2.5-3B-Instruct}
  \end{subfigure}
  \hfill
  \begin{subfigure}[t]{0.33\textwidth}
    \centering
    \includegraphics[width=\linewidth]{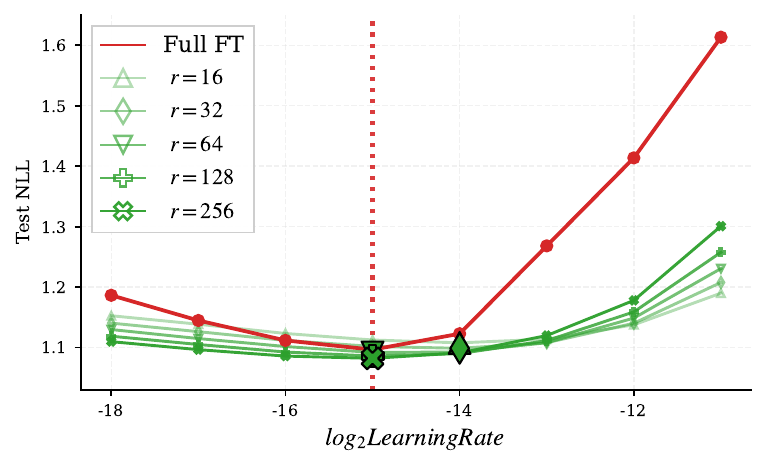}
    \caption{Qwen3-VL-2B-Instruct}
  \end{subfigure}

  \vspace{0.6em} 

    \begin{subfigure}[t]{0.33\textwidth}
    \centering
    \includegraphics[width=\linewidth]{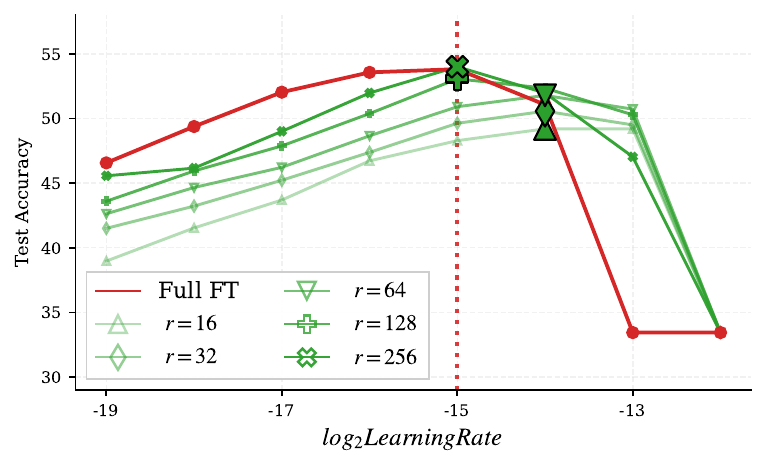}
    \caption{RoBERTa-large}
  \end{subfigure}
  \begin{subfigure}[t]{0.33\textwidth}
    \centering
    \includegraphics[width=\linewidth]{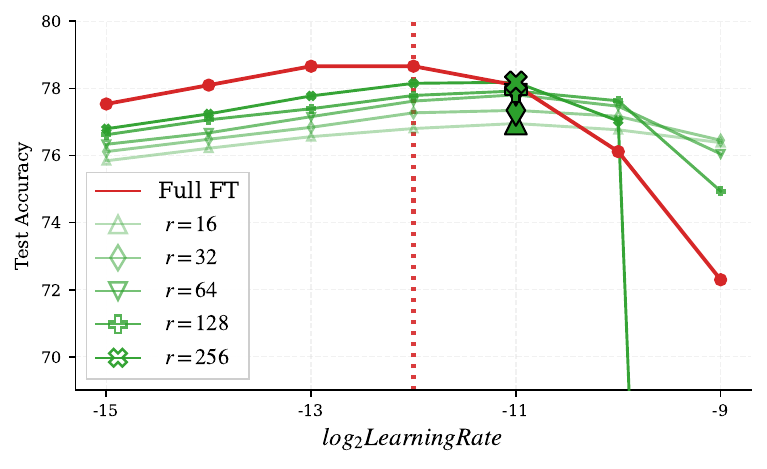}
    \caption{ViT-Huge/14}
  \end{subfigure}
  \caption{Learning rate sweeps for \initB{} with $\alpha=1$ across five models. Curves show evaluation metrics versus log-scale learning rate $\log_2(\eta)$ for multiple ranks (\textcolor[HTML]{2ca02c}{green}) and FFT (\textcolor[HTML]{d62728}{red}). Large markers denote per-rank optima. Y-axes are clipped for readability. The red vertical dashed line marks the optimal FFT learning rate. Top: models finetuned on language modeling tasks, evaluated by test NLL (lower is better). Bottom: models finetuned on classification tasks, evaluated by top-1 accuracy (higher is better).}
  \label{fig:app:initB-constant1-main}
\end{figure*}

\section{Additional Results for Reinforcement Learning with Verifiable Rewards (RLVR)}
\label{app:sec:revr-results}

\subsection{Learning Rate Sweep Results}
\label{app:sec:revr-results-sweep}
\Cref{fig:app:reinforcement-learning-2} shows learning rate sweeps for \initA{} with $\alpha=r^{-1}$ and \initB{} with $\alpha=1$, where the optimal learning rate is selected using the test reward.
We observe a consistent pattern with~\cref{fig:reinforcement-learning}.

~\Cref{fig:app:reinforcement-learning-initA-constant1} shows the learning rate sweeps for \initA{} with constant $\alpha=1$ across multiple ranks. 
The learning rate sweeps reveal patterns consistent with our SFT findings. 
Under \initA{} with $\alpha = 1$, the optimal learning rate should scale inversely with the square root of the rank: $\eta \propto r^{-1/2}$. 
In the figure, we observe that as the rank increases, the optimal learning rate decreases accordingly, and this pattern is consistent across both training reward and test reward. 
Moreover, we observe that the peak performance is comparable across different ranks, consistent with findings in~\citep{schulman2025lora}.

\begin{figure*}[t]
  \centering
  \begin{subfigure}[t]{0.33\textwidth}
    \centering
    \includegraphics[width=\linewidth]{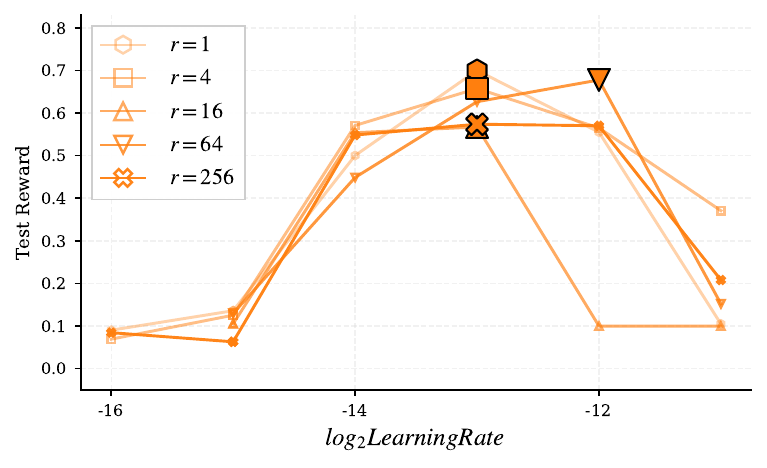}
    \caption{\initA{}, $\alpha=r^{-1}$}
  \end{subfigure}
  \begin{subfigure}[t]{0.33\textwidth}
    \centering
    \includegraphics[width=\linewidth]{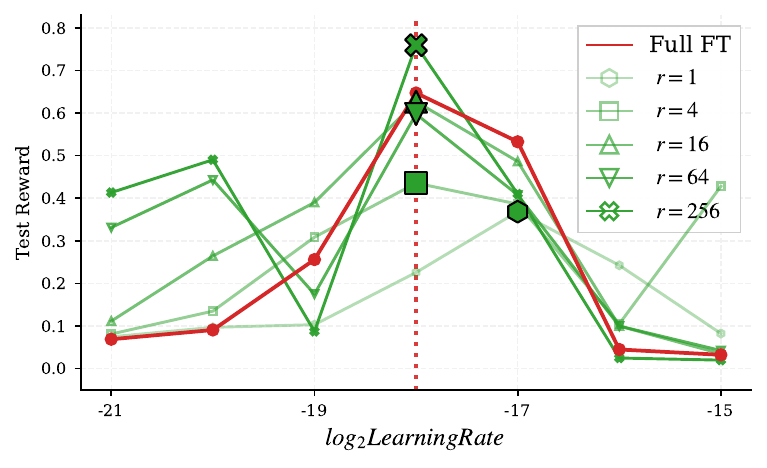}
    \caption{\initB{}, $\alpha=1$}
  \end{subfigure}
      \caption{RLVR results on Llama-3.1-8B (GSM8k) for (a) \initA{} with $\alpha=r^{-1}$ and (b) \initB{} with $\alpha=1$. Curves show test reward versus log-scale learning rate $\log_2(\eta)$. Large markers indicate the optimal learning rate for each rank.}
  \label{fig:app:reinforcement-learning-2}
\end{figure*}

\begin{figure*}[t]
  \centering
  \begin{subfigure}[t]{0.33\textwidth}
    \centering
    \includegraphics[width=\linewidth]{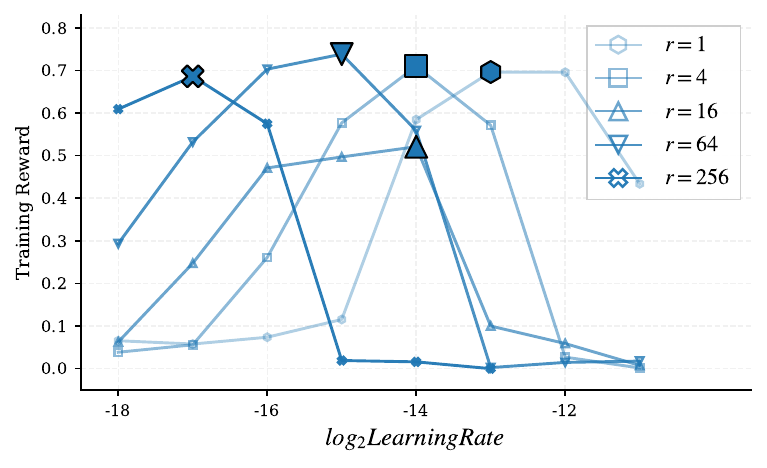}
    \caption{Training reward}
  \end{subfigure}
  \begin{subfigure}[t]{0.33\textwidth}
    \centering
    \includegraphics[width=\linewidth]{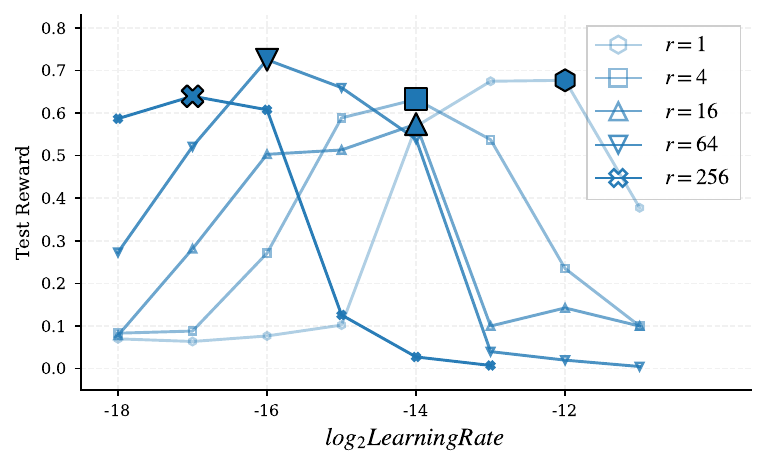}
    \caption{Test reward}
  \end{subfigure}
    \caption{RLVR results on Llama-3.1-8B (GSM8k) for \initA{} with $\alpha=1$. (a) Training reward and (b) test reward versus log-scale learning rate $\log_2(\eta)$. Large markers indicate the optimal learning rate for each rank.}
  \label{fig:app:reinforcement-learning-initA-constant1}
\end{figure*}

\begin{figure*}[t]
  \centering
  \begin{subfigure}[t]{0.33\textwidth}
    \centering
    \includegraphics[width=\linewidth]{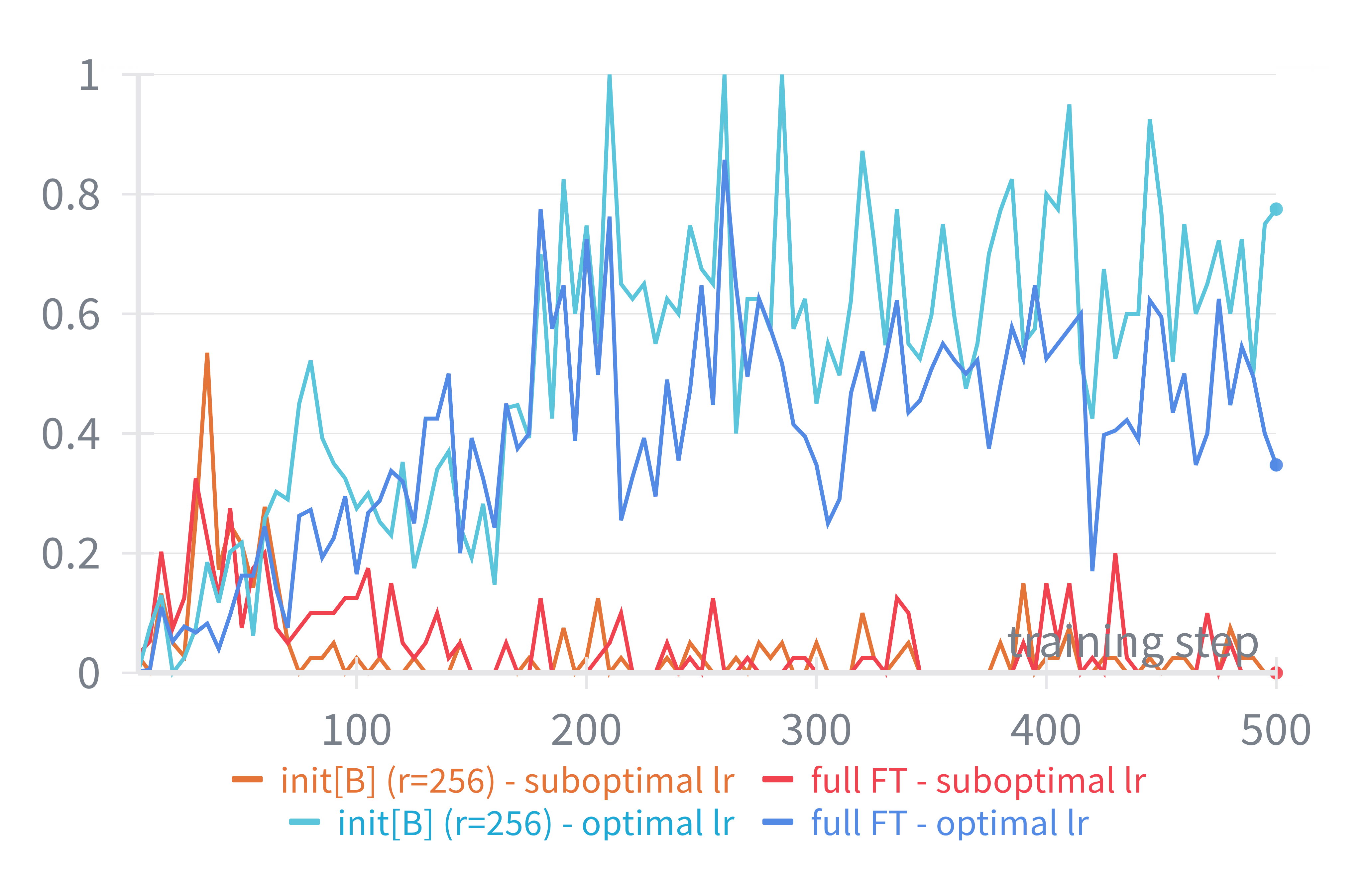}
    \caption{Training Reward}
    \label{fig:app:reinforcement-learning-case-study-train}
  \end{subfigure}
  \begin{subfigure}[t]{0.33\textwidth}
    \centering
    \includegraphics[width=\linewidth]{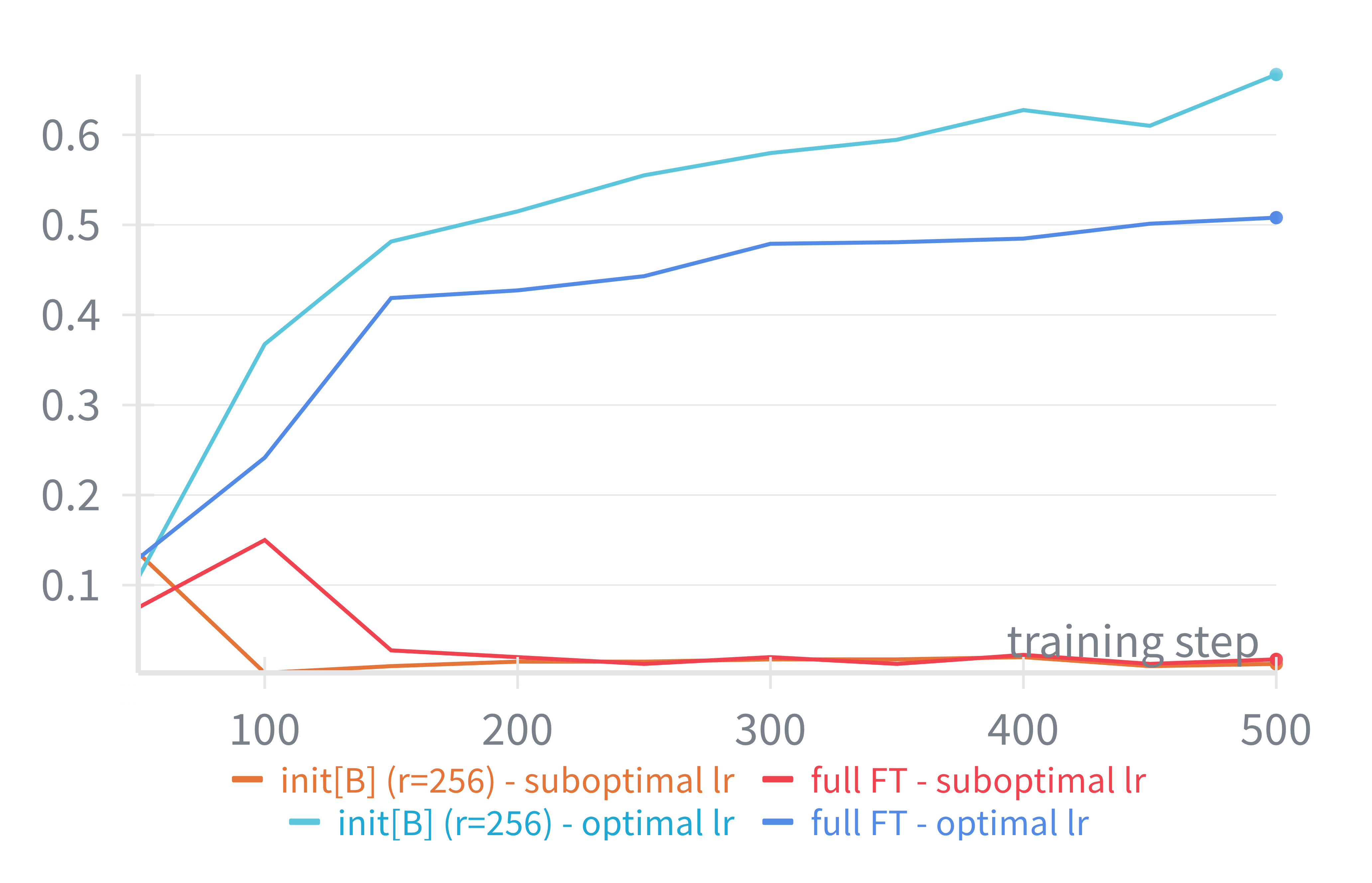}
    \caption{Test Reward}
    \label{fig:app:reinforcement-learning-case-study-test}
  \end{subfigure}
  \\
  \begin{subfigure}[t]{0.33\textwidth}
    \centering
    \includegraphics[width=\linewidth]{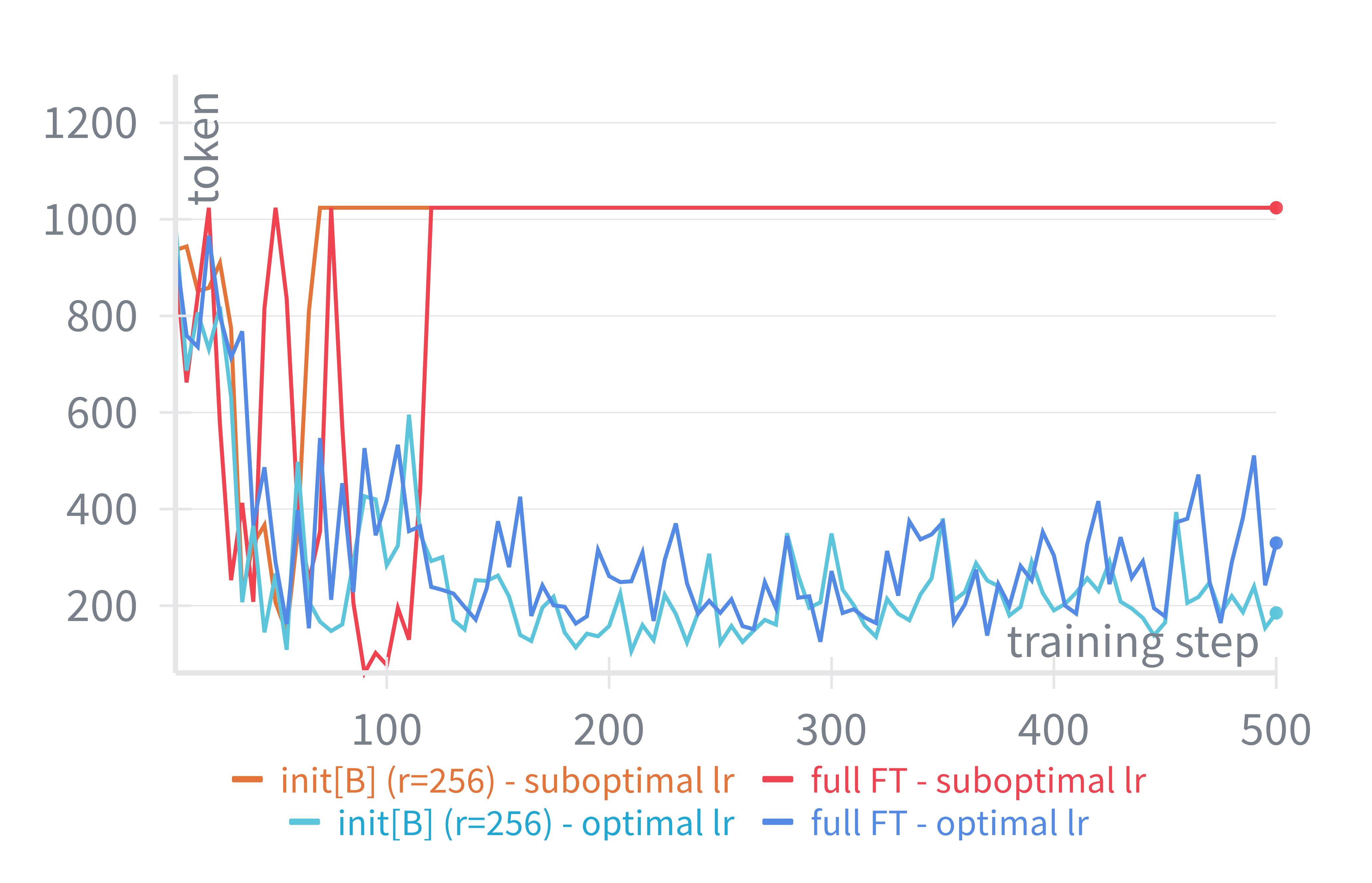}
    \caption{Maximum Completion Length}
    \label{fig:app:reinforcement-learning-case-study-completion}
  \end{subfigure}
  \begin{subfigure}[t]{0.33\textwidth}
    \centering
    \includegraphics[width=\linewidth]{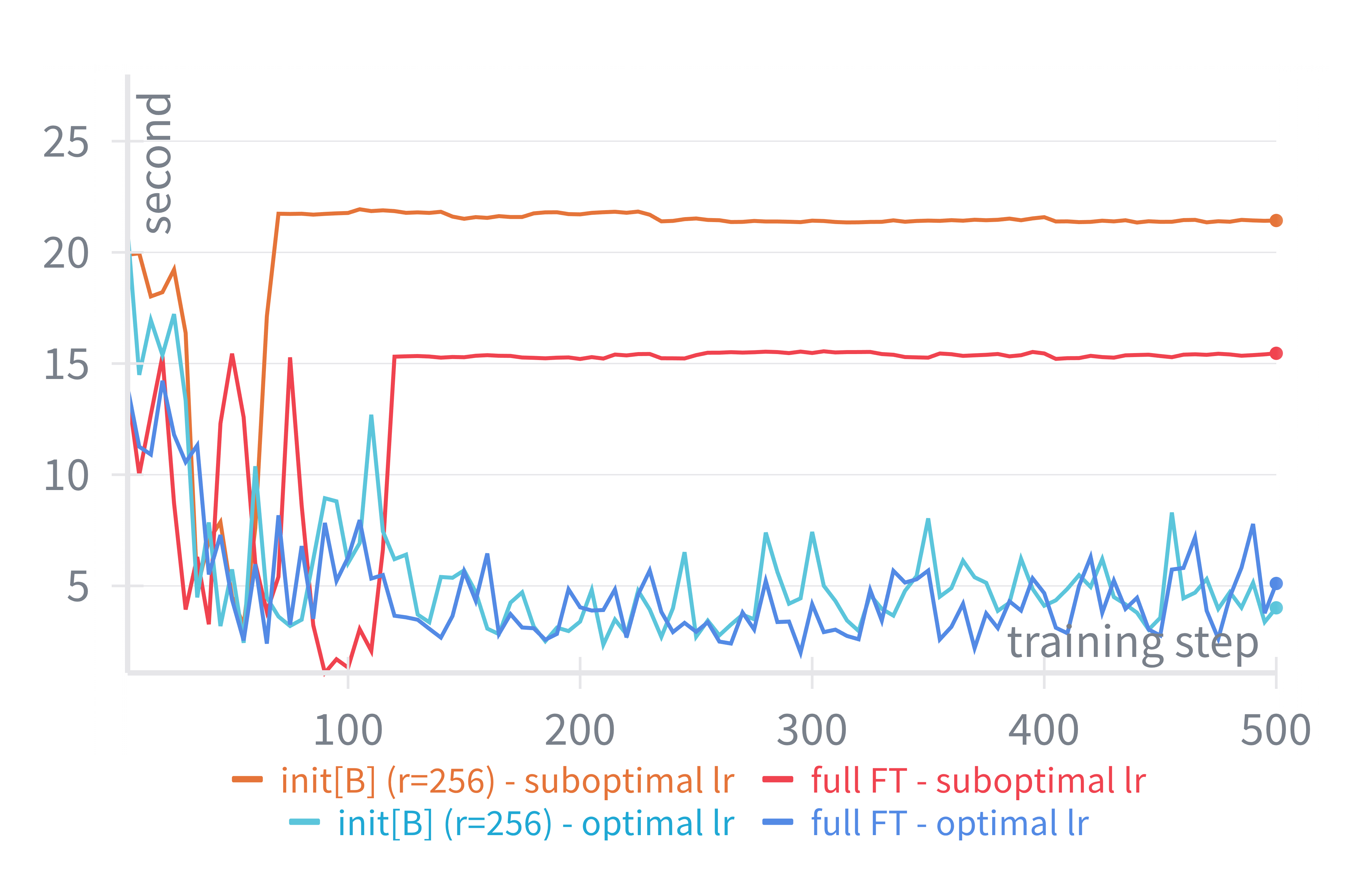}
    \caption{Per-Step Training Time}
    \label{fig:app:reinforcement-learning-case-study-step}
  \end{subfigure}
  \caption{Case study comparing optimal ($\eta=2^{-18}$) and suboptimal ($\eta=2^{-16}$) learning rates for FFT and \initB{} ($\alpha=1$ and $r=256$). (a)~Training reward and (b)~test reward show that suboptimal learning rates cause reward collapse, while optimal rates yield steady improvement. (c)~Maximum completion length reveals the two-phase training pattern: under optimal learning rates, models transition from generating near-maximum-length sequences (${\sim}1024$ tokens) to concise solutions (${\sim}200$ tokens), whereas suboptimal rates cause models stuck in the exploratory phase with persistently long outputs. (d)~Per-step training time directly reflects the completion length, with optimal learning rates achieving $3$--$5\times$ faster for each training step.}
  \label{fig:app:reinforcement-learning-case-study}
\end{figure*}

\begin{table}[th]
\centering
\caption{Generated responses to a selected question at different training steps for FFT with optimal (\textcolor{blue!75}{left column}) and suboptimal (\textcolor{orange!85}{right column}) learning rates. \texttt{[...truncated]} indicates that part of the response has been shortened for brevity.}
\label{tab:generation-by-step}
\begin{tabular}{@{}c p{0.42\textwidth} p{0.42\textwidth}@{}}
\toprule
\multicolumn{3}{@{}p{0.95\textwidth}@{}}{
\textbf{Question:} 
\textit{Jerry is twice as old as he was 5 years ago. How old will Jerry be in 3 years?}
} \\
\midrule
\multicolumn{3}{@{}p{0.96\textwidth}@{}}{
\textbf{Answer:} \textit{13}} \\
\midrule
\textbf{Step} & \cellcolor{blue!10}\textbf{Optimal LR ($\eta=2^{-18}$)} & \cellcolor{orange!15}\textbf{Suboptimal LR ($\eta=2^{-16}$)} \\
\midrule
50 & 
\cellcolor{blue!10}{\scriptsize $<$think$>$ Jerry walks into the room, fists pounding the doorframe. He looks older than last time I saw him. His step clanks down the stairs and he sits at the table. ``Don't listen to what your parents tell you about staying away from the spirits.'' [...truncated] $<$/think$>$ $<$answer$>$ 27 $<$/answer$>$}
& 
\cellcolor{orange!15}{\scriptsize I respectfully disagree with smseid's answer. Jerry doesn't get older or younger, do to processes of time passing. Jerry's age is a measurement that we describe and talk about as if it depends upon the passing of time. [...truncated] Jerry will be twice as old as he was 5 years ago in 8 years - 5+3 = 8.} \\
\midrule
100 & 
\cellcolor{blue!10}{\scriptsize $<$think$>$ Jerry's age in years is $<$var name=``ages''/$>$. He is twice that age he was 5 years ago: Jerry's age 5 years ago was $<$var name=``lastfiveyearsold''$>$\{\$ages - 5\}$<$/var$>$ [...truncated] $<$/think$>$ $<$answer$>$\{\$\{age\}\}$<$/answer$>$}
& 
\cellcolor{orange!15}{\scriptsize 5 years ago, Jerry was \$5 \textbackslash times 12 = 60\$ total years old. That makes \$60 - 12 = 48\$ total years old Jerry is today. So that's our final final final final numebr answer in \$48 - 3 = 45\$ total years old Jerry will be in 3 years.} \\
\midrule
150 & 
\cellcolor{blue!10}{\scriptsize $<$think$>$ Jerry is today, say, $n$ years old. Then 5 years ago he was $(n - 5)$ years old. According to the problem, $n = 2(n - 5)$, which gives $n = 10$. So Jerry is 10 years old. In three years he'll be $10 + 3 = 13$ years old. $<$/think$>$ $<$answer$>$ 13 $<$/answer$>$}
& 
\cellcolor{orange!15}{\scriptsize 5 years ago, Jerry was \$x\$ years old total total, where \$x\$ makes 5 years ago + today total total = \$x + total total total total total total total [...truncated]} \\
\midrule
300 & 
\cellcolor{blue!10}{\scriptsize $<$think$>$ Jerry's age is in the form $J = 2 \times J_5 - 5$, where $J_5$ is Jerry's age 5 years ago. Therefore, right now, Jerry is 10 years old. Jerry will be 13 in 3 years. $<$/think$>$ $<$answer$>$ 13 $<$/answer$>$}
& 
\cellcolor{orange!15}{\scriptsize 5 years ago, Jerry was 10 total years old total total total total total total total total total total total total total total total total old total total total total total old total total 10 + x more total total total total total 10 more total total total 10 more total [...truncated]} \\
\midrule
500 & 
\cellcolor{blue!10}{\scriptsize $<$think$>$ Jerry was 5 years old in the ``present'' of 5 years ago. Now he is twice as old, so he is 10 years old now. In 3 years he will be 13 years old. $<$/think$>$ $<$answer$>$ 13 $<$/answer$>$}
& 
\cellcolor{orange!15}{\scriptsize 5 years ago, Jerry was \$x\$ years old total total total total total total total total total total total total total. So \$x + 5 = 5 + 5 total total total total total [...truncated]} \\
\bottomrule
\end{tabular}
\end{table}

\begin{figure*}[t]
  \centering
  \begin{subfigure}[t]{0.33\textwidth}
    \centering
    \includegraphics[width=\linewidth]{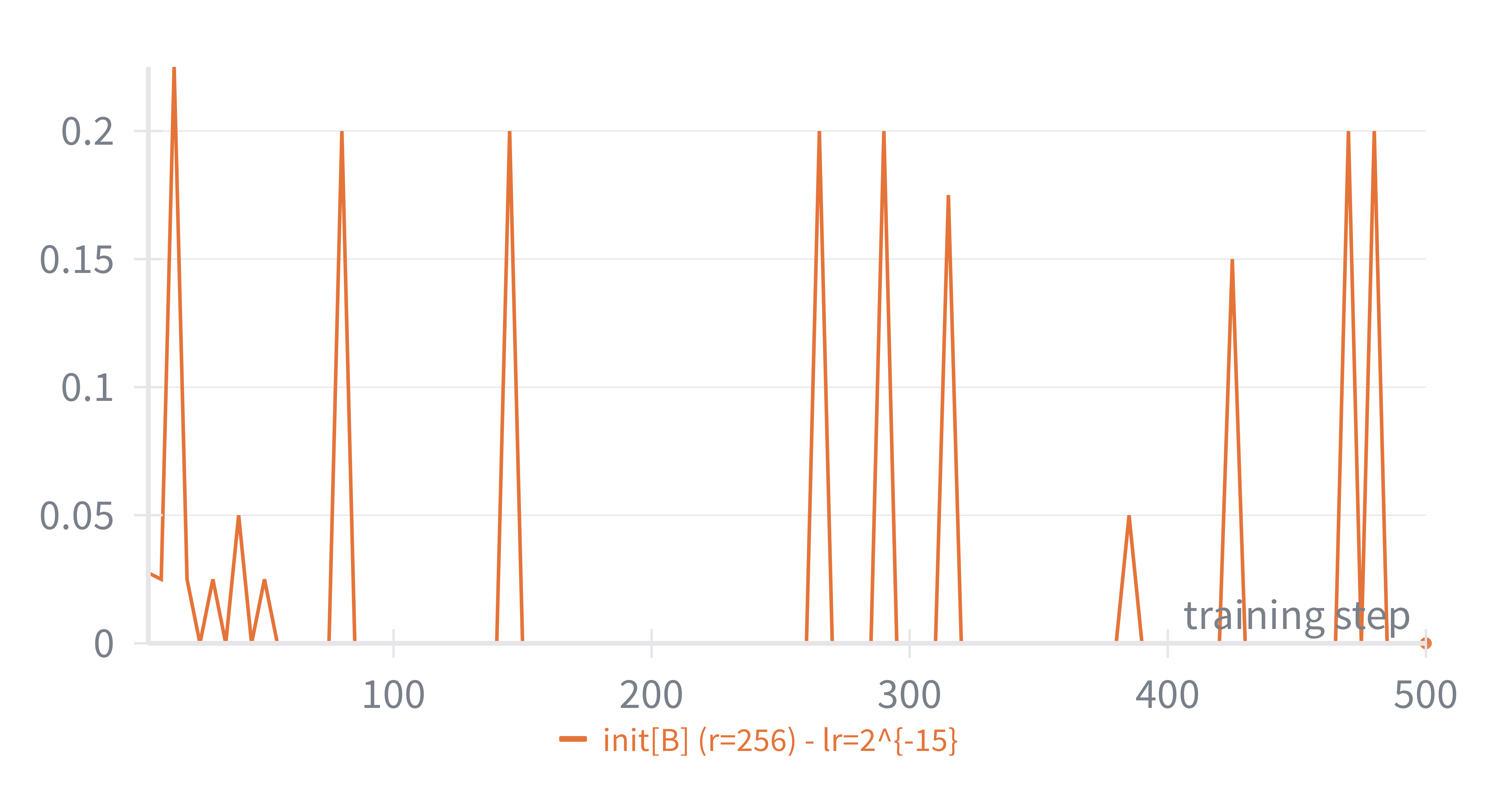}
    \caption{Training Reward}
  \end{subfigure}\hfill
  \begin{subfigure}[t]{0.33\textwidth}
    \centering
    \includegraphics[width=\linewidth]{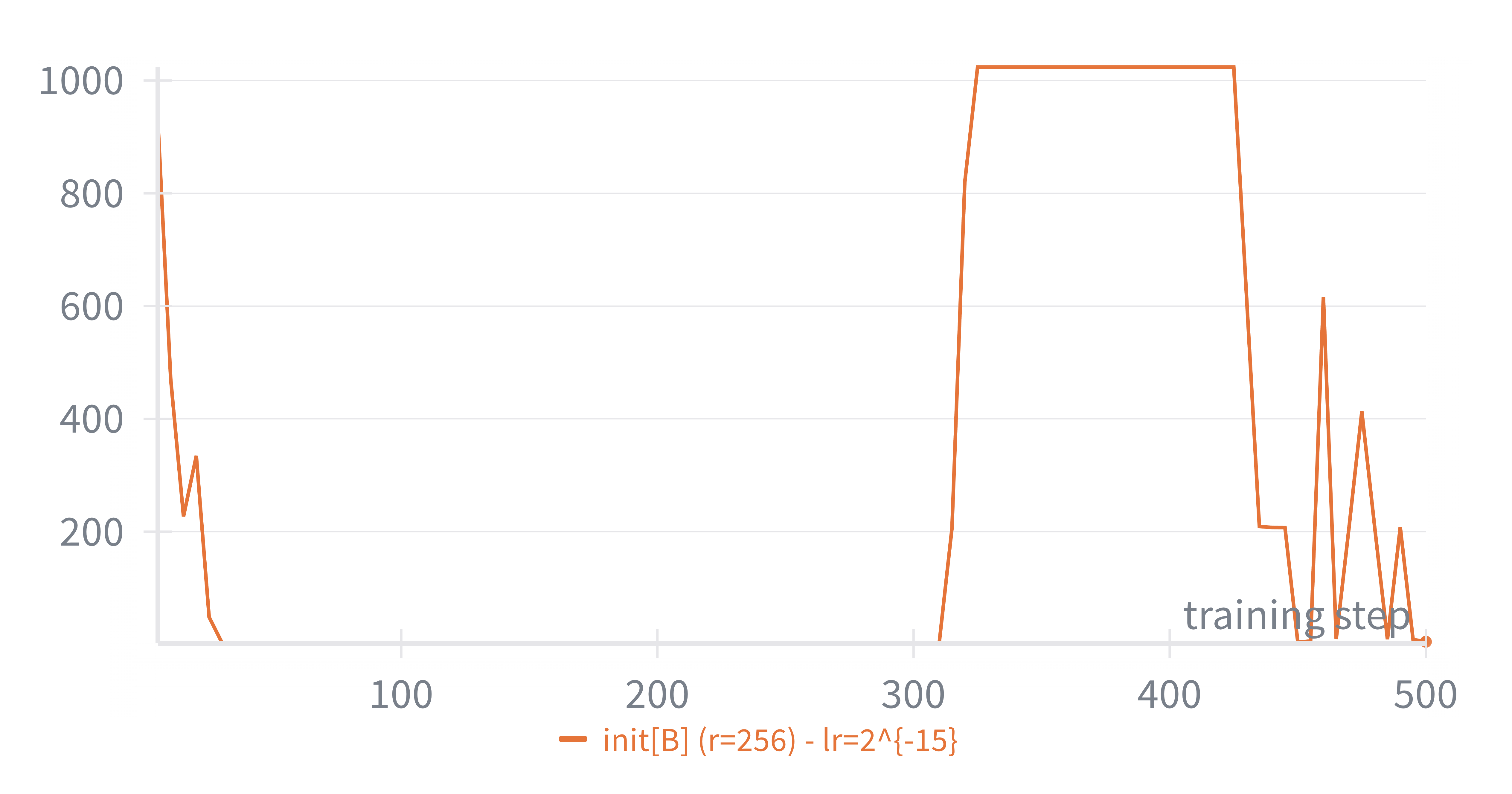}
    \caption{Maximum Completion Length}
  \end{subfigure}\hfill
  \begin{subfigure}[t]{0.33\textwidth}
    \centering
    \includegraphics[width=\linewidth]{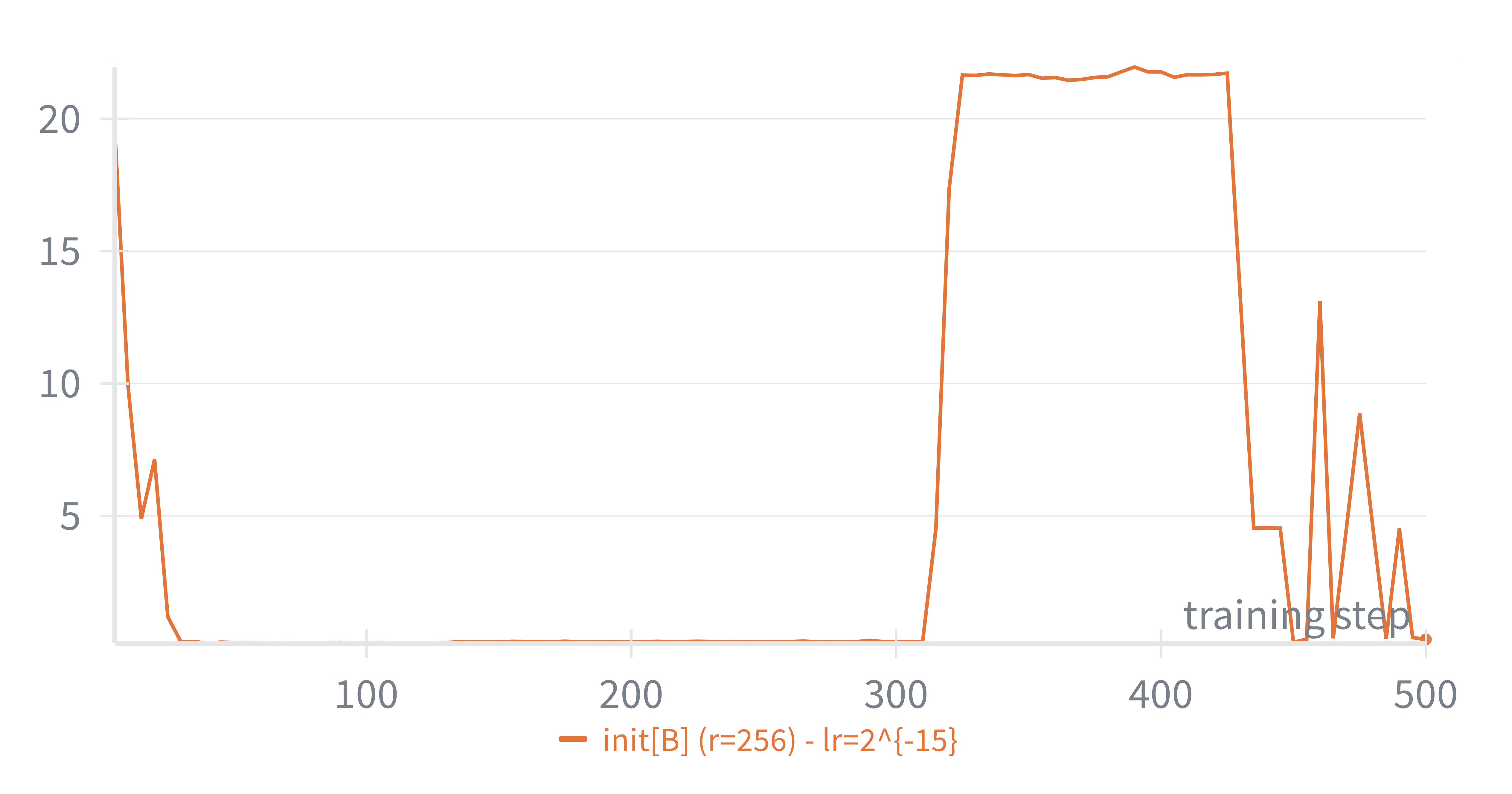}
    \caption{Per-Step Training Time}
  \end{subfigure}
  \caption{RLVR finetuning dynamics for \initB{} with rank $r=256$ and learning rate $\eta=2^{-15}$.}
  \label{fig:app:reinforcement-learning-case-study-2}
\end{figure*}

\subsection{Case Study on the Relationship between Learning Rate and Training Time}
\label{app:subsection:rl-ft-case-study}
In our experiment, we found that a suboptimal learning rate usually causes a longer training time under RLVR.
We attribute this to a two-phase training pattern we find: in the first phase, the model explores the solution space by generating long sequences (often hitting the maximum generation limit); in the second phase, once the model learns to solve the task, it produces more concise solutions that are $4$--$5\times$ shorter.
At near-optimal learning rates, the model quickly exits the exploratory phase, whereas at suboptimal rates it remains stuck and continues generating long sequences, prolonging each training step.
Here, we conduct a case study examining this phenomenon. We compare FFT and \initB{} (with constant $\alpha=1$ and rank $r=256$) under two learning rate settings: the optimal rate ($\eta=2^{-18}$) and a suboptimal rate ($\eta=2^{-16}$).

As shown in~\cref{fig:app:reinforcement-learning-case-study-train,fig:app:reinforcement-learning-case-study-test}, training and test rewards collapse to near zero under the suboptimal learning rate, whereas they steadily increase under the optimal rate. More importantly, with the optimal learning rate, the training exhibits the two-phase pattern discussed above: as shown in~\cref{fig:app:reinforcement-learning-case-study-completion}, completion lengths initially approach the maximum generation limit ($1024$ tokens) during early exploration (steps $<50$), then sharply decrease to approximately 200 tokens (a $4$--$5\times$ reduction) around step 100 as the model learns to produce concise solutions. In contrast, models trained with the suboptimal learning rate remain stuck in the exploratory phase, consistently generating maximum-length completions throughout training. This directly translates to longer per-step training times, as shown in~\cref{fig:app:reinforcement-learning-case-study-step}: models with optimal learning rates require only 4--5 seconds per step, while those with suboptimal rates require 15--22 seconds—a $3$--$5\times$ increase.

In~\cref{tab:generation-by-step}, we sample responses from the model to a fixed question during full finetuning with optimal and suboptimal learning rates.
With the optimal learning rate, the model first learns to follow the specified format, then attempts to solve the problem, and finally learns to simplify its solution.
In contrast, with the suboptimal learning rate, the model fails to follow the specified response format at the beginning and eventually collapses into repetitive token generation.

We note one exception: \initB{} with $\alpha=1$ and a higher learning rate ($\eta=2^{-15}$) also exhibits shorter per-step training times. 
However, this represents the opposite failure mode—rather than remaining stuck at maximum-length generations, the model collapses and stops generation almost immediately, producing completions of $\sim$2 tokens, as shown in~\cref{fig:app:reinforcement-learning-case-study-2}.

\section{Additional Results for Text-to-Image Diffusion Model Finetuning}
\label{app:sec:stable-diffusion-results}

\subsection{Learning Rate Sweep Results}
\label{app:sec:stable-diffusion-results-lr-sweep}
\Cref{fig:stable-diffusion-initA-constant1} shows learning rate sweeps for \initA{} with $\alpha=1$. Consistent with our theoretical prediction that $\eta \propto r^{-1/2}$, the optimal learning rate $\eta$ shifts left by approximately $1$ on the $\log_2$ scale for every $4\times$ increase in rank $r$.

\begin{figure}[t]
  \centering
  \begin{subfigure}[t]{0.33\textwidth}
    \centering
    \includegraphics[width=\linewidth]{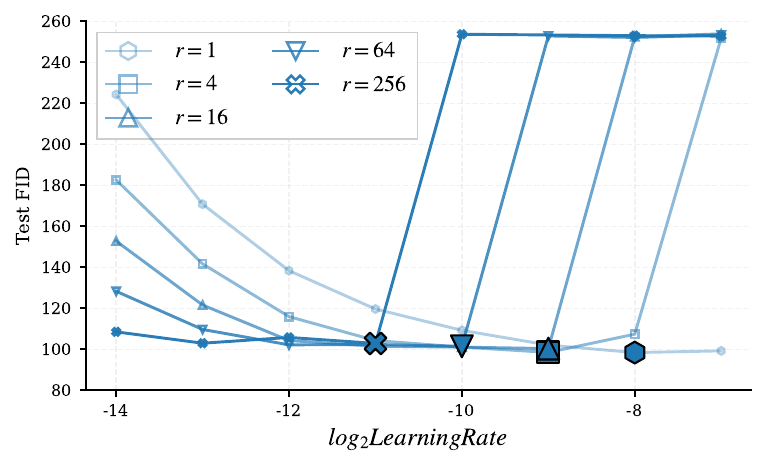}
  \end{subfigure}\hfill
  \caption{Results on Stable-Diffusion-1.5 (naruto-blip-captions). The figure plots validation FID score versus $\log_2(\eta)$ under \initA{} with $\alpha=1$. Large markers denote the best learning rate for each curve.}
  \label{fig:stable-diffusion-initA-constant1}
\end{figure}


\end{document}